 \documentclass[3p,times,procedia]{elsarticle}

\usepackage{setspace}
\usepackage{amsthm}
\usepackage{amsmath}
\usepackage{amssymb}
\usepackage{graphicx}
\usepackage[unicode=true]{hyperref}
\usepackage{epstopdf}
\usepackage{algorithm}
\usepackage{algorithmic,color}

\def\argmin{\mathop{\rm argmin}}

\newcommand{\zb}[1]{\ensuremath{\boldsymbol{#1}}}

\newtheorem{theorem}{Theorem}

\algsetup{indent=2em}

\begin{document}

\begin{frontmatter}

\title{Variational Image Segmentation Model Coupled with Image Restoration Achievements}


\author{Xiaohao Cai}

\address{Department of Plant Sciences, and Department of Applied Mathematics
and Theoretical Physics, University of Cambridge, Cambridge, UK\\
cai@mathematik.uni-kl.de}

\begin{abstract}
Image segmentation and image restoration are two important topics in 
image processing with great achievements. 
In this paper, we propose a new multiphase segmentation model by combining 
image restoration and image segmentation models.
Utilizing image restoration aspects, the proposed segmentation model can effectively
and robustly tackle high noisy images, blurry images, images with missing pixels,
and vector-valued images. In particular, one of the most important segmentation models, the
piecewise constant Mumford-Shah model, can be extended easily in this way to segment
gray and vector-valued images corrupted for example by noise, blur or missing pixels 
after coupling a new data fidelity term which 
comes from image restoration topics. It can be solved efficiently using 
the alternating minimization algorithm, and we prove the convergence of this algorithm 
with three variables under mild condition. 
Experiments on many synthetic and real-world images demonstrate that 
our method gives better segmentation results in comparison to others state-of-the-art 
segmentation models especially for blurry images and images with missing pixels values.
\end{abstract}

\begin{keyword}


image segmentation \sep image restoration \sep piecewise constant Mumford-Shah model
\end{keyword}

\end{frontmatter}

\section{Introduction}\label{sec:introduction}

Image segmentation and image restoration are two important subjects in image processing. 
Image segmentation consists in partitioning a given image into multiple 
segments to transfer the representation of the image into a more meaningful one which 
is easier to analyze. It is typically used to locate objects and boundaries within
an image. Image restoration is the operation of estimating the desired clean image
from its corrupted version. Corruption may come in many forms, such as blur, 
noise, camera misfocus, or information lost. Obviously, image segmentation
can be used as preprocessing or postprocessing of image restoration.
In other words, these two topics can influence each other.

Let $\Omega \subset \mathbb{R}^2$ be a bounded, open, connected set, 
and $f: \Omega \rightarrow \mathbb{R}$ a given image.
Without loss of generality, we restrict the range of $f$ to [0,1].
Let $g: \Omega \rightarrow \mathbb{R}$ denote the desired clean
image, then $f = g + n_f$, when $n_f$ is the additive noise.
Many image restoration models can be written in the form 
\begin{equation} \label{image-rtrn}
E(g)=\lambda \Phi(f, g) + \phi(g),
\end{equation}
where $\Phi(f, g)$ is the data fidelity term, $\phi(g)$ is the regularization term, and
$\lambda > 0$ is a regularization parameter balancing
the trade-off between terms $\Phi(f, g)$ and $\phi(g)$.
If set $\Phi(f, g) = \int_{\Omega}(f-g)^2 dx$ and $\phi(g) = \int_{\Omega}|\nabla g|dx$
(total variation term), then model \eqref{image-rtrn} goes to the ROF model proposed
by  Rudin, Osher and Fatemi in 1992 \cite{ROF}, i.e.,
\begin{equation} \label{rof-model}
E(g)=\lambda\int_{\Omega}(f-g)^2 dx + \int_{\Omega}|\nabla g|dx.
\end{equation}
One important advantage of model \eqref{rof-model} is that it preserves 
the edge information of $f$ very well, but also introduces the staircase effect. 
To remove the staircase effect, many works are designed based on higher-order derivative 
terms, see \cite{CCS,CMM,LLT,SS,YK}. For example in \cite{CCS}, the tight-frame technic was 
used in $\phi(g)$ to obtain more details of the higher-order derivative 
information from $f$. The relationship between the total variation and the tight-frame can be found
in \cite{SWBMW}. As we know, the data fidelity term $\Phi(f, g) = \int_{\Omega}(f-g)^2 dx$ is 
especially effective for Gaussian noise \cite{ROF}. For removing other types of noise than Gaussian noise, 
$\Phi(f, g) = \int_{\Omega} (g - f \log g)  dx$ and $\Phi(f, g) = \int_{\Omega} |f-g|  dx$ 
are proposed for Poisson noise and impulsive noise in \cite{C91} and \cite{N}, respectively.
Please refer to \cite{AA,CHN,CHN05,C91,HNW,N,LCA,SST,ST10} and references therein
for the details of the Poisson noise and impulsive noise removal. 
Note that the image restoration model \eqref{image-rtrn} can be extended to
process blurry image after introducing a problem related linear operator ${\cal A}$ in front 
of $g$ \cite{RO}.

Let $\Gamma \in \Omega$ represent the boundary of the objects within an image, 
and $\Omega_i $ be the parts of the segmented objects fulfill
$\Omega=\cup_i \Omega_i \cup \Gamma$. The Mumford-Shah model is one of the most 
important image segmentation models, and has been studied extensively in the last twenty 
years. More precisely, in \cite{MS}, Mumford and Shah proposed an energy
minimization problem which approximates the true solution
by finding optimal piecewise smooth approximations.
The energy minimization problem was formulated as
\begin{equation} \label{ms}
E(g, \Gamma)=\frac{\lambda}{2} \int_{\Omega}(f-g)^2 dx + 
\frac{\mu}{2} \int_{\Omega\setminus\Gamma}|\nabla g|^2 dx+ {\rm Length}(\Gamma),
\end{equation}
where $\lambda$ and $\mu$ are positive parameters, and
$g:\Omega\rightarrow \mathbb{R}$ is continuous or even
differentiable in $\Omega\setminus\Gamma$ but may be
discontinuous across $\Gamma$. Because model \eqref{ms} is nonconvex, 
it is very challenging to find or approximate its minimizer, see \cite{C,CA,GM}.
Many works \cite{KLM,VC} concentrate on simplifying model \eqref{ms} by restricting $g$ 
to be piecewise constant function ($g=c_i$ in $\Omega_i$), i.e.,
\begin{equation}  \label{pcms}
E(g, \Gamma)=\frac{\lambda}{2} \sum_{i=1}^K \int_{\Omega_i}(f-c_i)^2 dx
+ {\rm Length}(\Gamma),
\end{equation}
where $K$ is the number of phases. Using the coarea formula \cite{FWR},
model \eqref{pcms} can be rewritten as
\begin{equation}  \label{pcms-noncon}
\begin{split}
E(c_i, u_i) & = \lambda \sum_{i= 1}^K \int_{\Omega}(f-c_i)^2 u_idx+ 
\sum_{i= 1}^K \int_{\Omega} |\nabla u_i| dx, \\
{\rm s.t.} & \quad \sum_{i= 1}^K u_i(x) = 1, u_i(x) \in \{0, 1\},  \forall x\in \Omega.
\end{split}
\end{equation}
Moreover, model \eqref{pcms-noncon} with $K=2$ is the Chan-Vese model \cite{CV},
and with fixed $c_i$ is a special case of the Potts model \cite{Potts}.
Due to the nonconvex property of \eqref{pcms-noncon}, 
in \cite{CEN}, the exact convex version of \eqref{pcms-noncon} was proposed
when $K=2$ and $c_i$ fixed. For $K>2$, recently, several authors are focused 
on relaxing $u_i$ and solving the following model
\begin{equation} \label{pcms:convex}
\begin{split}
E(u_i, c_i) & = \lambda\sum_{i= 1}^K \int_{\Omega}(f-c_i)^2 u_idx+ 
\sum_{i= 1}^K \int_{\Omega} |\nabla u_i| dx, \\
{\rm s.t.} & \quad \sum_{i= 1}^K u_i(x) = 1, u_i(x) \ge 0, \forall x\in \Omega. 
\end{split}
\end{equation}
Please refer to \cite{BarChan2011,HSHMS,LNZS,PCCB,YBTBmul} and references therein
for more details related with \eqref{pcms:convex}. 
One drawback of model \eqref{pcms:convex} is that {\it it can not segment images 
corrupted by blur or information lost}, which is one main problem to solve in this paper. 

In \cite{PCS}, a model of coupling image restoration 
and segmentation based on a statistical framework of generalized linear models 
was proposed, but the analysis and algorithm therein are only 
focused on two-phase segmentation problem. 
In our previous work \cite{CCZ}, a two-stage segmentation method which
provides a better understanding of the link between image segmentation
and image restoration was proposed. The method suggests that for segmentation, 
it is reasonable and practicable to extract the different phases in $f$ from using
image restoration methods first and thresholding followed.
Moreover, in our recent work \cite{CS13}, we proved that the solution of the 
Chan-Vese model \cite{CV} for certain $\lambda$  can actually be given by thresholding the minimizer of 
the ROF model \eqref{rof-model} using a proper threshold, which clearly provides 
one kind of relationships between image segmentation and image restoration.

In this paper, start from extending the piecewise constant Mumford-Shah model \eqref{pcms} 
to manage blurry image, a novel segmentation model by composing 
model \eqref{pcms} with a data fidelity term which comes 
from image restoration topics is proposed. 
Since only a new fidelity term is added,  and usually the fidelity term possesses good
property such as differentiable, the solution of the proposed model is not more involved 
comparing with solving model \eqref{pcms}. It can be solved efficiently using the
alternating minimization (AM) algorithm \cite{CT} with the ADMM or 
primal-dual algorithms \cite{BPCPE,CP,GO}. We prove that under mild condition, 
the AM algorithm converges for the proposed model. The proposed model can 
segment blurry images easily but model \eqref{pcms} can not. 
Moreover, it can also deal with images with information lost and vector-valued images 
for example color images. Due to the advantage of two data fidelity terms,
one from image restoration and the other from image segmentation,  our model 
is much more robust and stable. Experiments on many kinds of synthetic and real-world 
images demonstrate that our method 
gives better segmentation results in comparison with other state-of-the-art segmentation 
methods especially for blurry images and images with missing pixels values.

{\bf Contributions.} The main contributions of this paper are summarized as follows.
\begin{itemize}
\item[1)] Coupling the variational image segmentation models and the image restoration
models, which provides a new way for multiphase image segmentation.
\item[2)] Extend the piecewise constant Mumford-Shah model \eqref{pcms} by
cooperating the image restoration achievements so that the new constructed variation segmentation
model can handle blurry case easily.
\item[3)] Thanks to the image restoration achievements, the new variation segmentation
model has the potential to process many different types of noises, for example
Gaussian, Poisson, and impulsive noises.  
\item[4)] The kinds of vector-valued image for example the color image and the observed image 
with information lost are also covered in the proposed variational segmentation model.
\item[5)] The convergence of the AM algorithm with three variables 
to the proposed variational model is proved.
\end{itemize}

The rest of this paper is organized as follows. In Section \ref{sec:model}, we propose
our new segmentation model and extend it so that it can deal with 
vector-valued images and images with some pixels values missing. In Section \ref{sec:solve},
the AM algorithm to our model will be introduced. 
The convergence of it will be proved in Section \ref{sec:converge}.  
In Section \ref{sec:experiments}, the comparison of the proposed method on various synthetic and 
real-world images with the state-of-the-art multiphase segmentation methods will be shown. 
Conclusions are given in Section \ref{sec:conclusions}.

%-------------------------------------------------------------------
\section{The Proposed Variational Image Segmentation Model}\label{sec:model}
We propose our image segmentation model by combining the piecewise constant 
Mumford-Shah model \eqref{pcms-noncon} with the fidelity term $\Phi(f, g)$ which comes
from the image restoration model \eqref{image-rtrn}. More precisely, our proposed 
segmentation model aims to minimize the energy
\begin{equation}  \label{our-model-general}
\begin{split}
E(u_i, c_i,g) & = \mu \Phi(f, {\cal A}g) +
\lambda \Psi(g, u_i, c_i)+ 
\sum_{i= 1}^K \int_{\Omega} |\nabla u_i| dx, \\
{\rm s.t.} & \quad \sum_{i= 1}^K u_i(x) = 1, u_i(x) \in \{0, 1\},  \forall x\in \Omega,
\end{split}
\end{equation}
where $g\in L^2(\Omega)$ and ${\cal A}$ is the problem related linear operator. 
For example ${\cal A}$ can be the identity operator 
for a noisy observed image $f$ or a blurring operator if there are noise and blur in $f$.
The first term $\Phi(f, {\cal A}g)$ is the data fidelity term arising from the image restoration 
model \eqref{image-rtrn}. It controls $g$ not far away from the given corrupted image $f$, 
in other words, it aims to deblure and denoise according to 
the types of noises in $f$. Term $\Psi(g, u_i, c_i)$ is also the data fidelity term but comes 
from the image segmentation model, which aims to separate $g$ into $K$ specified 
segments. In this paper, we restrict ourselves to
\[
\Psi(u_i, c_i, g) = \sum_{i= 1}^K \int_{\Omega}(g-c_i)^2 u_idx.
\]
The last term in \eqref{our-model-general} is the regularization term which controls the length 
of the boundaries of the segmented parts $u_i$. The type of the used data fidelity term 
$\Phi(f, {\cal A}g)$ changes according to different noise models, for example,
\begin{itemize}
\item[i.] Gaussian noise: $\Phi(f, {\cal A}g) =  \int_{\Omega}(f - {\cal A} g)^2 dx$;
\item[ii.] Poisson noise (I-divergence): $\Phi(f, {\cal A}g) 
=  \int_{\Omega} \big ({\cal A} g - f \log ({\cal A}g)\big )dx$;
\item[iii.] Impulsive noise: $\Phi(f, {\cal A}g) =  \int_{\Omega} |f - {\cal A}g| dx$.
\end{itemize}
Compared with model \eqref{pcms-noncon}, the model \eqref{our-model-general},
in addition to the ability of segmenting blurry images, its two data fidelity terms make it 
much more robust and stable to process the observed corrupted image $f$. 

Obviously, for fixed $g$, model \eqref{our-model-general} is reduced to model \eqref{pcms-noncon}.
The following theorem \ref{thm:existence} gives the uniqueness of $g$ when minimizingl 
\eqref{our-model-general} for fixed $c_i$ and $u_i$.

\begin{theorem}\label{thm:existence}
Assume $\Phi(f, {\cal A}g)$ in \eqref{our-model-general} is convex and continuous, then there exists 
one and only one $g$ which minimizes energy \eqref{our-model-general} for fixed $c_i$ and $u_i$. 
\end{theorem}
\begin{proof} See the Appendix.
\end{proof}

In the following, we restrict ourselves to $\Phi(f, {\cal A}g) =  \int_{\Omega}(f - {\cal A} g)^2 dx$ as 
one example to show how to extend model \eqref{our-model-general} so that it can
handle images with missing information and vector-valued images. 
Let $\Omega'$ be the set containing the pixels whose pixels values are missing. 
Then model \eqref{our-model-general} can be extended to segment images with missing 
information as 
\begin{equation} \label{our-model-inpainting}
E(u_i, c_i,g) = \mu \int_{\Omega}(f-{\cal A}g)^2 \omega dx 
+ \lambda\sum_{i=1}^K \int_{\Omega}(g-c_i)^2 \omega u_i dx 
+ \sum_{i= 1}^K \int_{\Omega} |\nabla u_i| dx,
\end{equation}
where 
\begin{equation} \label{our-model-inpainting-res}
\sum_{i= 1}^K u_i(x) = 1, u_i(x) \in \{0, 1\}, \
\omega(x) = 
\begin{cases}
1, \quad {\rm if} \  x\in \Omega\setminus \Omega', \\
0, \quad {\rm otherwise}.
\end{cases}
\end{equation}
For the observed vector-valued image represented as $\zb f = (f_1, \cdots, f_N)$, 
let $\zb g = (g_1, \cdots, g_N)$ and ${\zb c}_i = (c_{i, 1}, \cdots, c_{i, N})$, 
model \eqref{our-model-general} 
can be extended to segment vector-valued images with missing pixels values as  
\begin{equation}\label{our-model-vector}
E(u_i, {\zb c}_i,\zb g) =
\mu \sum_{j=1}^N\int_{\Omega}(f_j-{\cal A}_jg_j)^2 \omega dx 
+ \lambda \sum_{i=1}^K \sum_{j=1}^N \int_{\Omega}(g_j-c_{i,j})^2 \omega u_i dx 
+ \sum_{i= 1}^K \int_{\Omega} |\nabla u_i| dx,
\end{equation}
where $u_i$ and $\omega$ are defined in \eqref{our-model-inpainting-res}.

%-------------------------------------------------------------------
\section{The AM Algorithm}\label{sec:solve}
%------------------------------------------------------------------
We first transfer \eqref{our-model-general} by relaxing $u_i$ to the following version
\begin{align}
E(u_i, c_i, g) {} = {} &  \mu \Phi(f, {\cal A}g)
+ \lambda\sum_{i=1}^K \int_{\Omega} (g-c_i)^2 \omega u_i dx 
+ \sum_{i=1}^K \int_{\Omega} |\nabla u_i| dx,  \label{our-model-multiphase-con} \\
{\rm s.t.} \quad &   \sum_{i=1}^K u_i(x) =1, u_i(x) \ge 0, \forall x\in \Omega. 
\label{our-model-multiphase-res}
\end{align}

Using the AM algorithm, a partial minimizer $(g, c_i, u_i)$ of \eqref{our-model-multiphase-con} can be 
computed alternatively as follows:
\begin{itemize}
\item[i.] Find $g$ as minimizer of \eqref{our-model-multiphase-con} for fixed $u_i$ and $c_i$.
Obviously, $g$ is only contained in the first two terms of \eqref{our-model-multiphase-con},
and the second term can be regarded as one kind of Tikhonov regularizations when 
solving $g$, see \cite{T}.
The algorithm to find $g$ depends on the choice of $\Phi(f, {\cal A}g)$. For example when 
$\Phi(f, {\cal A}g) =  \int_{\Omega}(f - {\cal A} g)^2 \omega dx$, since it is differentiable, 
we have
\begin{equation} \label{sol_g}
g = (\mu {\cal A}^T{\cal A} + \lambda)^{-1} \big (\mu {\cal A}^T f 
+ \lambda \sum_{i=1}^K c_i u_i \big)\omega.
\end{equation}
For solving $g$ according to the choice of $\Phi(f, {\cal A}g)$ to Poisson noise or impulsive noise, 
we leave this problem in our future work. 

\item[ii.] Find $c_i$ as minimizer of \eqref{our-model-multiphase-con} for fixed $u_i$ and $g$. 
Let $c = (c_1, \cdots, c_K)$, clearly, $c_i$ is just related with the second term 
of \eqref{our-model-multiphase-con}, therefore
\begin{equation} \label{sol_c}
c_i = \frac{\int_{\Omega} g \omega u_i dx}{\int_{\Omega}\omega u_i dx}.
\end{equation}

\item[iii.] Find $u_i$ as minimizer of \eqref{our-model-multiphase-con} for fixed $g$ and $c_i$. 
The discussion of this is given in the following.
\end{itemize}

Note that when $g$ is fixed, the first term of model \eqref{our-model-multiphase-con}
is constant, hence the problem of finding $u_i$ is reduced to minimize
model \eqref{pcms:convex} with $\omega$.  Therefore, 
there are many methods we can use, for example the ADMM method in 
\cite{BPCPE,GO,HSHMS} which will be given explicitly in the following. Alternatively, one can apply 
the primal-dual algorithm \cite{CP,PCCB} or the max-flow approach \cite{YBTBmul}. 

Let $u(j) = (u_i(j))_{i=1}^K, s = (s_i)_{i=1}^K = \big( (g-c_i)^2 \omega \big)_{i=1}^K$, 
then our problem can be transferred to be
\begin{equation}  \label{our-model-multiphase-u}
\min_{v, u, d}   \lambda \langle v, s \rangle + \|d\|_1 + \iota_S(u), \quad 
{\rm s.t.}  \quad \nabla v = d, v = u, 
\end{equation}
where $\iota_S(\cdot)$ is the indicator function defined as
\[
\iota_S(y) := 
\begin{cases}
0, & {\rm if } \ y\in S, \\
+\infty, & {\rm otherwise},
\end{cases}
\]
and $S := \{y \in \mathbb{R}^K | \sum_{i=1}^K y_i = 1, y\ge 0 \}$. 
Then iterate the following steps until converge
\begin{equation}\label{sol_u}
\begin{split}
v^{k+1}  {}={} &  \argmin_v \Big\{ \lambda \langle v, s \rangle  
	+ \sigma (\|b_d^k + \nabla v - d^k\|^2 +  \|b_u^k + v - u^k\|^2)\Big\}, \\
d^{k+1}  {}= {}&  \argmin_d \Big\{ \|d\|_1 + \sigma \|b_d^k + \nabla v^{k+1} - d\|^2\Big\}, \\
u^{k+1}  {}= {}&  \argmin_d \Big\{ \iota_S(u) + \sigma \|b_u^k + v^{k+1} - u\|^2\Big\}, \\
b_d^{k+1}  {}= {}& b_d^{k} + \nabla v^{k+1} - d^{k+1}, \\
b_u^{k+1} {} = {}& b_u^{k} + v^{k+1} - u^{k+1}.
\end{split}
\end{equation}
After $u = (u_i)_{i=1}^K$ is solved, each segment $\Omega_i$ can be obtained by 
\begin{equation} \label{sol_omega}
\Omega_i = \Big \{x| u_i(x) = \max \big\{u_1(x), \cdots, u_K(x) \big\}, 
	\forall x\in \Omega \Big\}.
\end{equation}
In summary, the AM algorithm to solve model \eqref{our-model-multiphase-con} is given in 
Algorithm \ref{alg:solve_our_model}.

\begin{algorithm}[t]
  \caption{The AM Algorithm to Model \eqref{our-model-multiphase-con}}
  \label{alg:solve_our_model}
\begin{algorithmic}[1]
 \REQUIRE Observed image $f$, number of phases $K$, $c^{(0)}$, and $u^{(0)}$.
  \WHILE{$\|c^{(k+1)} - c^k\| > \epsilon$}
  \STATE  find $g^{(k+1)}$ as minimizer of \eqref{our-model-multiphase-con} for fixed $u^{(k)}$, and $c^{(k)}$ using 
    \eqref{sol_g};
  \STATE  find $c^{(k+1)}$ as minimizer of \eqref{our-model-multiphase-con} for fixed $g^{(k+1)}$ and $u^{(k)}$
  	using \eqref{sol_c};
  \STATE find $u^{(k+1)}$ as minimizer of \eqref{our-model-multiphase-con} for fixed $g^{(k+1)}$ and $c^{(k+1)}$
     using \eqref{sol_u};
  \ENDWHILE
  \RETURN  $\Omega_i, i=1, \ldots, K$ using \eqref{sol_omega}\;
  \end{algorithmic}
\end{algorithm}

%------------------------------------------------------------------
\section{Convergence Analysis} \label{sec:converge}
%------------------------------------------------------------------
In this section, we discuss the convergence property of Algorithm 1. 
We first give the very general conclusion to the AM algorithm for three variables. 
Let $X \subset \mathbb{R}^{m_1}, Y\subset \mathbb{R}^{m_2}$ and $Z\subset \mathbb{R}^{m_3}$ 
be closed sets, and the energy function
$ E: X\times Y\times Z \rightarrow \mathbb R$ be continuous 
and bounded from below.
To process the AM algorithm, we start with some initial guess 
$y^{(0)}, z^{(0)}$, then one successively obtains the alternating sequence 
(between $z, y$ and $x$) of conditional minimizers
\[
z^{(0)}, y^{(0)} \rightarrow x^{(0)} \rightarrow z^{(1)} \rightarrow y^{(1)}
 \rightarrow x^{(1)} \rightarrow \cdots
\]
from solving, for $k = 0, 1, \ldots$, 
\begin{equation} \label{AM-ite}
\begin{split}
 x^{(k)} &\in \argmin_x E (x, y^{(k)}, z^{(k)}),  \\
 z^{(k+1)} &\in \argmin_z E (x^{(k)}, y^{(k)}, z),     \\
 y^{(k+1)} &\in \argmin_y E (x^{(k)}, y, z^{(k+1)}). 
\end{split}
\end{equation}

\begin{theorem}\label{AM-mono} \textbf{(Monotonicity of Alternating Minimization).}
Let $X \subset \mathbb{R}^{m_1}, Y \subset \mathbb{R}^{m_2}$ and 
$Z \subset \mathbb{R}^{m_3}$ be closed sets, and 
$E: X\times Y\times Z \rightarrow \mathbb R$ be continuous and bounded from below. 
Then, for each $k \ge 0$, the following relations are satisfied
\begin{eqnarray*}
E(x^{(k)}, y^{(k+1)}, z^{(k+1)}) &\le& E(x^{(k-1)}, y^{(k)}, z^{(k)}), \\
E(x^{(k)}, y^{(k)}, z^{(k+1)}) &\le& E(x^{(k-1)}, y^{(k-1)}, z^{(k)}), \\
E(x^{(k+1)}, y^{(k+1)}, z^{(k+1)}) &\le& E(x^{(k)}, y^{(k)}, z^{(k)}).
\end{eqnarray*}
Hence, the sequence $\big \{E(x^{(k)}, y^{(k)}, z^{(k)})_{k \in \mathbb N} \big \}$ converges 
monotonically.
\end{theorem}

\begin{proof} See the Appendix.
\end{proof}

\begin{theorem} \label{AM-cong}
Let $X \subset \mathbb{R}^{m_1}, Y \subset \mathbb{R}^{m_2}$ and 
$Z \subset \mathbb{R}^{m_3}$ be closed sets, and 
$E: X\times Y\times Z \rightarrow \mathbb R$ be continuous 
and bounded from below. Then, for any convergent subsequence 
$(x^{(k_i)}, y^{(k_i)}, z^{(k_i)})_{i\in \mathbb N}$ of $(x^{(k)}, y^{(k)}, z^{(k)})_{k\in \mathbb N}$ 
generated from formula \eqref{AM-ite} with
\[
(x^{(k_i)}, y^{(k_i)}, z^{(k_i)}) \longrightarrow (x^*, y^*, z^*),  \ {\rm as} \ i\rightarrow \infty,
\]
the following relations are satisfied:
\begin{equation} \label{partial-min}
\begin{split}
E(x^*, y^*, z^*) \leq E(x, y^*, z^*) &\quad   \forall x\in X,  \\
E(x^*, y^*, z^*) \leq E(x^*, y, z^*) & \quad  \forall y\in Y,  \\
E(x^*, y^*, z^*) \leq E(x^*, y^*, z) &  \quad \forall z\in Z.
\end{split}
\end{equation}
$(x^*, y^*, z^*)$ is called the the partial minimizer of $E(\cdot, \cdot, \cdot)$ if \eqref{partial-min} is satisfied.
\end{theorem}

\begin{proof} See the Appendix
\end{proof}

Let $\cal O$ be the set of all the partial minimizers of model 
\eqref{our-model-multiphase-con} defined in Theorem \ref{AM-cong}.
The following theorem \ref{AM-cong-our} gives the convergence property of Algorithm 1.

\begin{theorem} \label{AM-cong-our}
Assume the operator $\cal A$ in model \eqref{our-model-multiphase-con} is a continuous 
mapping and $\Phi(f, {\cal A}g)$ is continuous and nonnegative. As $k\rightarrow \infty$, 
if $(u^{(k)}, g^{(k)}, c^{(k)}) \rightarrow (u^*, g^*, c^*)$, then $(u^*, g^*, c^*) \in {\cal O}$.
If $(u^{(k)}, g^{(k)}, c^{(k)})_{k\in \mathbb N}$ does not converge, it must contain a convergent 
subsequence and every convergent subsequence converges to a partial minimizer
of model \eqref{our-model-multiphase-con}.
\end{theorem}

\begin{proof} See the Appendix.
\end{proof}

%-------------------------------------------------------------------
\section{Experimental Results}\label{sec:experiments}
%-------------------------------------------------------------------

We compare our segmentation model \eqref{our-model-multiphase-con}
with three state-of-the-art multiphase segmentation methods \cite{CCZ,HSHMS,YBTBmul}
for different phases synthetic and real-world images corrupted by noise, blur and missing pixels.
More precisely, methods \cite{YBTBmul} and  \cite{HSHMS} minimize model \eqref{pcms:convex} 
by using the max-flow approach and the ADMM algorithm, respectively. The difference
between methods \cite{YBTBmul} and  \cite{HSHMS} is that method \cite{YBTBmul}
minimizes $u_i$ with fixed $c_i$, while method \cite{HSHMS}
minimizes $u_i$ and $c_i$ both.
Method \cite{CCZ} is a  two-stage segmentation method, which solves a convex variant of 
the Mumford-Shah model \eqref{ms} first and a thresholding technique followed.
Moreover, method \cite{CCZ} is very effective in segmenting general kind of images including blurry 
images. All the codes of methods \cite{CCZ,HSHMS,YBTBmul} are provided by the authors, 
and the parameters in them are chosen by trial and error to give the best results of the
respective methods. Note that all the methods \cite{CCZ,HSHMS,YBTBmul} can only 
segment gray images. In order to compare the effectiveness of our 
model \eqref{our-model-multiphase-con} with 
model \eqref{pcms:convex} in color images, we first extend model \eqref{pcms:convex}
using the strategy in \eqref{our-model-vector} so that it can handle color images,
then method \cite{HSHMS} will be adopted to solve the extended model \eqref{pcms:convex}.
That means the comparison in color images will be executed between our method and the extended 
method \cite{HSHMS}.

The initial codebook $c_i$ for methods \cite{HSHMS,YBTBmul} are computed by 
the fuzzy C-means method \cite{BEF84} with 100 iteration steps, and the thresholds chosen 
in the thresholding technique of method \cite{CCZ} is using the automatic strategy therein. 
The tolerance $\epsilon$ and the step size $\sigma$ respectively in Algorithm 1 and 
\eqref{sol_u} are fixed to be $10^{-4}$ and $2$. The parameters $\lambda$ and 
$\mu$ in model \eqref{our-model-multiphase-con} are chosen empirically.

For adding Gaussian noise to the given image $f\in [0,1]$, we apply the MATLAB 
command {\tt imnoise} with zero mean and different variances for different
kinds of images. In particular, the variance used to add noise into the blurry images is fixed 
to be $10^{-4}$. If there is no special explanation, to blur an image, the Gaussian kernel used is 
size $15 \times 15$ with standard deviation 15 and the motion kernel is 15 pixels with an 
angle of 90 degrees. We apply the MATLAB command {\tt rand} to remove some information 
from the corrupted images randomly, and set the percentage of information lost to be $40\%$
unless special explanation. The {\it segmentation accuracy} (SA) defined as
\[
\textrm{SA} = \frac{\# \textrm{correctly classified pixels}}{\# \textrm{all pixels}} \times 100,
\]
which will be used to evaluate the accuracy of the involved methods in detail.
All the results were tested on a MacBook with 2.4 GHz processor and 4GB RAM. 

\begin{figure*}[!htb]
\begin{center}
\begin{tabular}{ccccc}
\includegraphics[width=28mm, height=28mm]{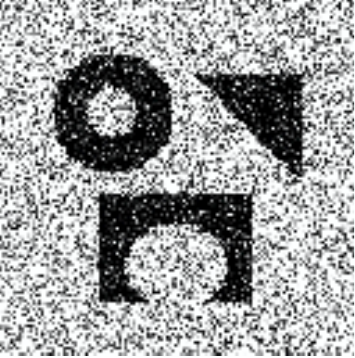}  &
\includegraphics[width=28mm, height=28mm]{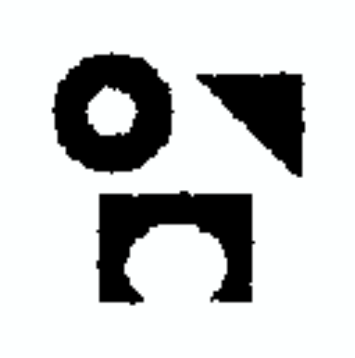}  &
\includegraphics[width=28mm, height=28mm]{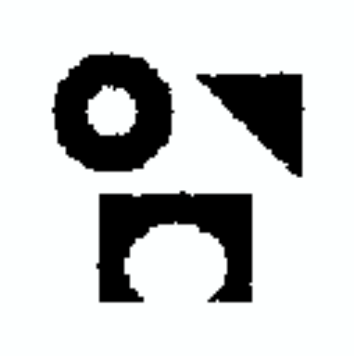}  &
\includegraphics[width=28mm, height=28mm]{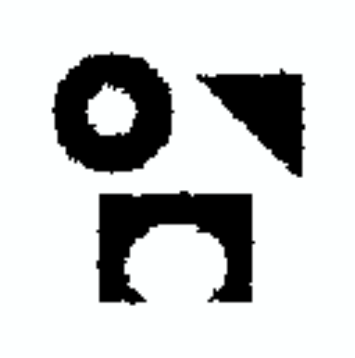}  &
\includegraphics[width=28mm, height=28mm]{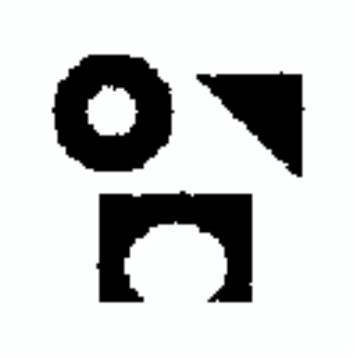} \\
{\small (A1)} &{\small (B1) \cite{YBTBmul} (99.50) }  & {\small (C1) \cite{HSHMS} (99.64) } & 
{\small (D1) \cite{CCZ} (99.48) } & {\small (E1) Our (99.65) } \\
\includegraphics[width=28mm, height=28mm]{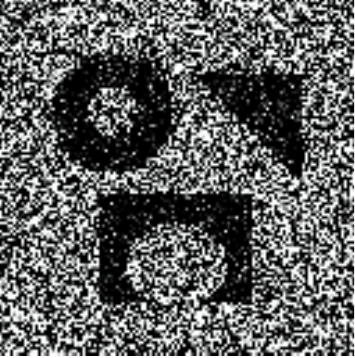} &
\includegraphics[width=28mm, height=28mm]{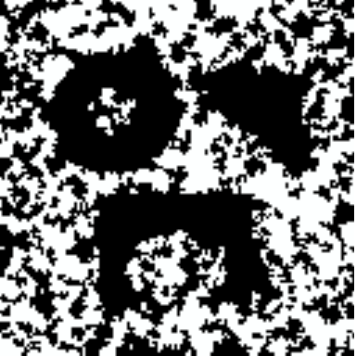} &
\includegraphics[width=28mm, height=28mm]{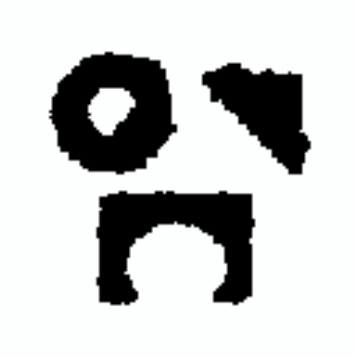} &
\includegraphics[width=28mm, height=28mm]{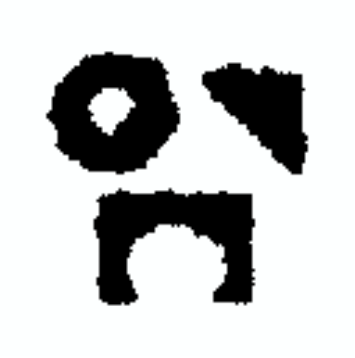} &
\includegraphics[width=28mm, height=28mm]{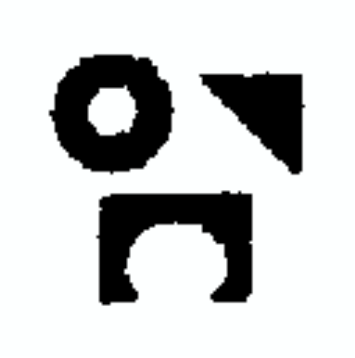}  \\
{\small (A2)} &{\small (B2) \cite{YBTBmul} (64.23) }  & {\small (C2) \cite{HSHMS} (98.13) } & 
{\small (D2) \cite{CCZ} (97.15) } & {\small (E2) Our (99.29) } \\
\includegraphics[width=28mm, height=28mm]{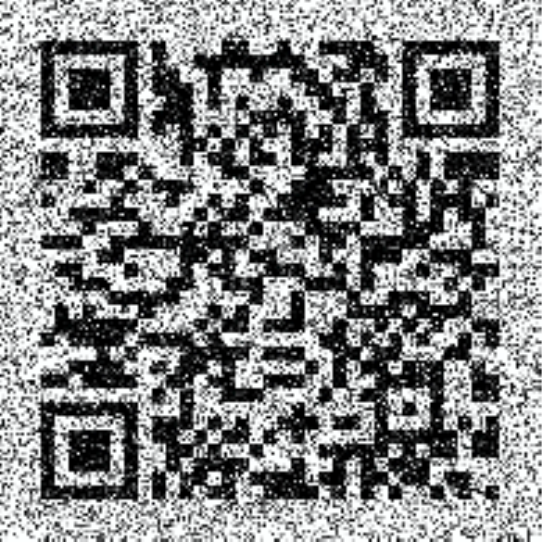}  &
\includegraphics[width=28mm, height=28mm]{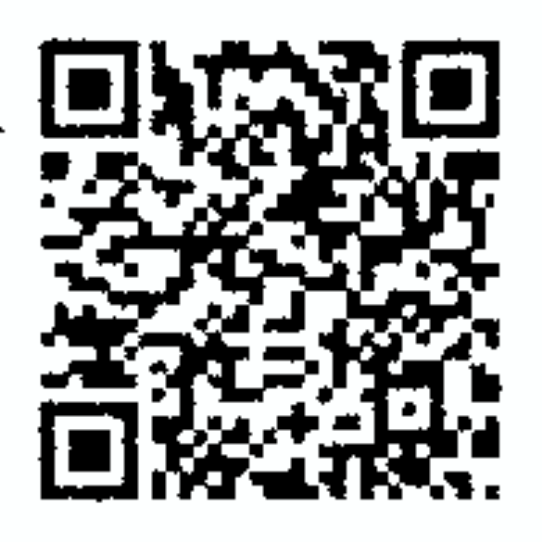}  &
\includegraphics[width=28mm, height=28mm]{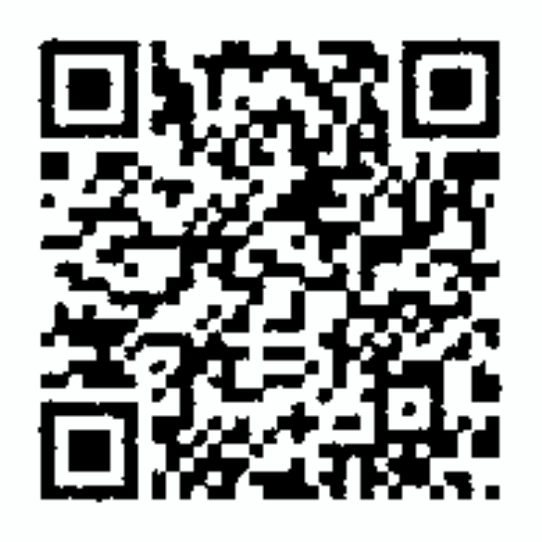}  &
\includegraphics[width=28mm, height=28mm]{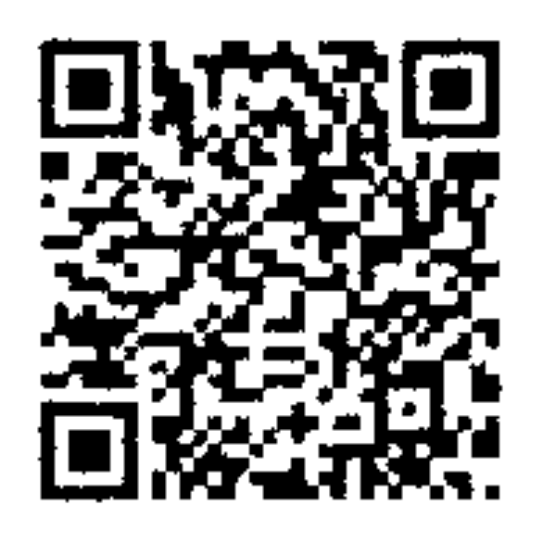}  &
\includegraphics[width=28mm, height=28mm]{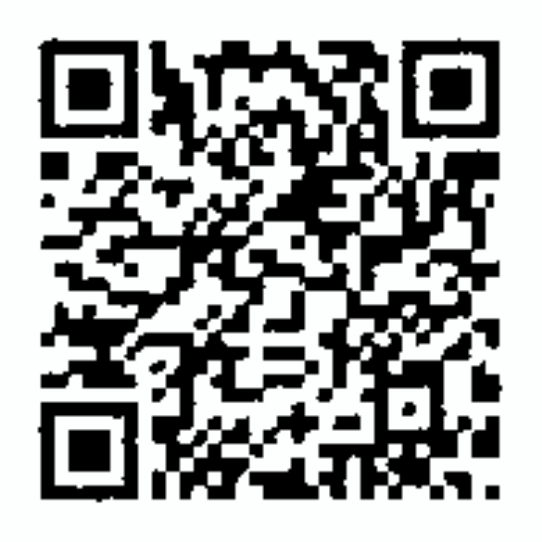} \\
{\small (A3)} &{\small (B3) \cite{YBTBmul} (97.91) }  & {\small (C3) \cite{HSHMS} (98.37) } & 
{\small (D3) \cite{CCZ} (98.08) } & {\small (E3) Our (98.43) } \\
\includegraphics[width=28mm, height=28mm]{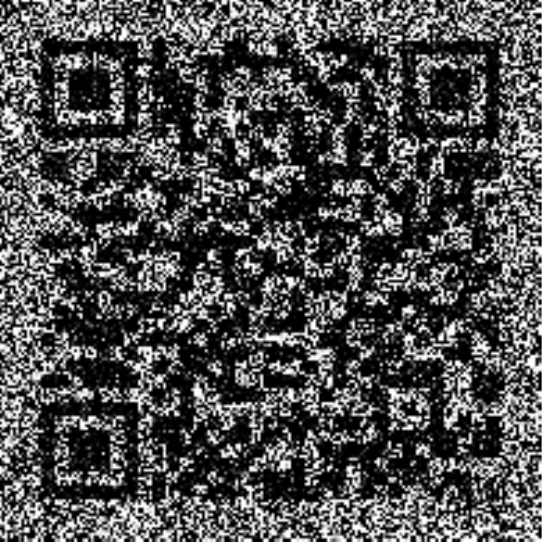} &
\includegraphics[width=28mm, height=28mm]{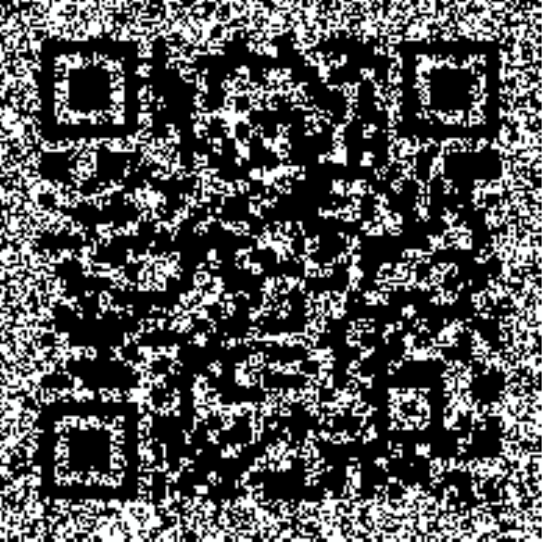} &
\includegraphics[width=28mm, height=28mm]{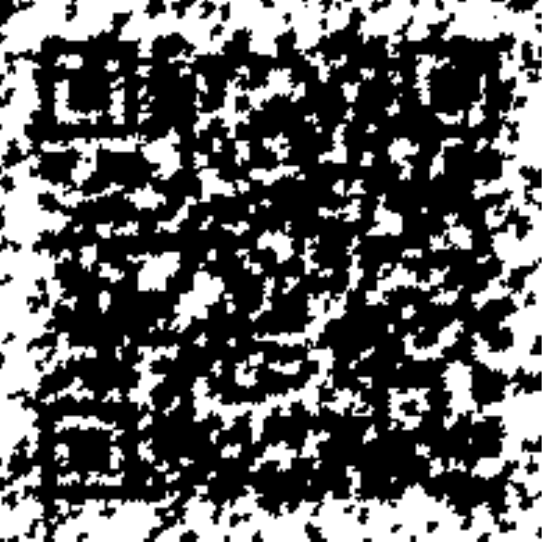} &
\includegraphics[width=28mm, height=28mm]{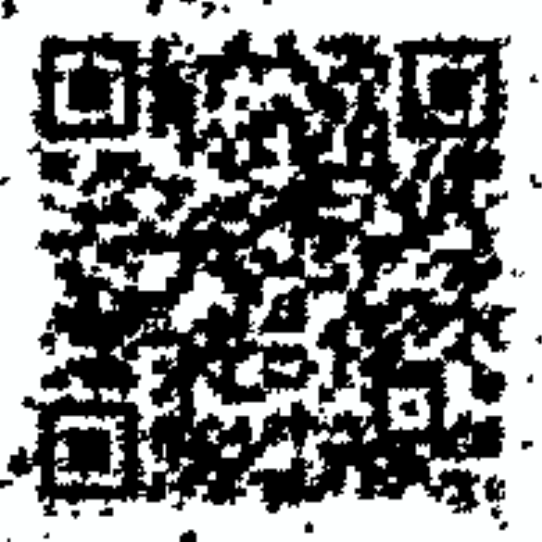} &
\includegraphics[width=28mm, height=28mm]{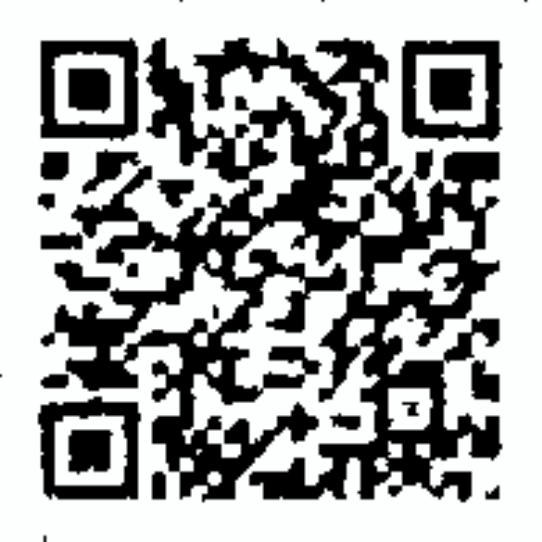}  \\
{\small (A4)} &{\small (B4) \cite{YBTBmul} (68.27) }  & {\small (C4) \cite{HSHMS} (74.28) } & 
{\small (D4) \cite{CCZ} (86.11) } & {\small (E4) Our (95.66) } 
\end{tabular}
\end{center}
\caption{Segmentation of two-phase synthetic images ($128\times 128$ and $195\times 195$). 
(A1) and (A3): the given noisy images; (A2) and (A4): the given noisy images with 
$40\%$ information lost; Columns two to five: the results of methods \cite{YBTBmul,HSHMS,CCZ}
and our method, respectively. Numbers in braces are the {\it segmentation accuracy}.
}\label{twophase-syn}
\end{figure*}

\subsection{Gray image segmentation}
{\it Example 1: two-phase synthetic images.} 
To illustrate the ability of our method in high level noisy images and images with information
lost, we first test it in two two-phase synthetic images, i.e. one contains different 
shapes, and the other is the 2D barcode image which represents data relating to the object
it is attached and is the most frequently used type to scan with smartphones, 
see Fig. \ref{twophase-syn}. Fig. \ref{twophase-syn}(A1)--(A4) give the images corrupted 
by Gaussian noise with variance 0.2, and Fig. \ref{twophase-syn}(A2) and (A4) give the
images with part information removed randomly. The columns two to four of Fig. \ref{twophase-syn}
are the results of methods \cite{YBTBmul,HSHMS,CCZ}, respectively. The last 
column of Fig. \ref{twophase-syn} is the results of our method. From the rows one and three 
of Fig. \ref{twophase-syn}, the results of segmenting  the given noisy images, we see that all 
the methods can give very good results. However, after comparing the {\it segmentation accuracy}
given in the braces under each result, 
we get that our method gives the highest SA compared with others three methods. That means
our model \eqref{our-model-multiphase-con} can really improve the segmentation accuracy 
compared with model \eqref{pcms:convex}.
From the second and the fourth rows of Fig. \ref{twophase-syn}, the results of segmenting the 
given images with part information lost, we can easily see that only our method gives good results both 
in visual and {\it segmentation accuracy}.

\begin{figure*}[!htb]
\begin{center}
\begin{tabular}{ccccc}
\includegraphics[width=28mm, height=28mm]{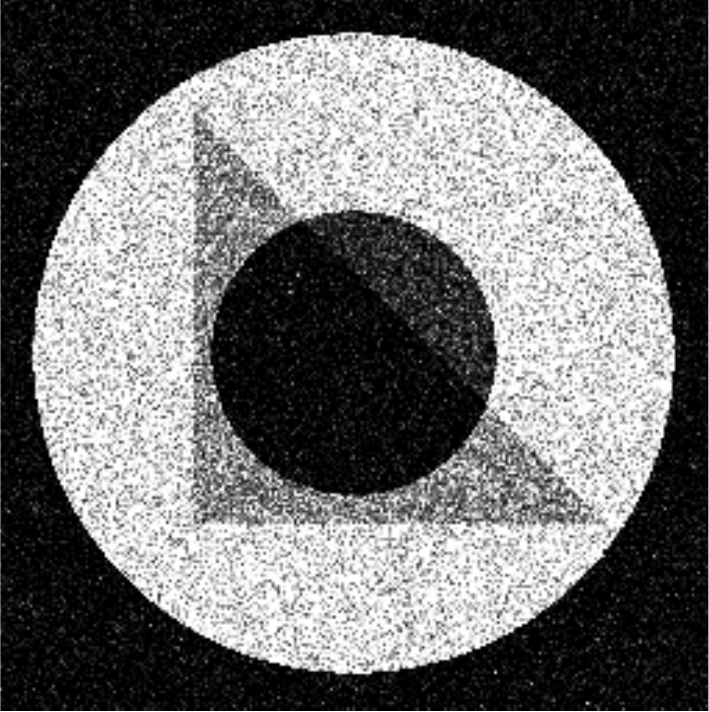}  &
\includegraphics[width=28mm, height=28mm]{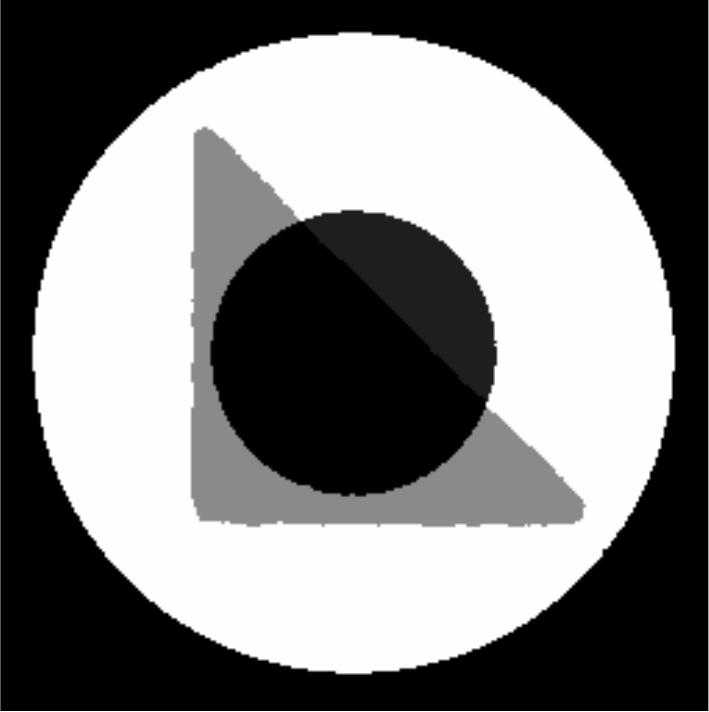}  &
\includegraphics[width=28mm, height=28mm]{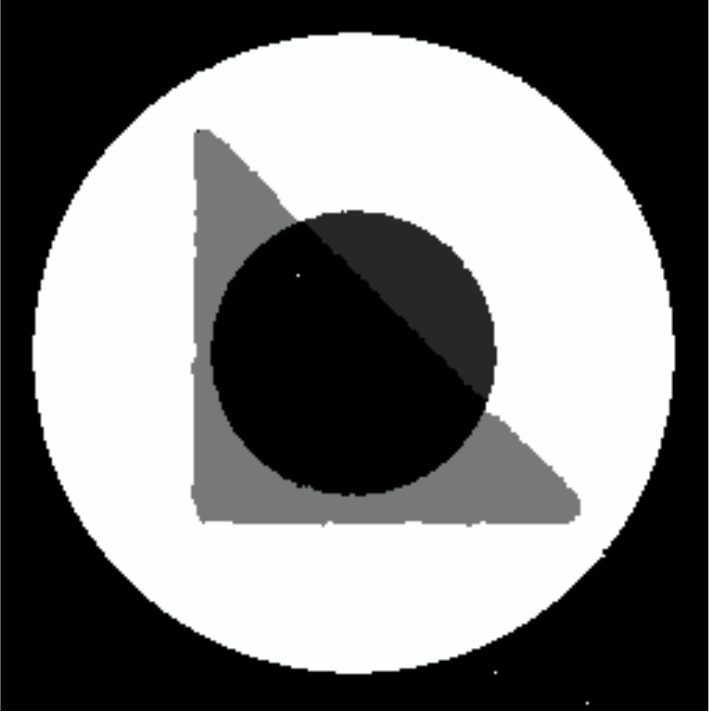}  &
\includegraphics[width=28mm, height=28mm]{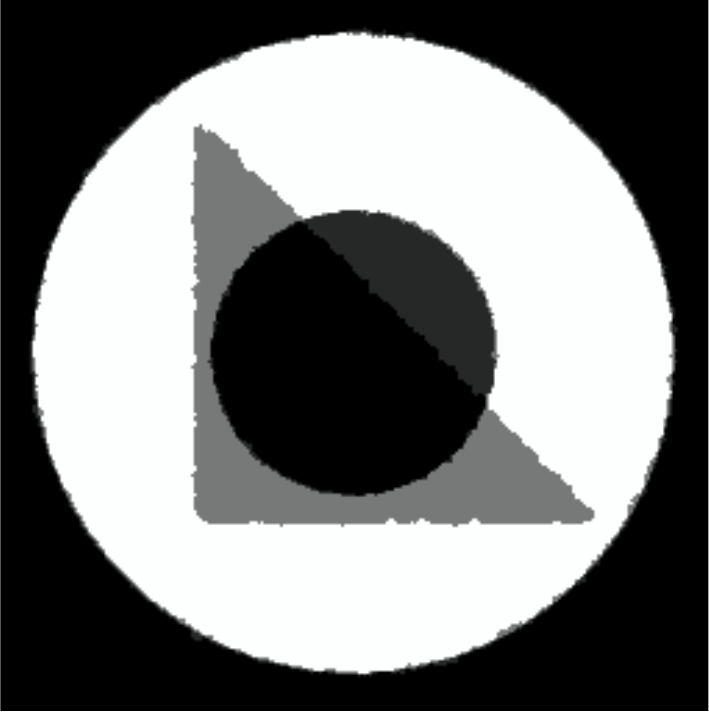}  &
\includegraphics[width=28mm, height=28mm]{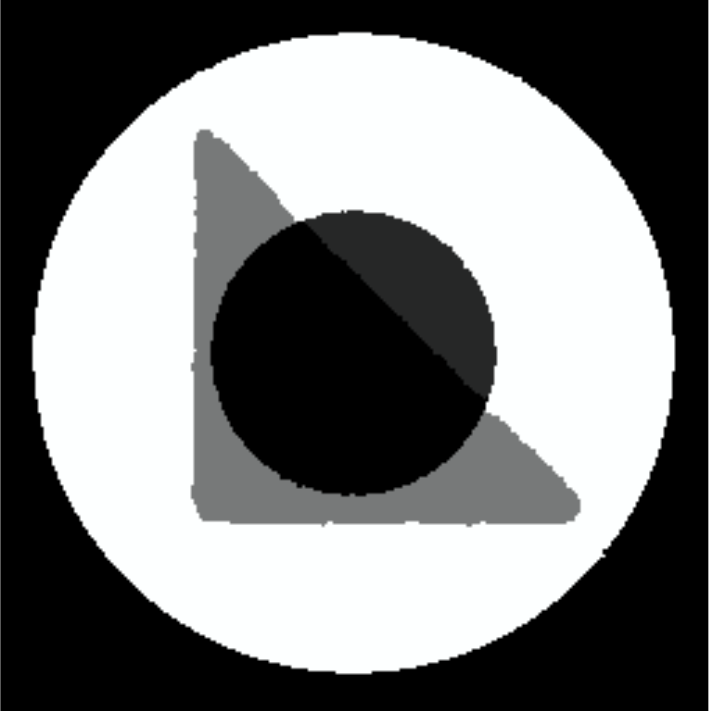} \\
{\small (A1)} &{\small (B1) \cite{YBTBmul} (99.64) }  & {\small (C1) \cite{HSHMS} (99.63) } & 
{\small (D1) \cite{CCZ} (97.96) } & {\small (E1) Our (99.65) } \\
\includegraphics[width=28mm, height=28mm]{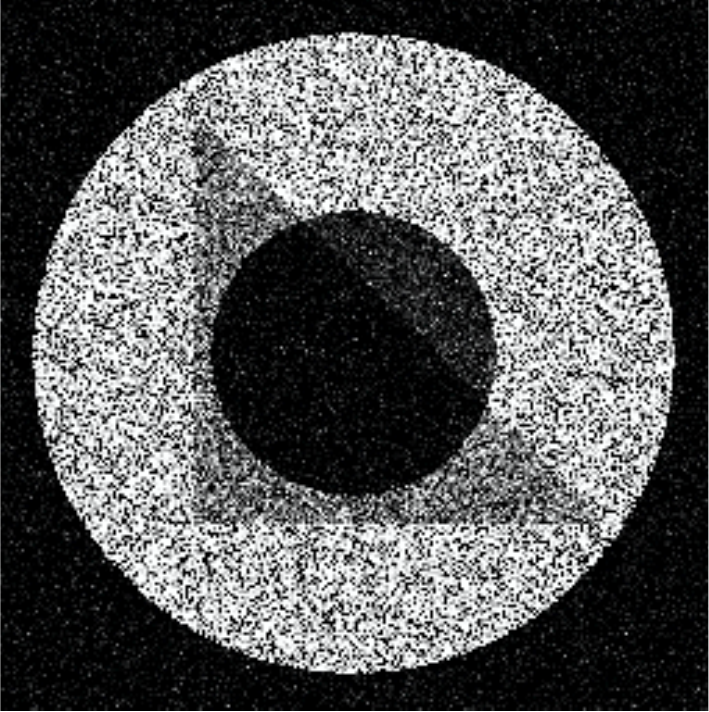} &
\includegraphics[width=28mm, height=28mm]{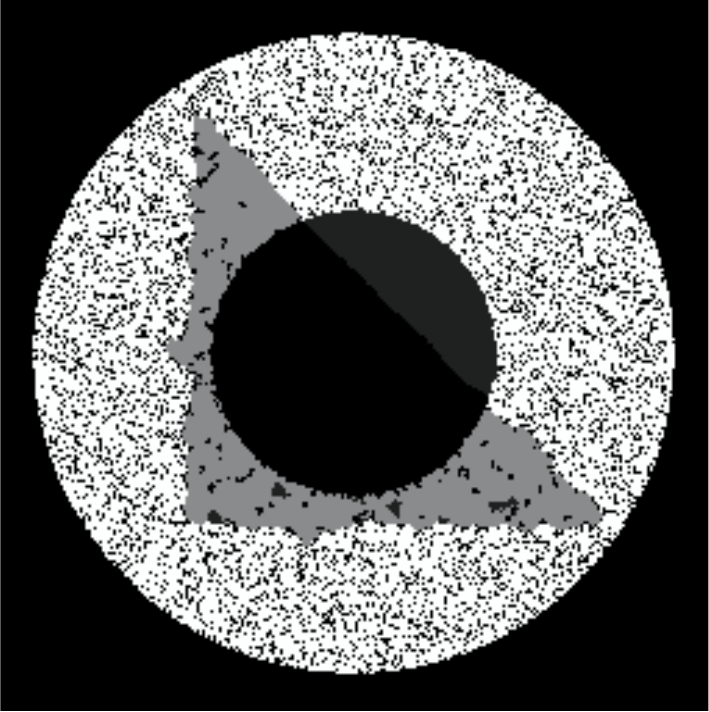} &
\includegraphics[width=28mm, height=28mm]{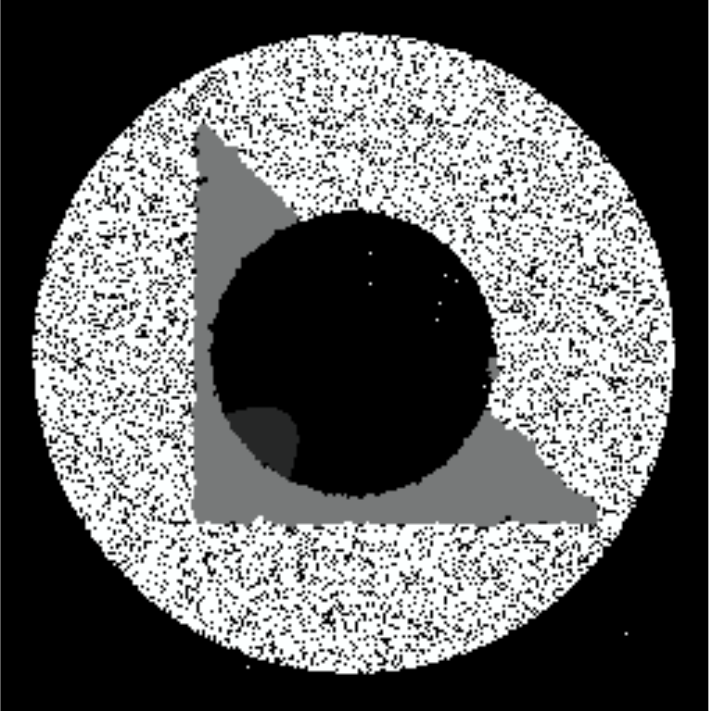} &
\includegraphics[width=28mm, height=28mm]{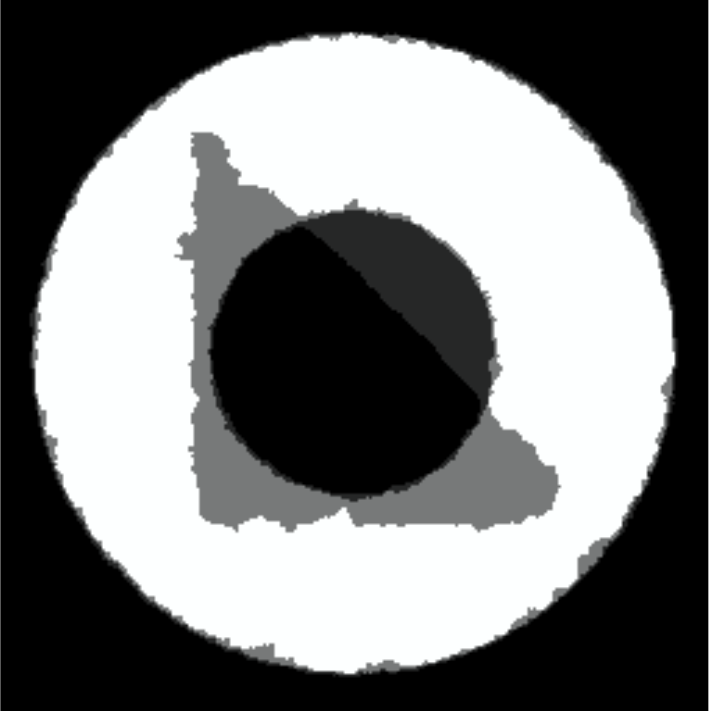} &
\includegraphics[width=28mm, height=28mm]{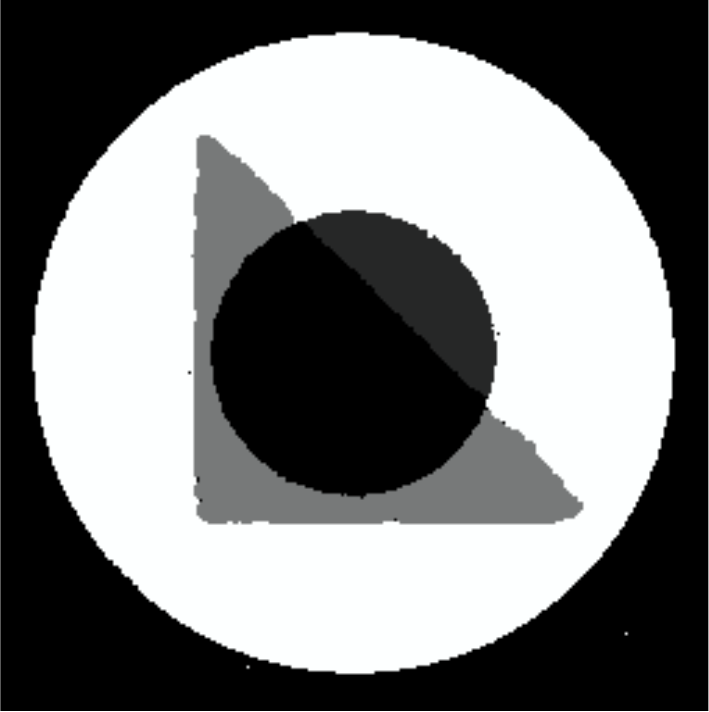}  \\
{\small (A2)} &{\small (B2) \cite{YBTBmul} (75.41) }  & {\small (C2) \cite{HSHMS} (86.89) } & 
{\small (D2) \cite{CCZ} (95.88) } & {\small (E2) Our (99.48) } \\
\includegraphics[width=28mm, height=28mm]{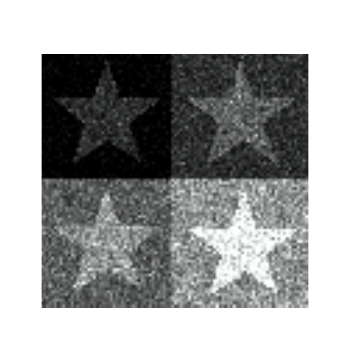}  &
\includegraphics[width=28mm, height=28mm]{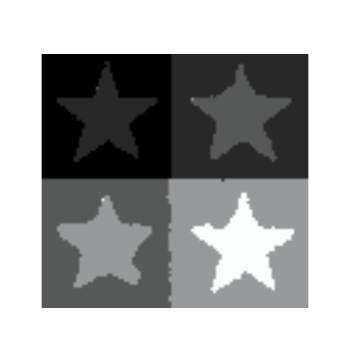}  &
\includegraphics[width=28mm, height=28mm]{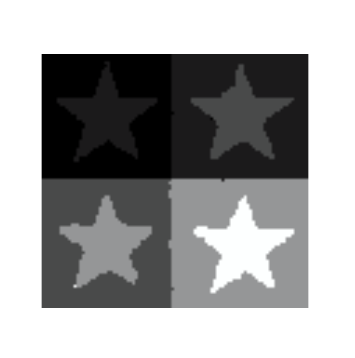}  &
\includegraphics[width=28mm, height=28mm]{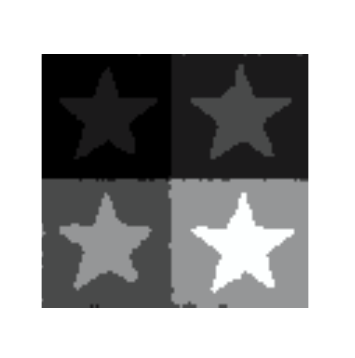}  &
\includegraphics[width=28mm, height=28mm]{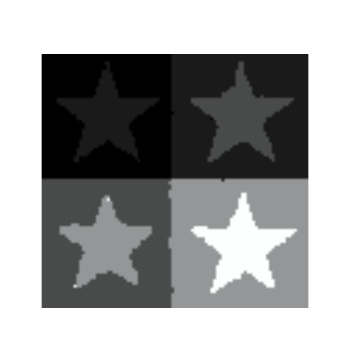} \\
{\small (A3)} &{\small (B3) \cite{YBTBmul} (97.58) }  & {\small (C3) \cite{HSHMS} (98.63) } & 
{\small (D3) \cite{CCZ} (97.83) } & {\small (E3) Our (98.72) } \\
\includegraphics[width=28mm, height=28mm]{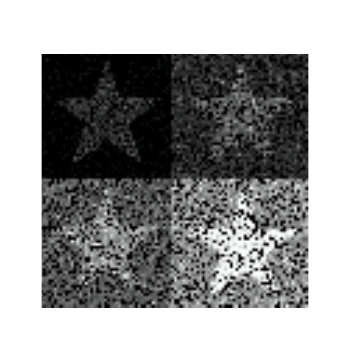} &
\includegraphics[width=28mm, height=28mm]{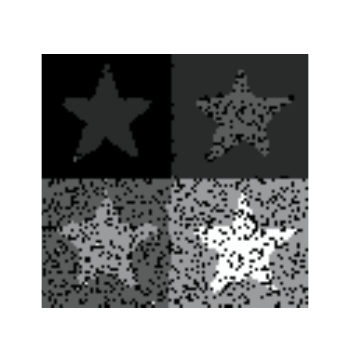} &
\includegraphics[width=28mm, height=28mm]{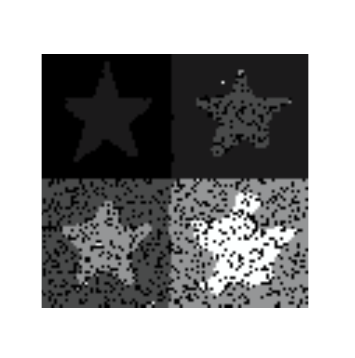} &
\includegraphics[width=28mm, height=28mm]{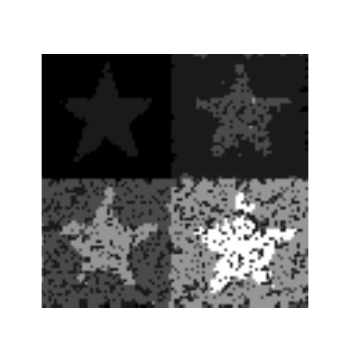} &
\includegraphics[width=28mm, height=28mm]{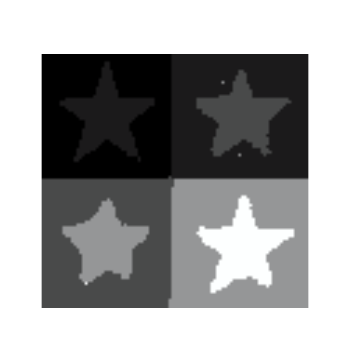}  \\
{\small (A4)} &{\small (B4) \cite{YBTBmul} (85.61) }  & {\small (C4) \cite{HSHMS} (84.17) } & 
{\small (D4) \cite{CCZ} (86.11) } & {\small (E4) Our (97.45) } 
\end{tabular}
\end{center}
\caption{Segmentation of fourphase and fivephase synthetic images 
($256\times 256$ and $91\times 91$). (A1) and (A3): the given noisy images; (A2) and (A4): 
the given noisy images with $20\%$ information lost; Columns two to five: the results of 
methods \cite{YBTBmul,HSHMS,CCZ} and our method, respectively. Numbers in braces 
are the {\it segmentation accuracy}.
}\label{mutiphase-syn}
\end{figure*}

\begin{figure*}[!htb]
\begin{center}
\begin{tabular}{ccccc}
\includegraphics[width=28mm, height=28mm]{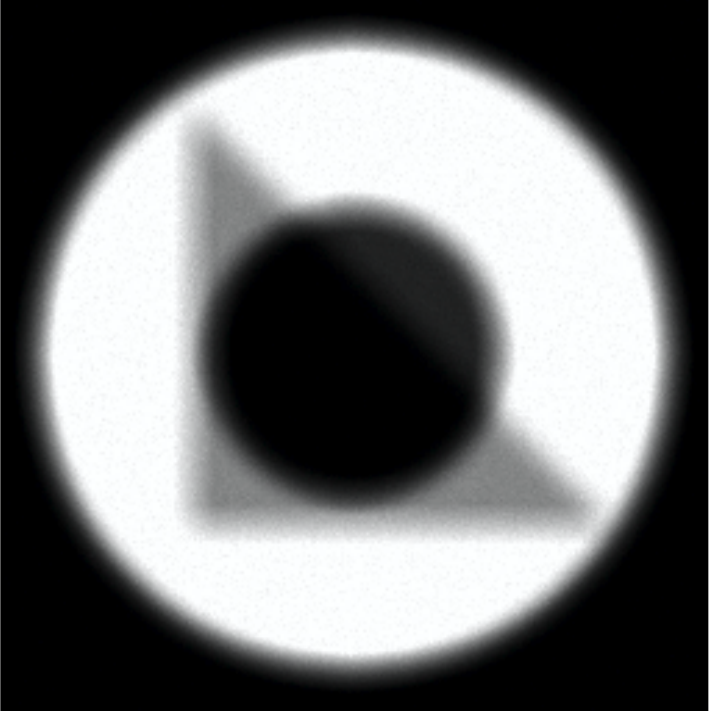}  &
\includegraphics[width=28mm, height=28mm]{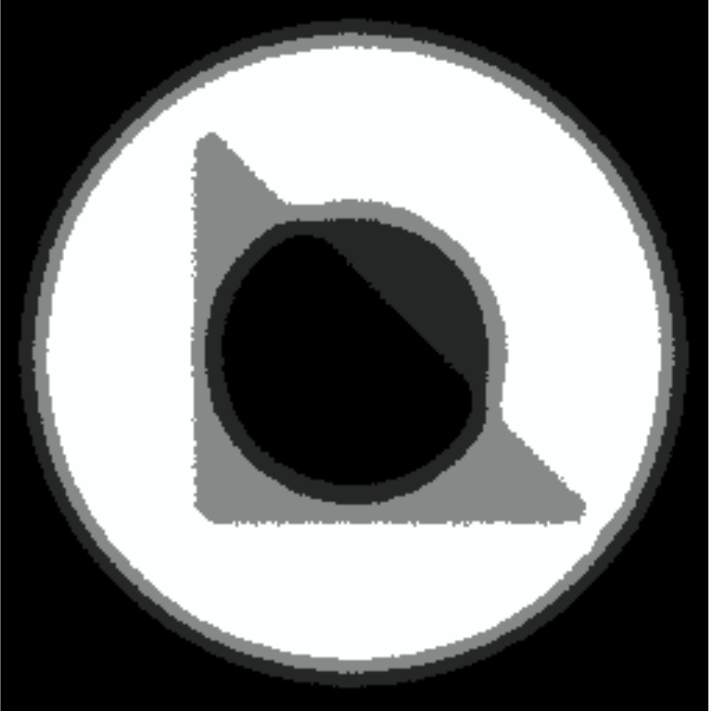}  &
\includegraphics[width=28mm, height=28mm]{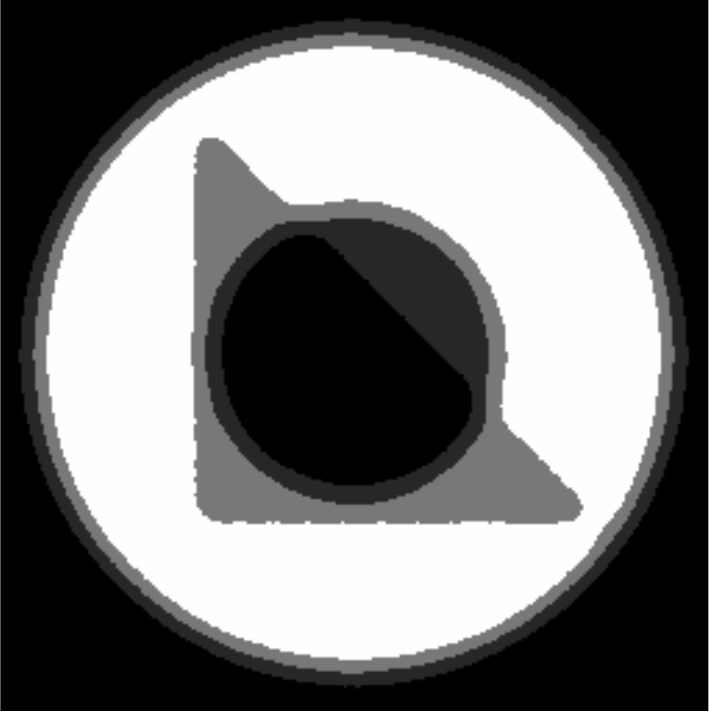}  &
\includegraphics[width=28mm, height=28mm]{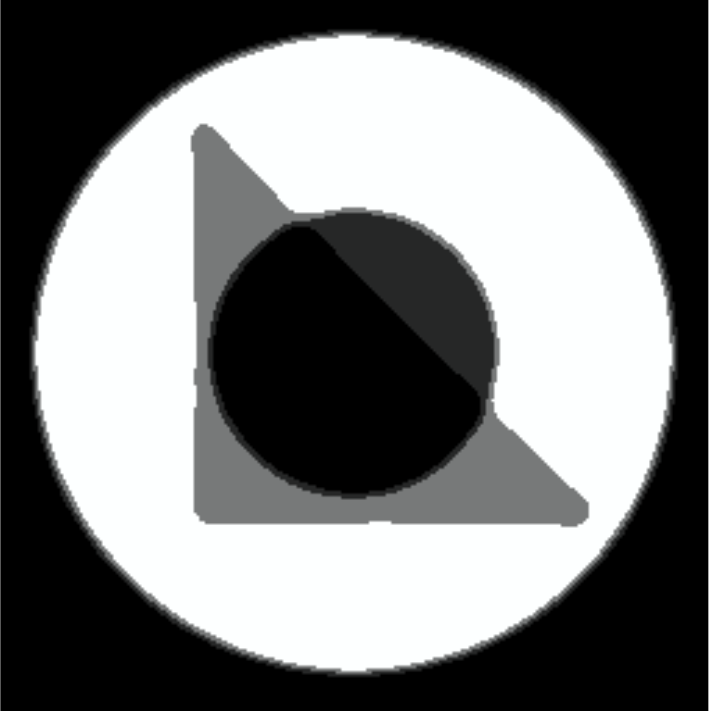}  &
\includegraphics[width=28mm, height=28mm]{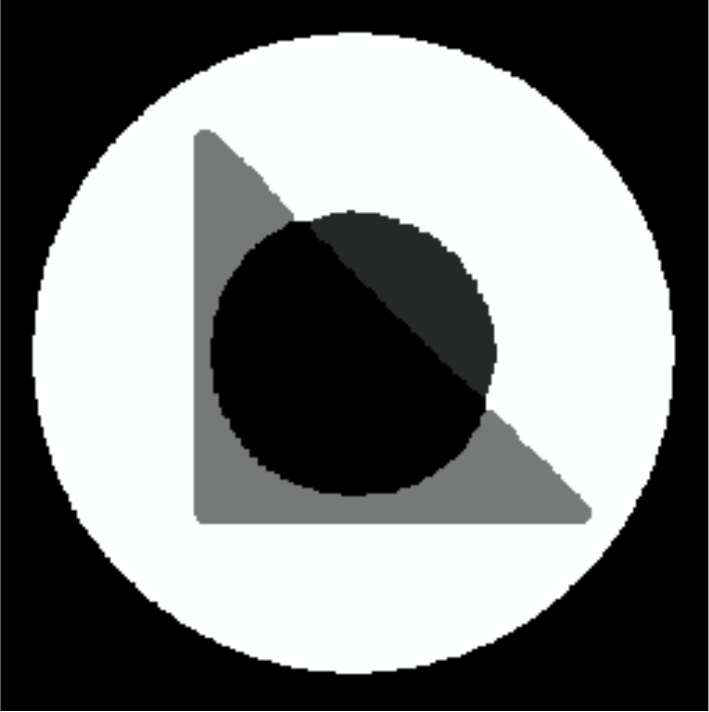} \\
{\small (A1)} &{\small (B1) \cite{YBTBmul} (86.05) }  & {\small (C1) \cite{HSHMS} (86.31) } & 
{\small (D1) \cite{CCZ} (95.61) } & {\small (E1) Our (99.44) } \\
\includegraphics[width=28mm, height=28mm]{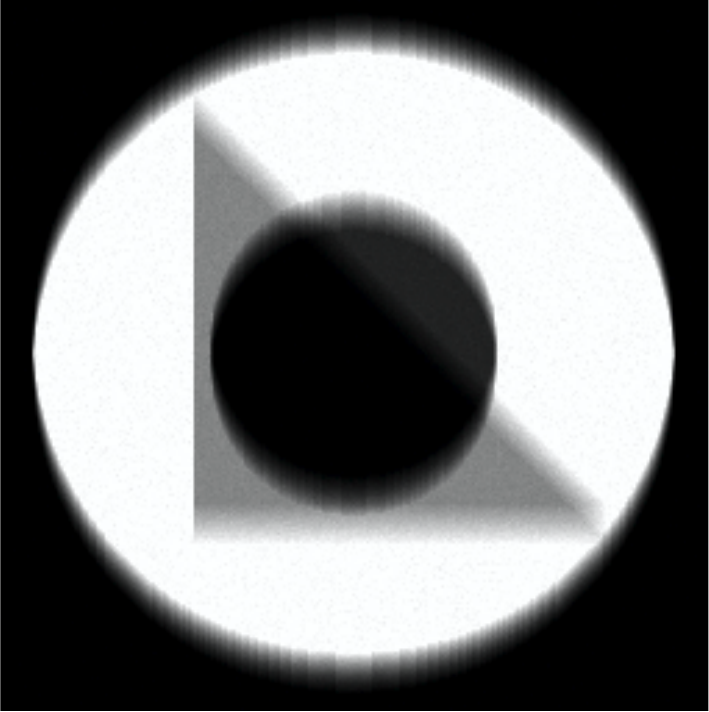}  &
\includegraphics[width=28mm, height=28mm]{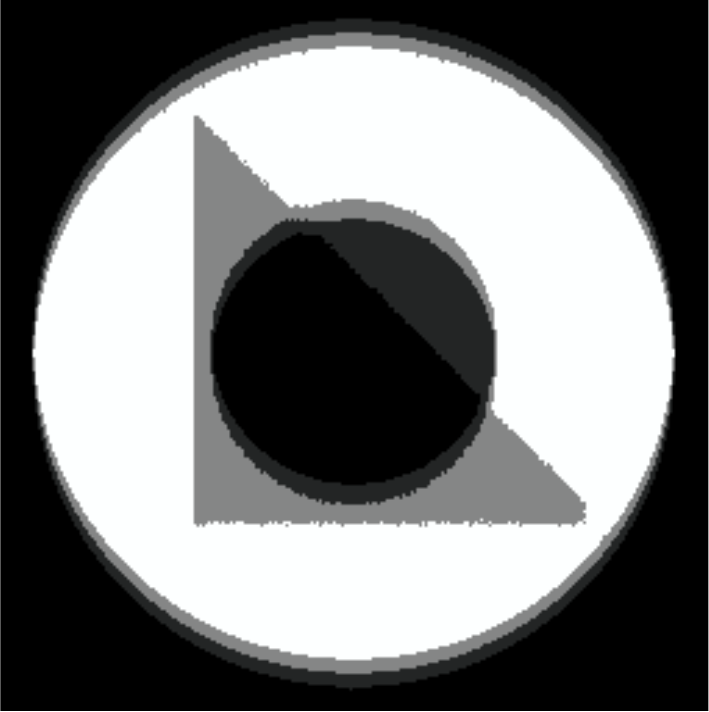}  &
\includegraphics[width=28mm, height=28mm]{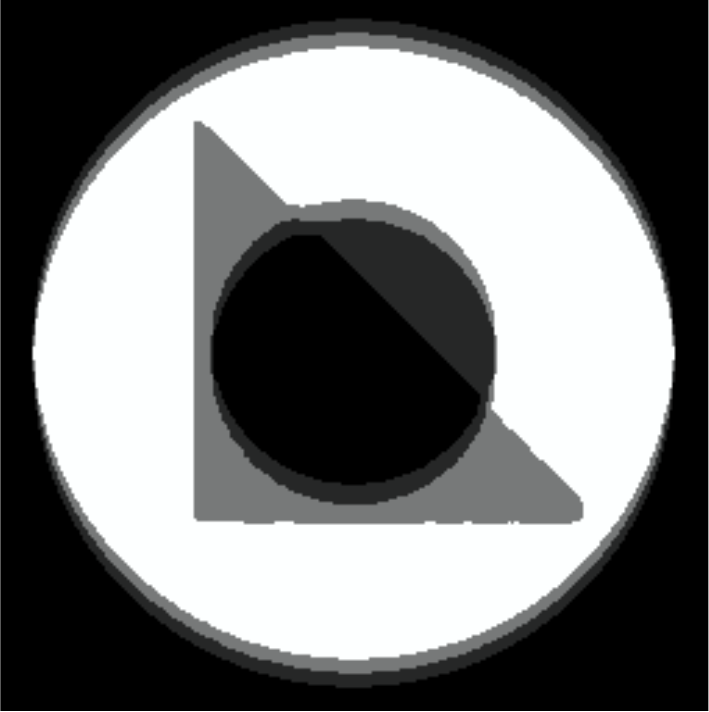}  &
\includegraphics[width=28mm, height=28mm]{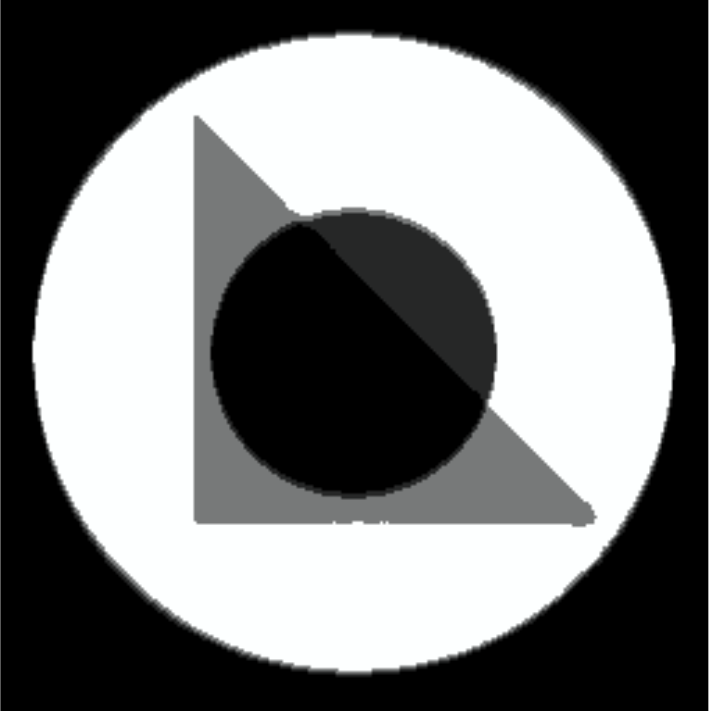}  &
\includegraphics[width=28mm, height=28mm]{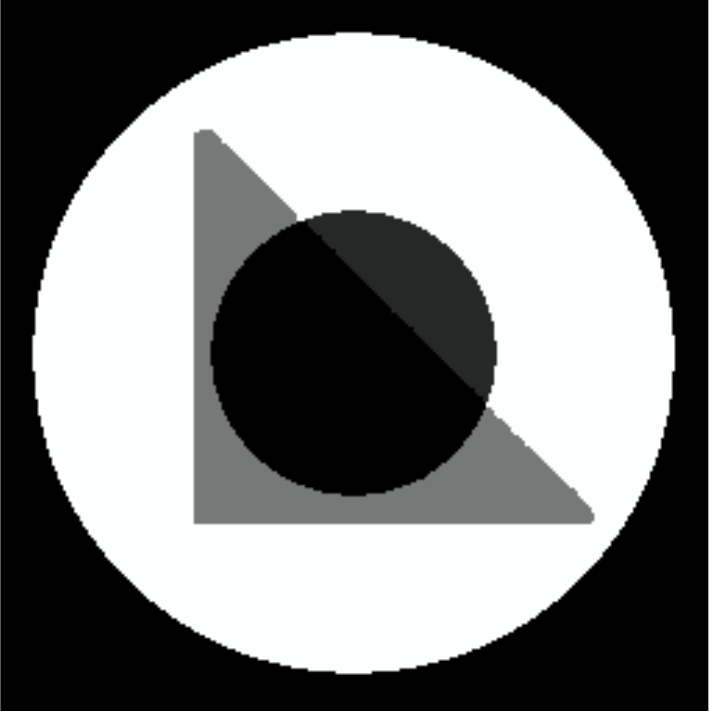} \\
{\small (A2)} &{\small (B2) \cite{YBTBmul} (90.42) }  & {\small (C2) \cite{HSHMS} (90.44) } & 
{\small (D2) \cite{CCZ} (97.24) } & {\small (E2) Our (99.92) } \\
\includegraphics[width=28mm, height=28mm]{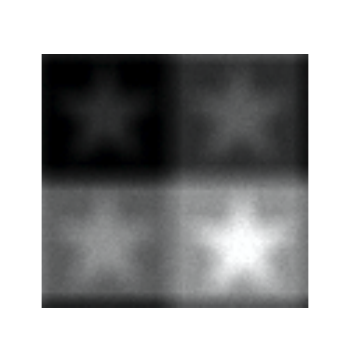}  &
\includegraphics[width=28mm, height=28mm]{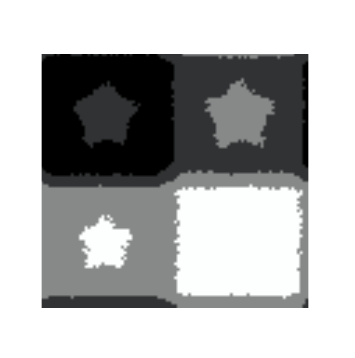}  &
\includegraphics[width=28mm, height=28mm]{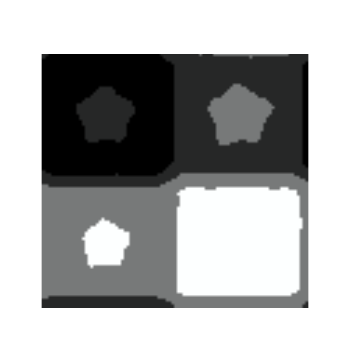}  &
\includegraphics[width=28mm, height=28mm]{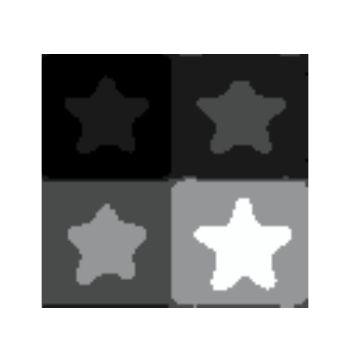}  &
\includegraphics[width=28mm, height=28mm]{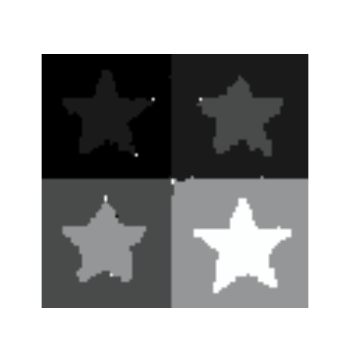} \\
{\small (A3)} &{\small (B3) \cite{YBTBmul} (72.91) }  & {\small (C3) \cite{HSHMS} (72.66) } & 
{\small (D3) \cite{CCZ} (92.66) } & {\small (E3) Our (96.38) } \\
\includegraphics[width=28mm, height=28mm]{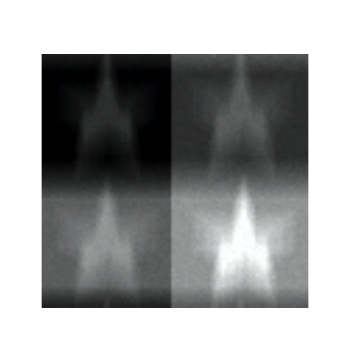}  &
\includegraphics[width=28mm, height=28mm]{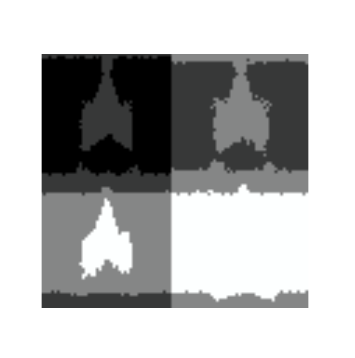}  &
\includegraphics[width=28mm, height=28mm]{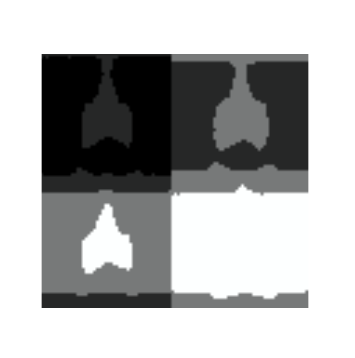}  &
\includegraphics[width=28mm, height=28mm]{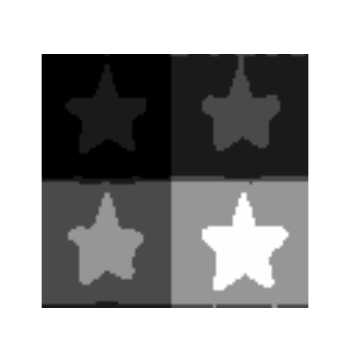}  &
\includegraphics[width=28mm, height=28mm]{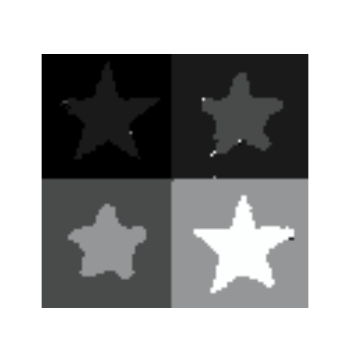} \\
{\small (A4)} &{\small (B4) \cite{YBTBmul} (71.05) }  & {\small (C4) \cite{HSHMS} (71.25) } & 
{\small (D4) \cite{CCZ} (92.53) } & {\small (E4) Our (96.96) } 
\end{tabular}
\end{center}
\caption{Segmentation of fourphase and fivephase synthetic blurry images 
($256\times 256$ and $91\times 91$). (A1) and (A3): the given images with Gaussion blur; 
(A2) and (A4): the given images with motion blur; Columns two to five: the results of 
methods \cite{YBTBmul,HSHMS,CCZ} and our method, respectively. Numbers in braces 
are the {\it segmentation accuracy}.
}\label{mutiphase-syn-blur}
\end{figure*}

{\it Example 2: multiphase synthetic images.} 
Two multiphase synthetic images will be tested in this example, i.e., one is four phases image 
with different shapes inside, and the other is five phases image including stars with different
intensities. In Fig. \ref{mutiphase-syn}(A1)--(A4), the variances used to add noise on the
four phases and five phases images are 0.05 and 0.01, respectively. Moreover, 
Fig. \ref{mutiphase-syn}(A2) and (A4) give the noisy images with $20\%$ of all pixels 
randomly removed. From the results in Fig. \ref{mutiphase-syn}, we can get very similar 
conclusions as those obtained in example 1, i.e., all the methods give very good results for noisy images 
but our method gives the highest {\it segmentation accuracy} which illustrates our model 
\eqref{our-model-multiphase-con} is superior compared with model \eqref{pcms:convex}; 
and only our method can give good results when the given images with information lost. 

To illustrate the effect of our method in segmenting blurry images, we first test our method in 
two synthetic multiphase images used in Fig. \ref{mutiphase-syn} but with Gaussion blur and 
motion blur involved, see Fig. \ref{mutiphase-syn-blur}. Fig. \ref{mutiphase-syn-blur}(A3) 
is blurred by using the gaussian kernel with size $10\times 10$ and standard deviation 10. 
After comparing with our method 
with methods \cite{YBTBmul,HSHMS,CCZ} in Fig. \ref{mutiphase-syn-blur}, we see 
that only method \cite{CCZ} and our method can give good results. 
More precisely, after comparing the {\it segmentation accuracy} of method \cite{CCZ} and 
our method, we can see that our method gives much higher SA which means our method 
gives better results than method \cite{CCZ}. Methods 
\cite{YBTBmul,HSHMS} are not able to segment the blurry images correctly. More
precisely, for the results of methods \cite{YBTBmul,HSHMS} in Fig. \ref{mutiphase-syn-blur}(A1) 
and (A2), the pixels around the boundaries are segmented uncorrectly; for their 
results in Fig. \ref{mutiphase-syn-blur}(A3) and (A4), even one star located in the right 
bottom corner is missed for both methods \cite{YBTBmul} and \cite{HSHMS}.

{\it Example 3: real-world images.}
Test our method in two real-world images, i.e., camera man and MRI (magnetic resonance imaging) 
brain image which comes from medical imaging subject, see Fig. \ref{mutiphase-real}. 
We first test our method in noisy images and images with information lost. 
In Fig. \ref{mutiphase-real}, the variance used for adding noise is 0.01, and the percentage 
of information lost  is $20\%$. The conclusions we get are very close to those obtained when 
we test the methods in synthetic images in examples 1 and 2. From the rows one and three of
Fig. \ref{mutiphase-real}, we see that all the methods give very good results in segmenting 
the two original real-world images. But for the images with information lost especially for the 
image in Fig. \ref{mutiphase-real}(A4), the results of 
methods \cite{YBTBmul,HSHMS} are worse than the results of methods \cite{CCZ} and ours, 
see Fig. \ref{mutiphase-real}(B4), (C4), (D4), and (E4). Moreover,
for the results of methods \cite{CCZ} and ours, we see that our result gives 
much more details for the white matter, see Fig. \ref{mutiphase-real}(D4) and (E4). 

The ability of our method in segmenting blurry images
is given in Fig. \ref{mutiphase-real-blur}. After comparing the results in 
Fig. \ref{mutiphase-real-blur}, we can see that method \cite{CCZ} and our method 
give very similar good results, on the contrary, the results of methods 
\cite{YBTBmul,HSHMS} are worse.

\begin{figure*}[!htb]
\begin{center}
\begin{tabular}{ccccc}
\includegraphics[width=28mm, height=28mm]{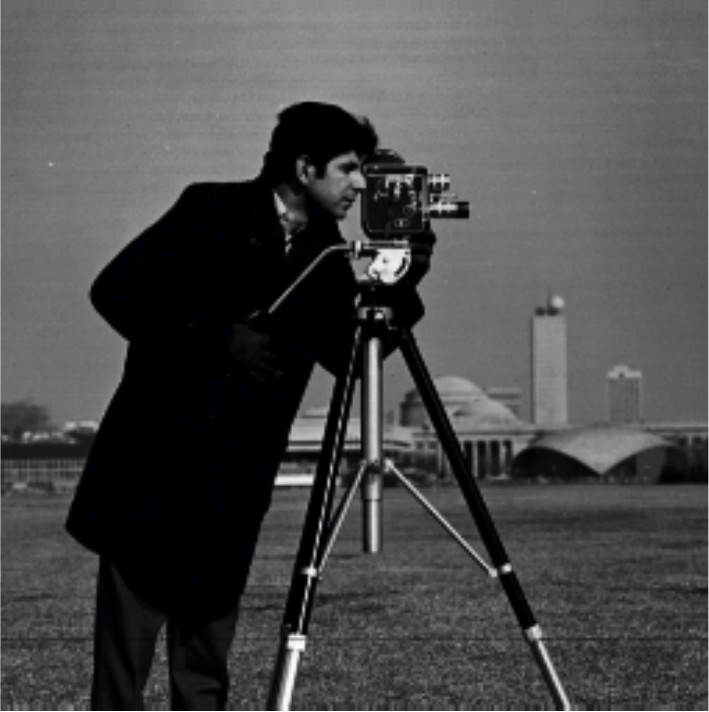}  &
\includegraphics[width=28mm, height=28mm]{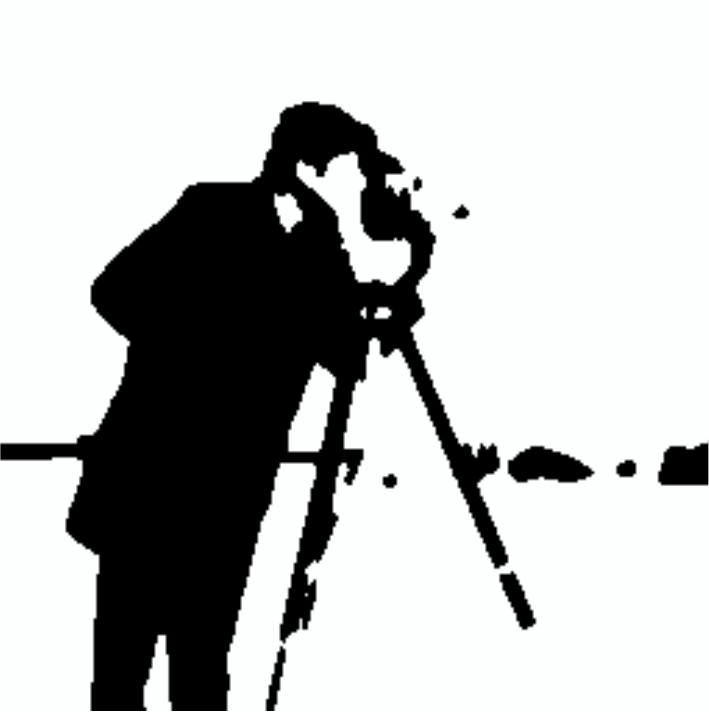}  &
\includegraphics[width=28mm, height=28mm]{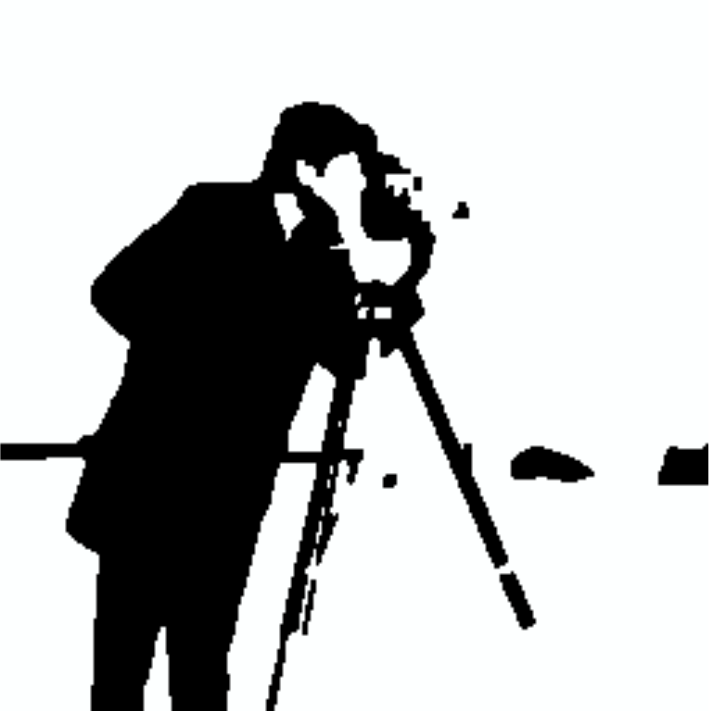}  &
\includegraphics[width=28mm, height=28mm]{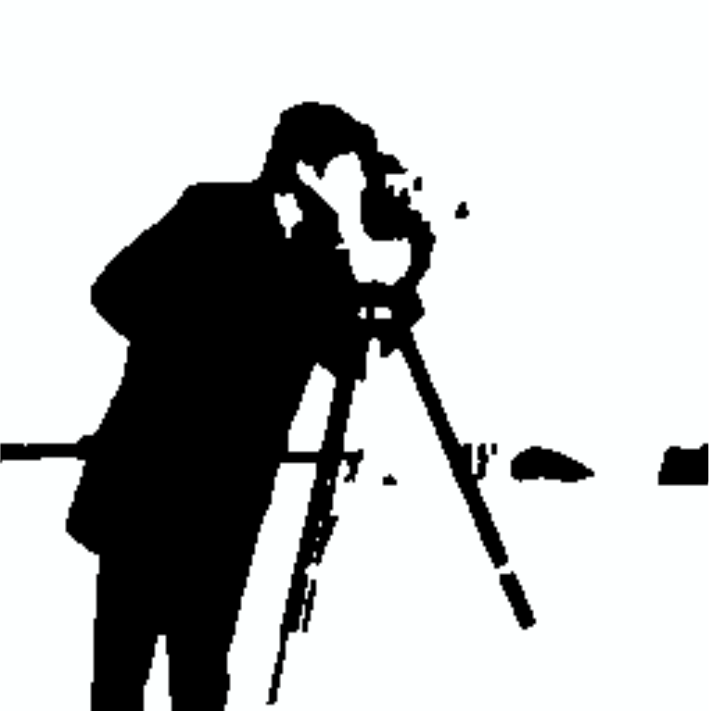}  &
\includegraphics[width=28mm, height=28mm]{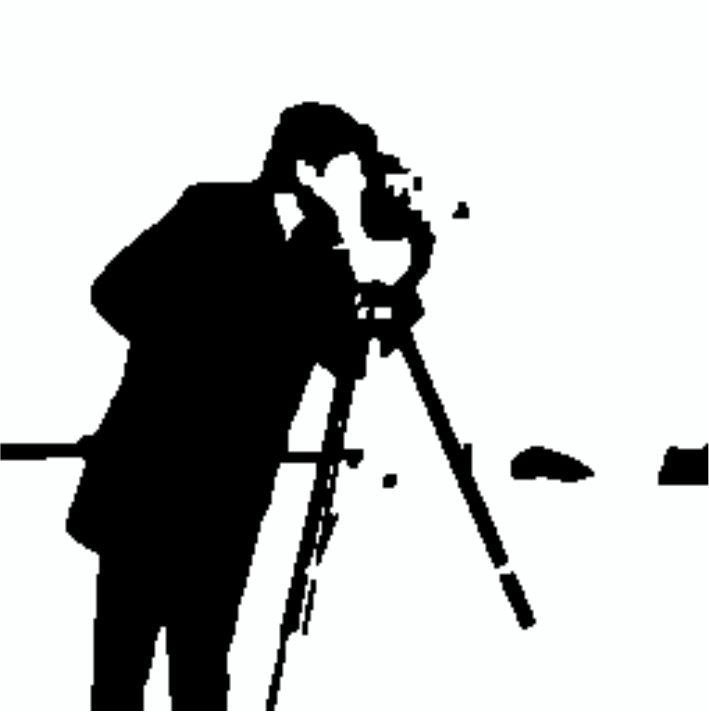} \\
{\small (A1) } &{\small (B1) \cite{YBTBmul} }  & {\small (C1) \cite{HSHMS} } & 
{\small (D1) \cite{CCZ}  } & {\small (E1) Our  } \\
\includegraphics[width=28mm, height=28mm]{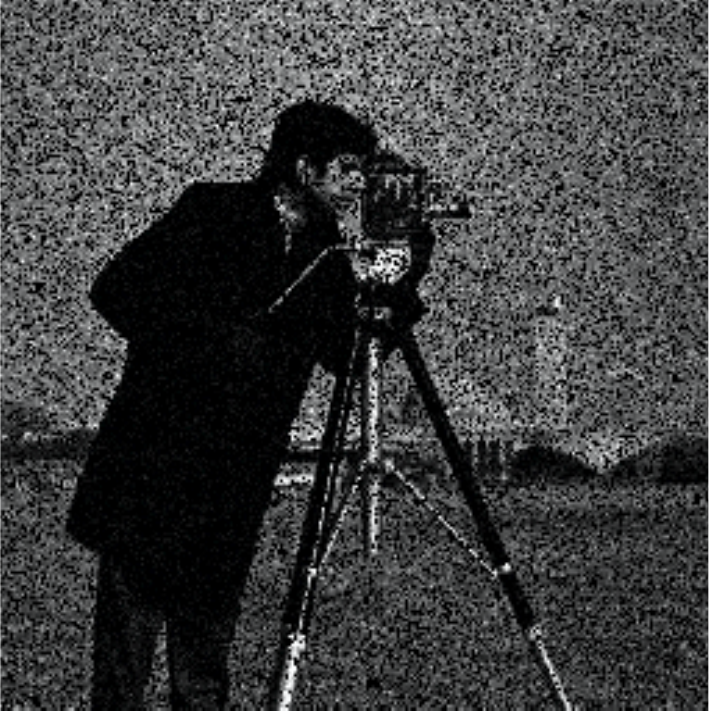}  &
\includegraphics[width=28mm, height=28mm]{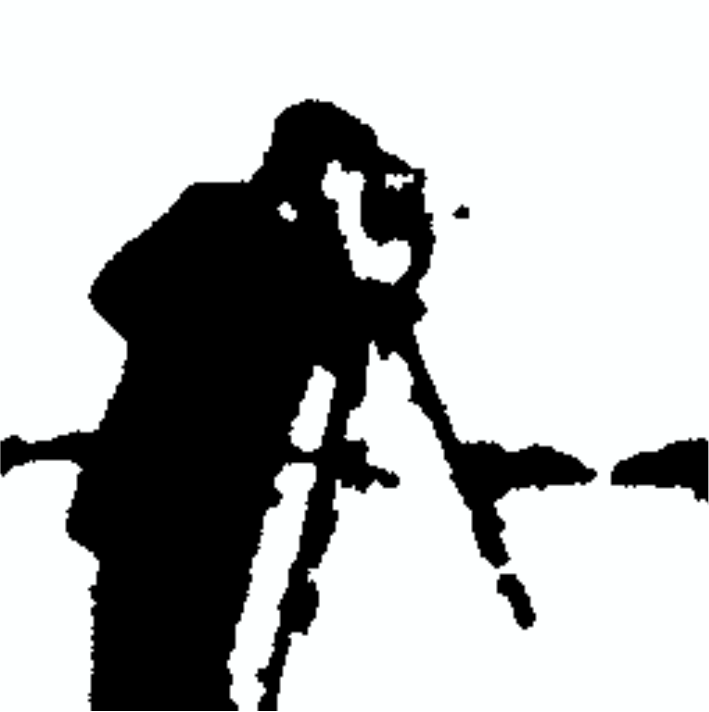}  &
\includegraphics[width=28mm, height=28mm]{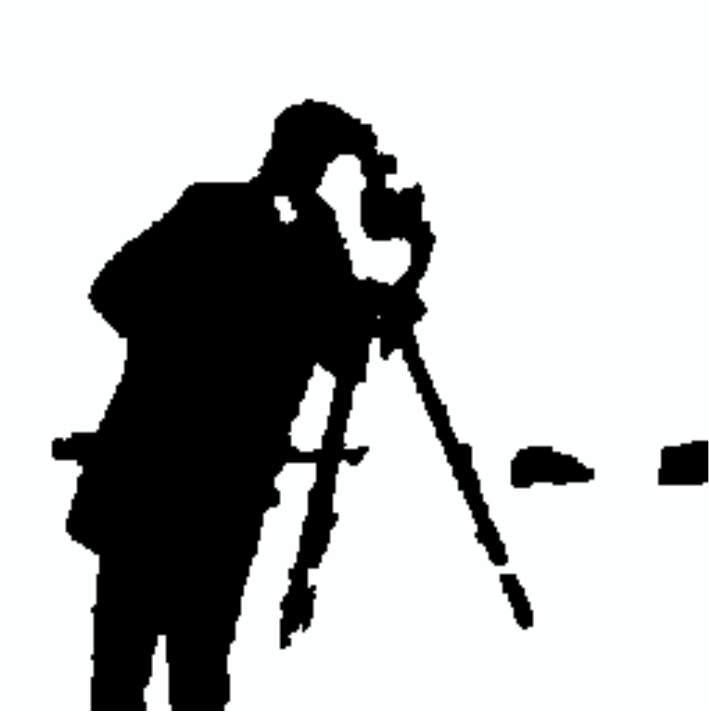}  &
\includegraphics[width=28mm, height=28mm]{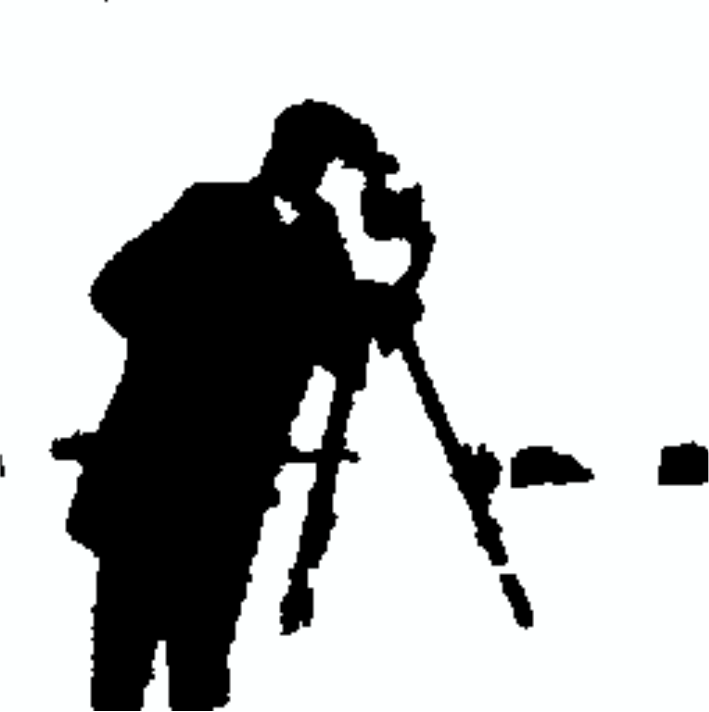}  &
\includegraphics[width=28mm, height=28mm]{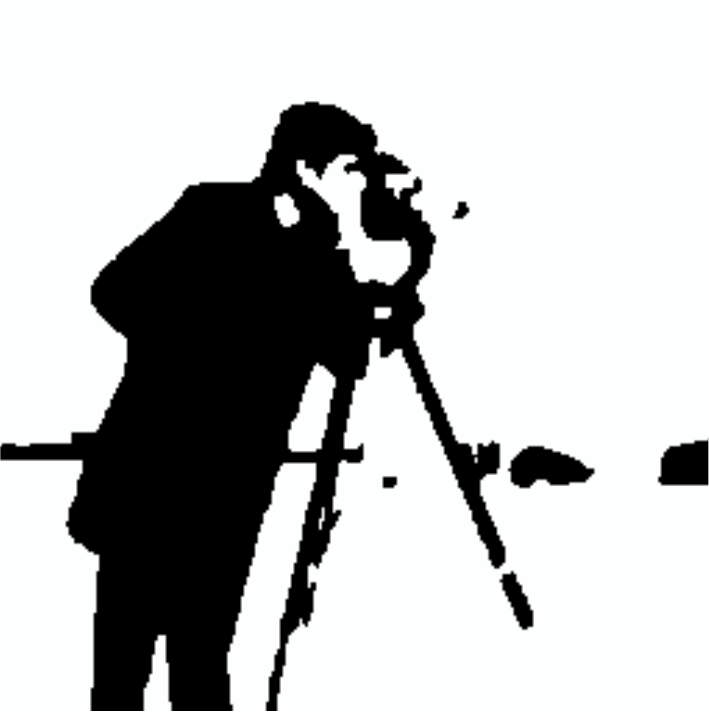} \\
{\small (A2)} &{\small (B2) \cite{YBTBmul}}  & {\small (C2) \cite{HSHMS} } & 
{\small (D2) \cite{CCZ} } & {\small (E2) Our  } \\
\includegraphics[width=28mm, height=33mm]{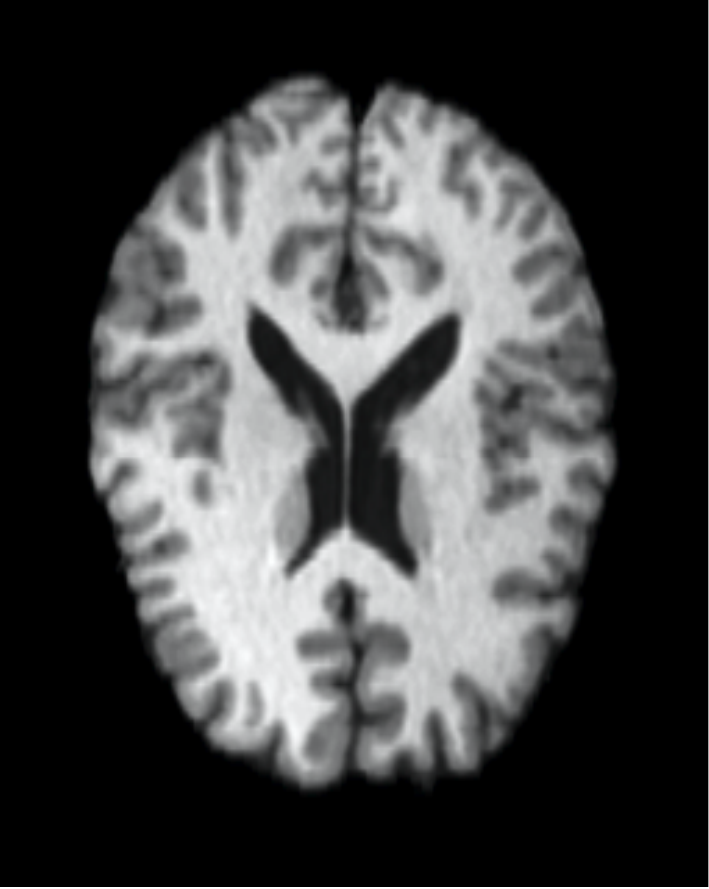}  &
\includegraphics[width=28mm, height=33mm]{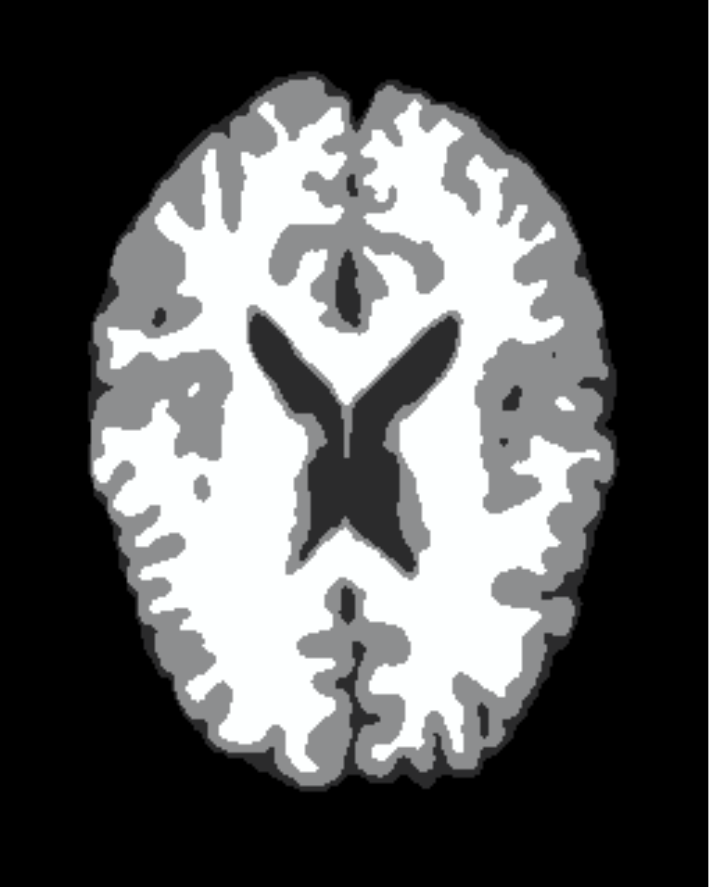}  &
\includegraphics[width=28mm, height=33mm]{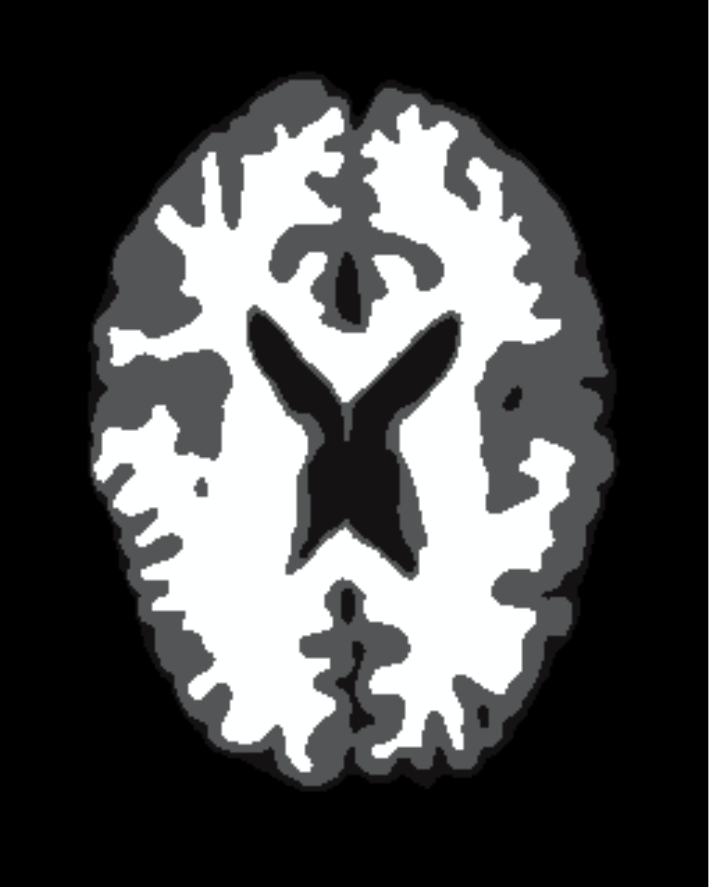}  &
\includegraphics[width=28mm, height=33mm]{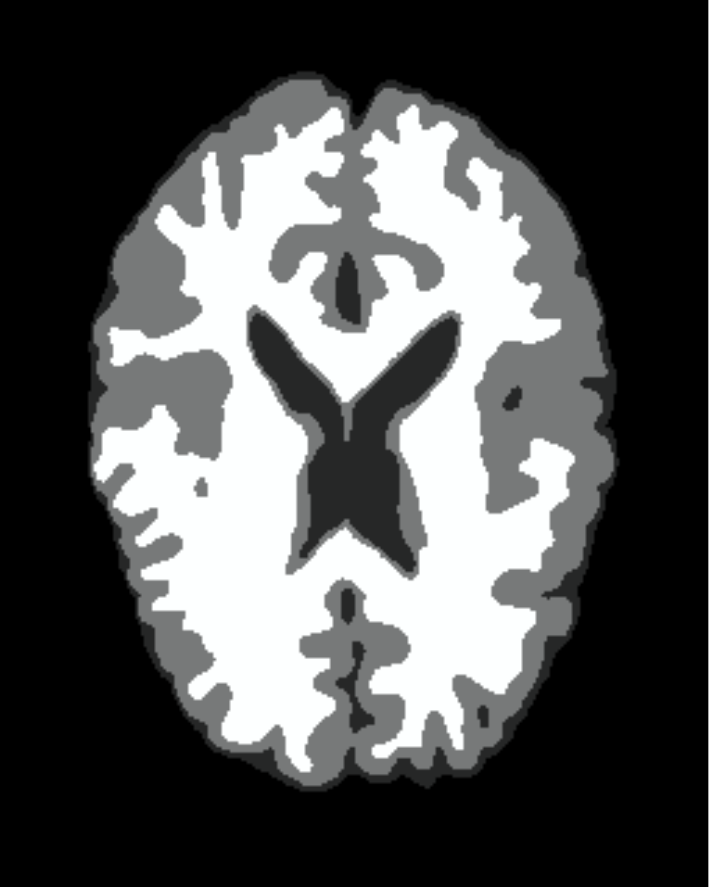}  &
\includegraphics[width=28mm, height=33mm]{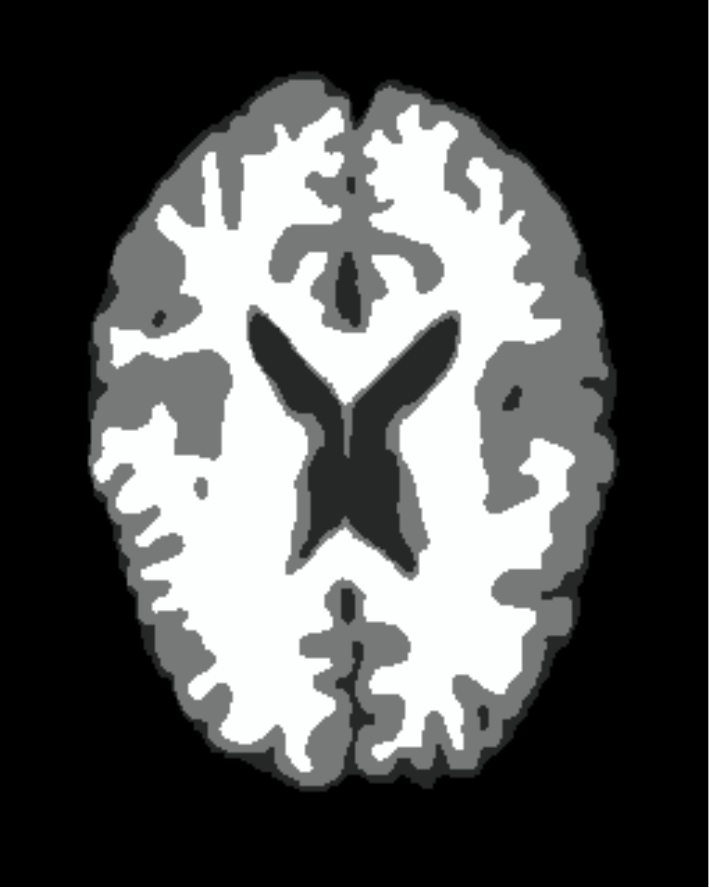} \\
{\small (A3)} &{\small (B3) \cite{YBTBmul} }  & {\small (C3) \cite{HSHMS} } & 
{\small (D3) \cite{CCZ} } & {\small (E3) Our  } \\
\includegraphics[width=28mm, height=33mm]{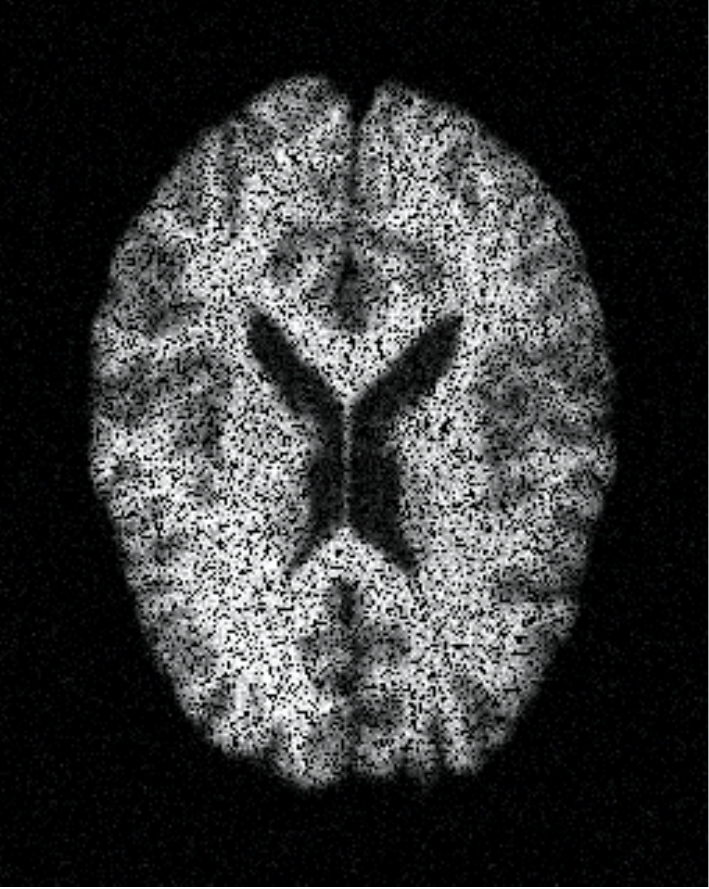} &
\includegraphics[width=28mm, height=33mm]{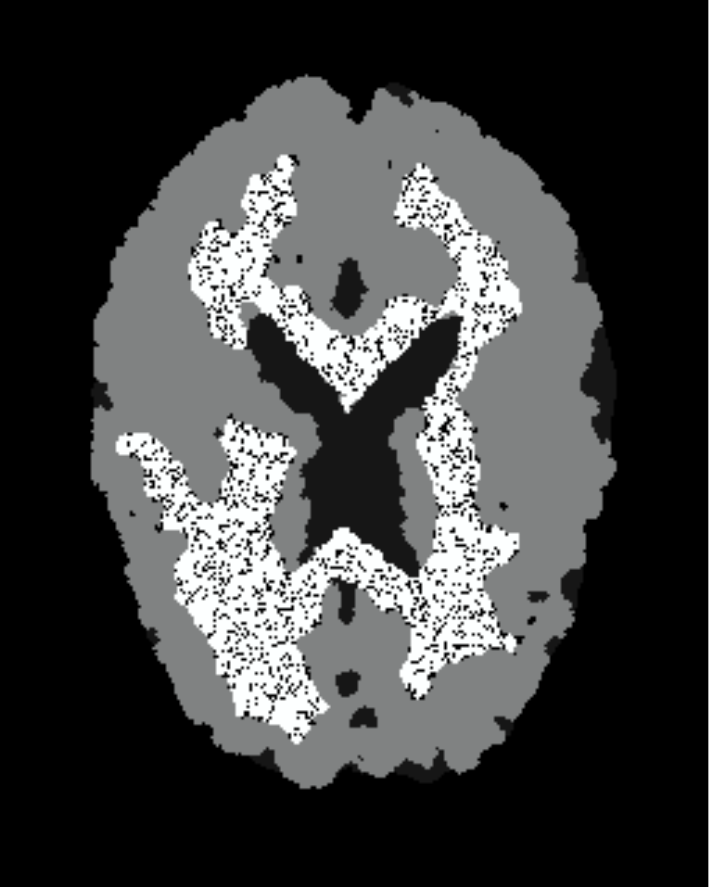} &
\includegraphics[width=28mm, height=33mm]{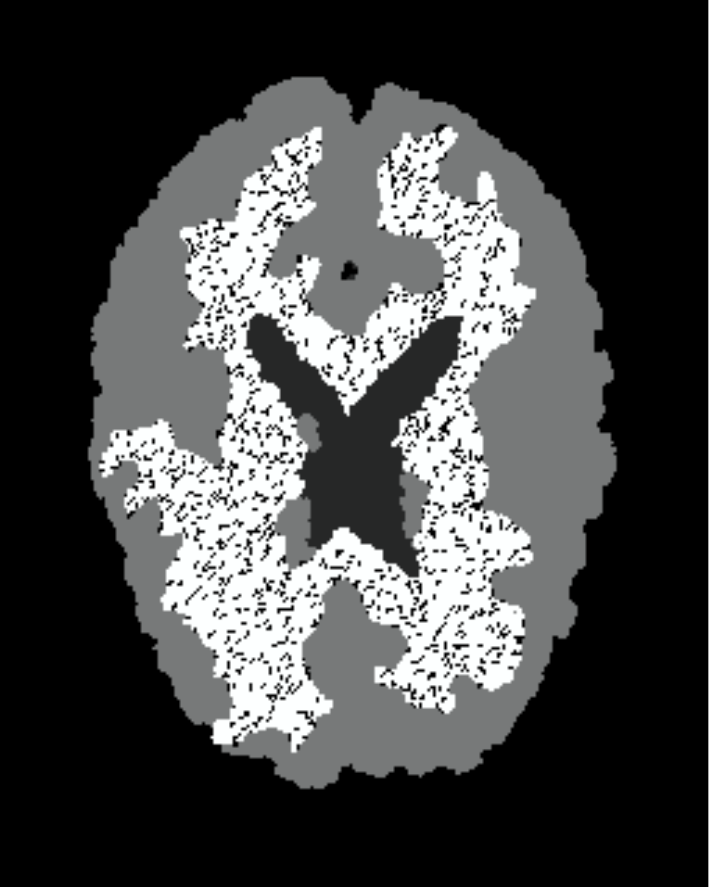} &
\includegraphics[width=28mm, height=33mm]{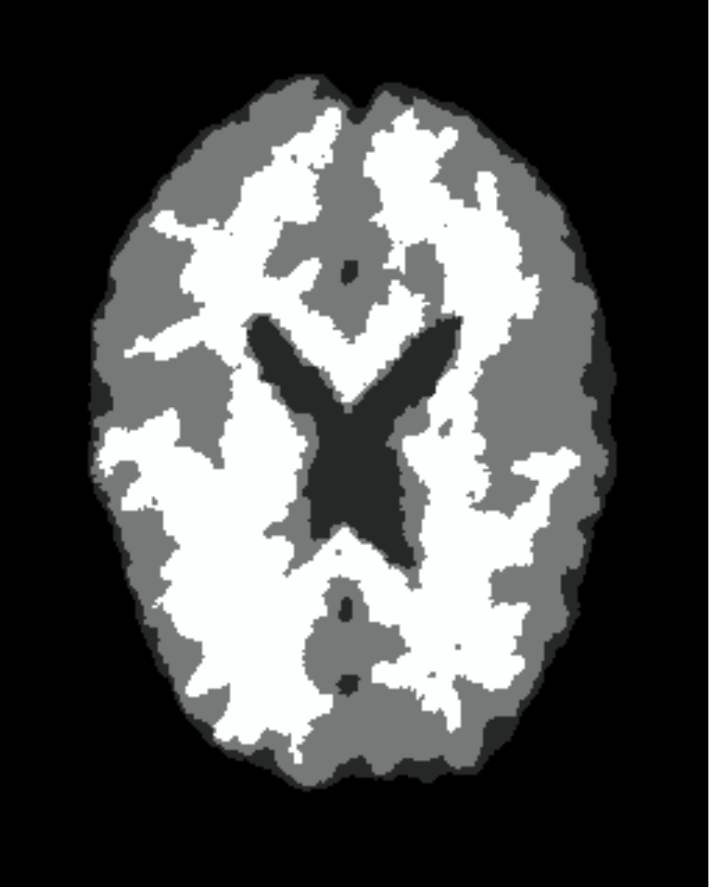} &
\includegraphics[width=28mm, height=33mm]{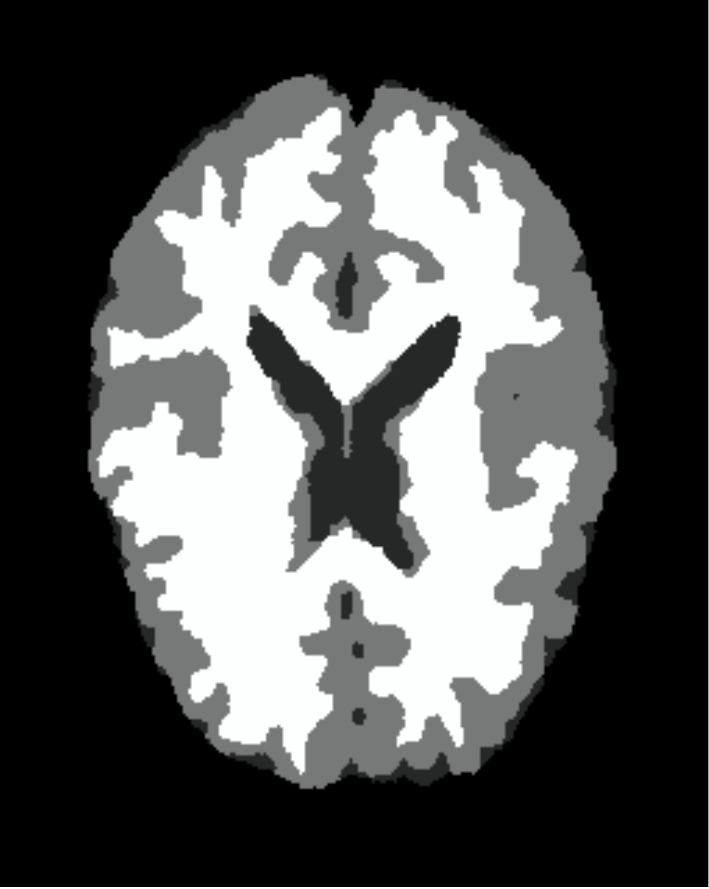}  \\
{\small (A4)} &{\small (B4) \cite{YBTBmul}}  & {\small (C4) \cite{HSHMS} } & 
{\small (D4) \cite{CCZ} } & {\small (E4) Our } 
\end{tabular}
\end{center}
\caption{Segmentation of real-world images: camera man and MRI brain 
($256\times 256$ and $319 \times 256$). 
(A1) and (A3): the given images; (A2) and (A4): the given noisy images with $20\%$ 
information lost; Columns two to five: the results of methods \cite{YBTBmul,HSHMS,CCZ} 
and our method, respectively. 
}\label{mutiphase-real}
\end{figure*}

\begin{figure*}[!htb]
\begin{center}
\begin{tabular}{ccccc}
\includegraphics[width=28mm, height=28mm]{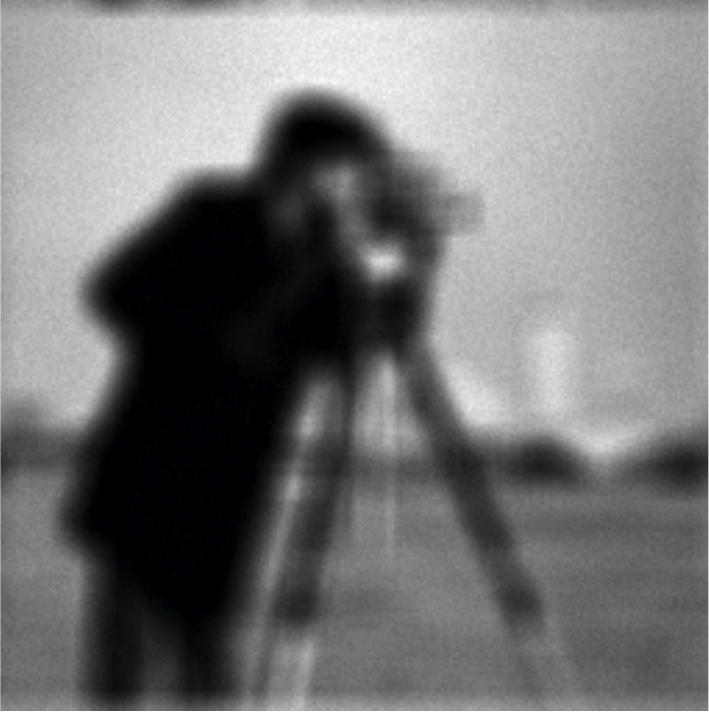}  &
\includegraphics[width=28mm, height=28mm]{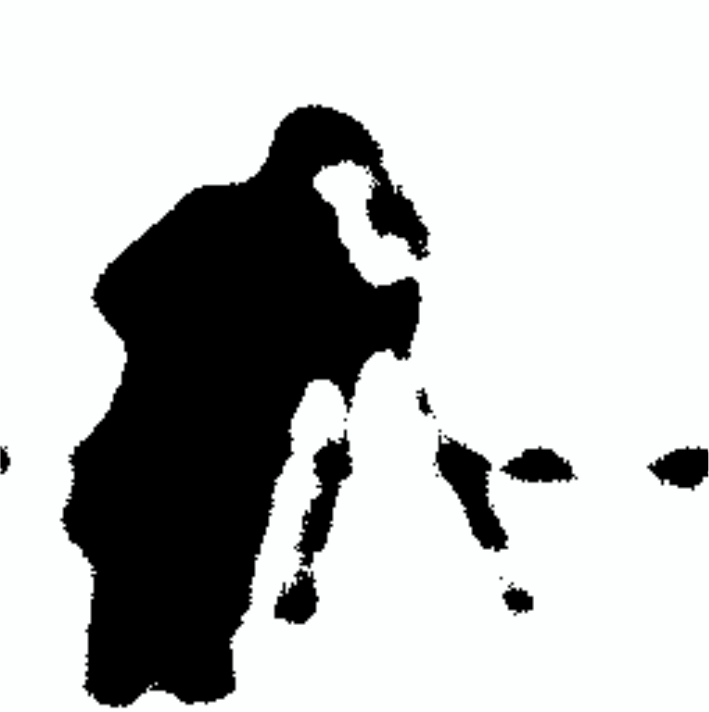}  &
\includegraphics[width=28mm, height=28mm]{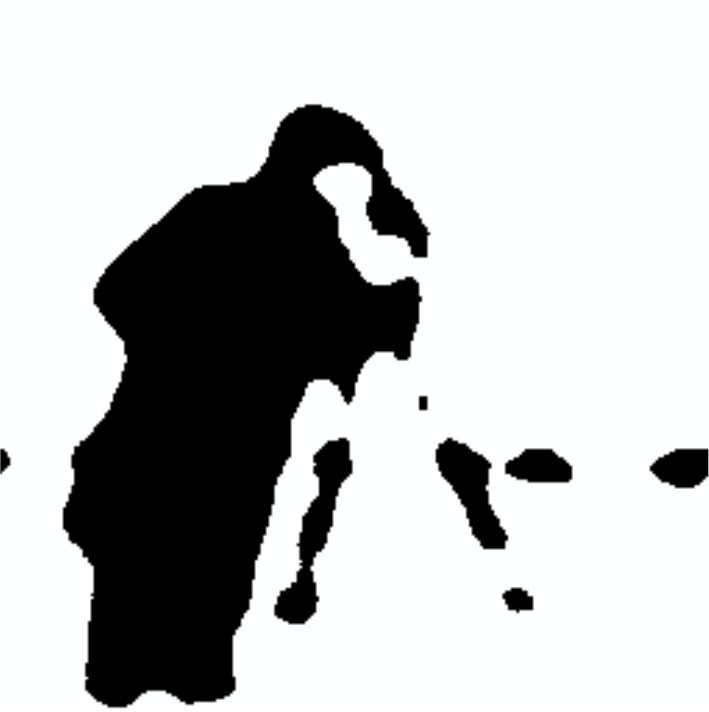}  &
\includegraphics[width=28mm, height=28mm]{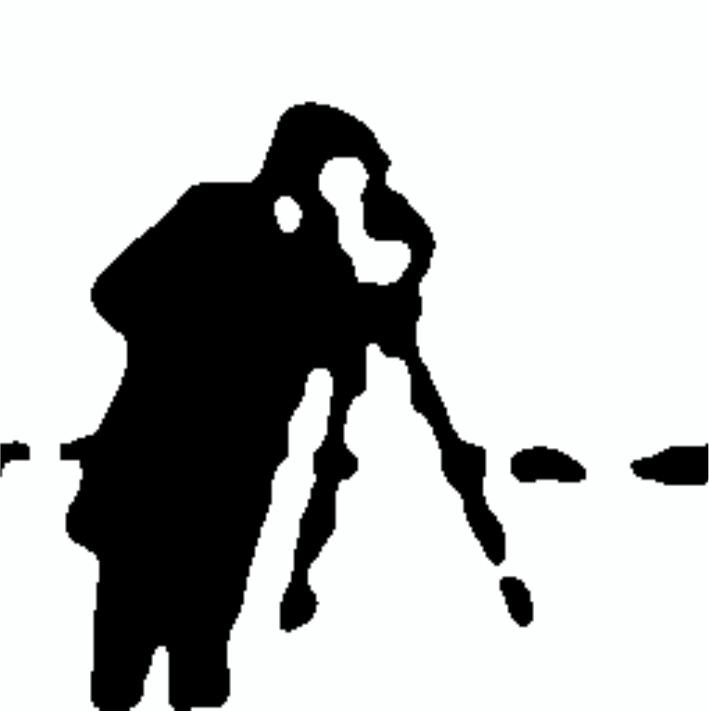}  &
\includegraphics[width=28mm, height=28mm]{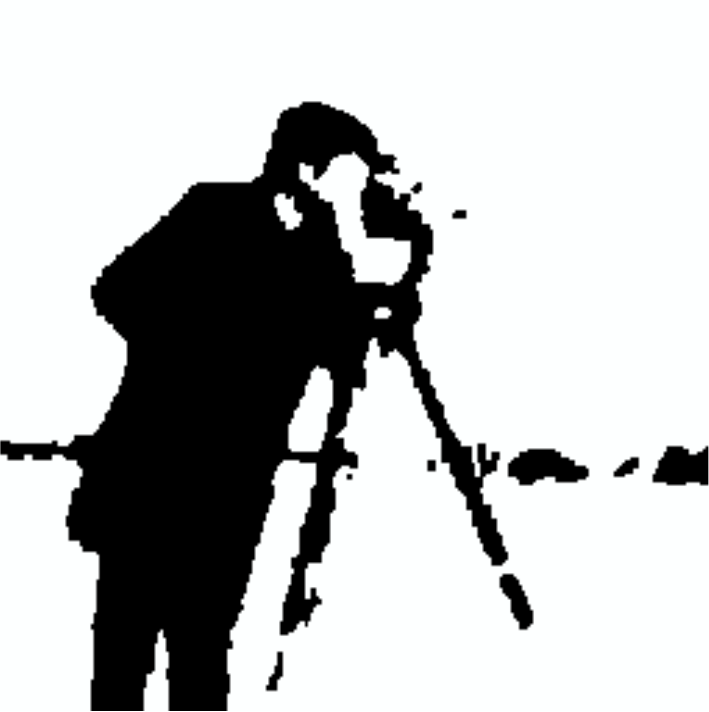} \\
{\small (A1) } &{\small (B1) \cite{YBTBmul} }  & {\small (C1) \cite{HSHMS} } & 
{\small (D1) \cite{CCZ}  } & {\small (E1) Our  } \\
\includegraphics[width=28mm, height=32mm]{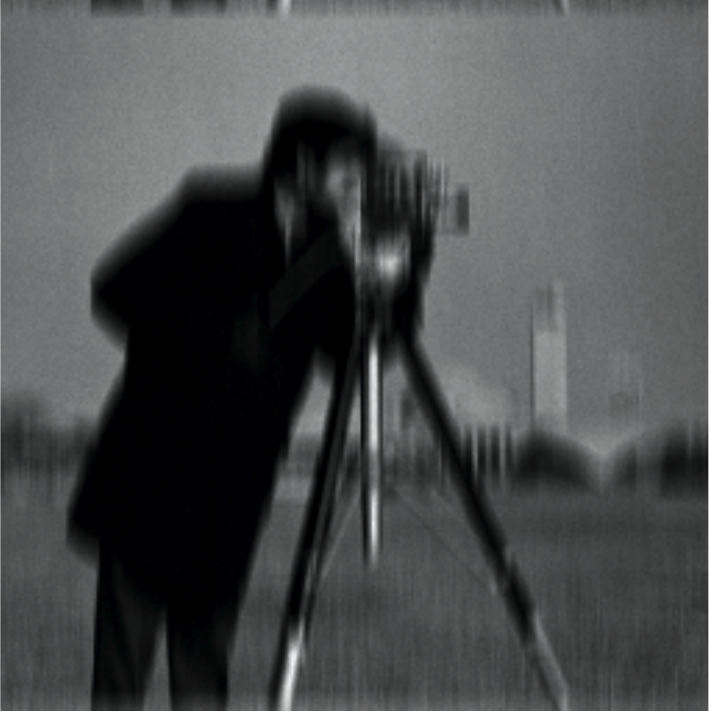}  &
\includegraphics[width=28mm, height=28mm]{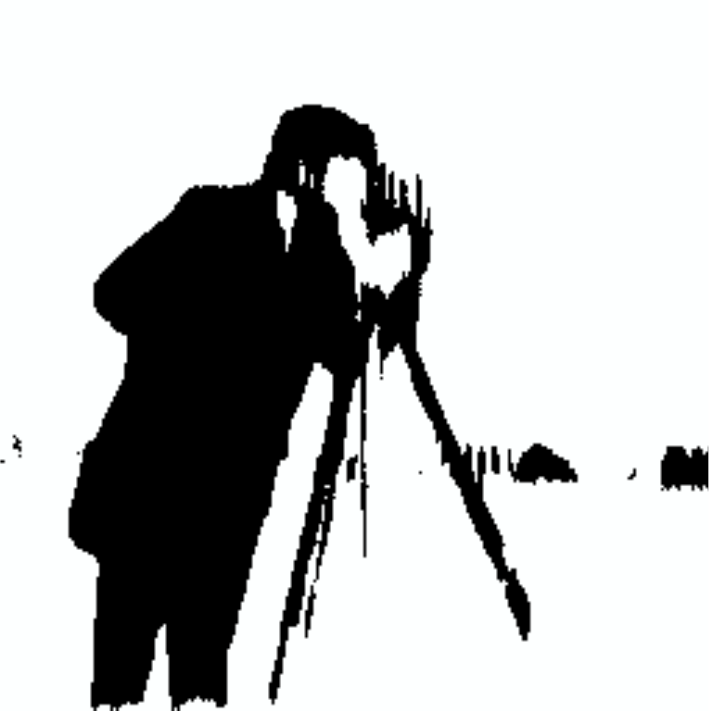}  &
\includegraphics[width=28mm, height=28mm]{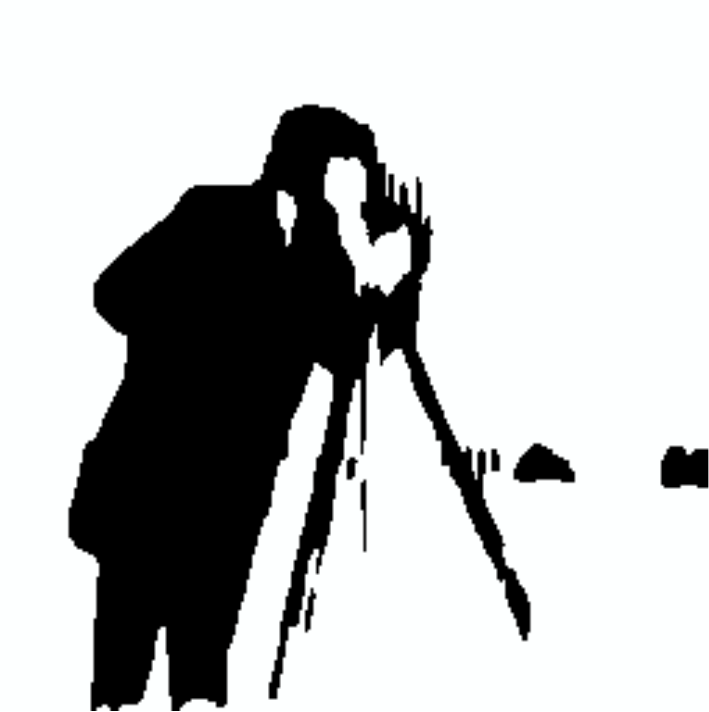}  &
\includegraphics[width=28mm, height=28mm]{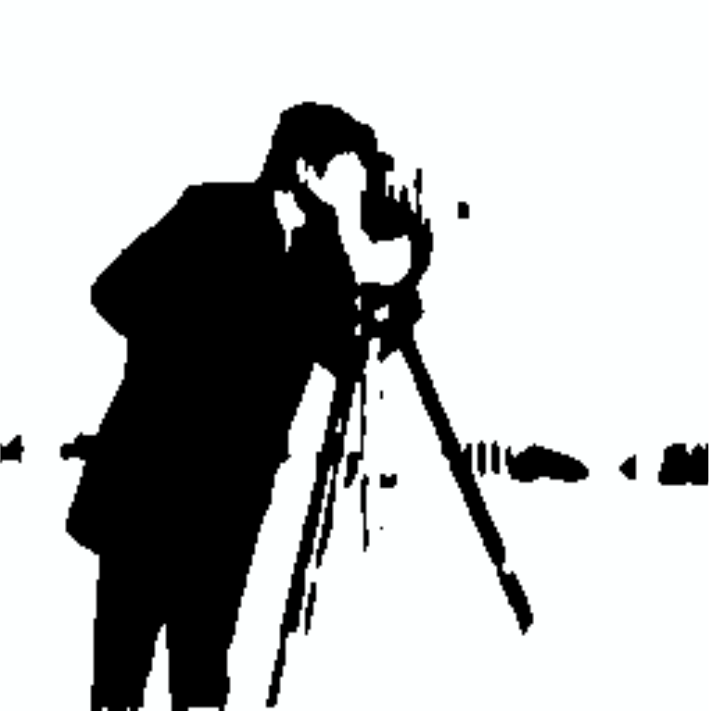}  &
\includegraphics[width=28mm, height=28mm]{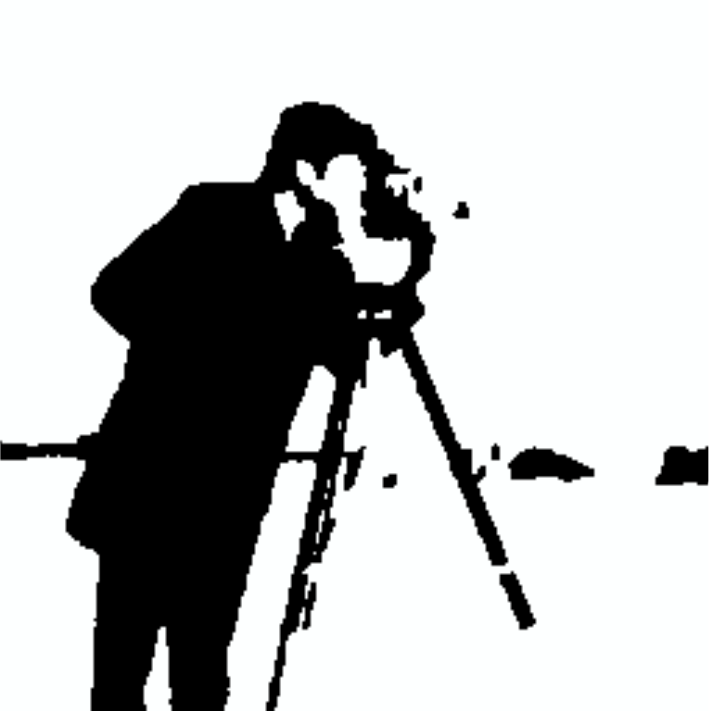} \\
{\small (A2)} &{\small (B2) \cite{YBTBmul}}  & {\small (C2) \cite{HSHMS} } & 
{\small (D2) \cite{CCZ} } & {\small (E2) Our  } \\
\includegraphics[width=28mm, height=33mm]{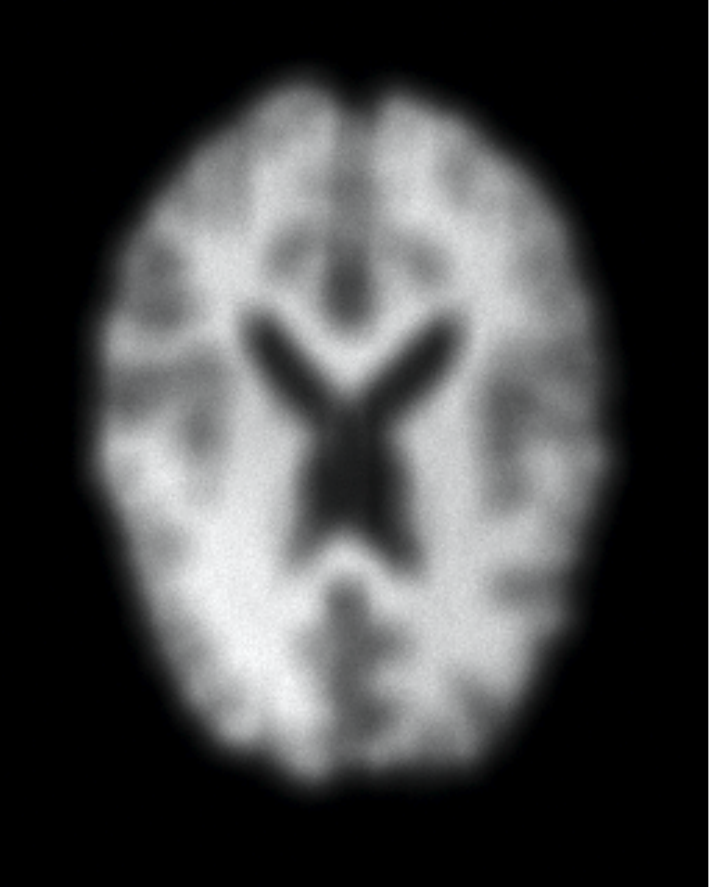}  &
\includegraphics[width=28mm, height=33mm]{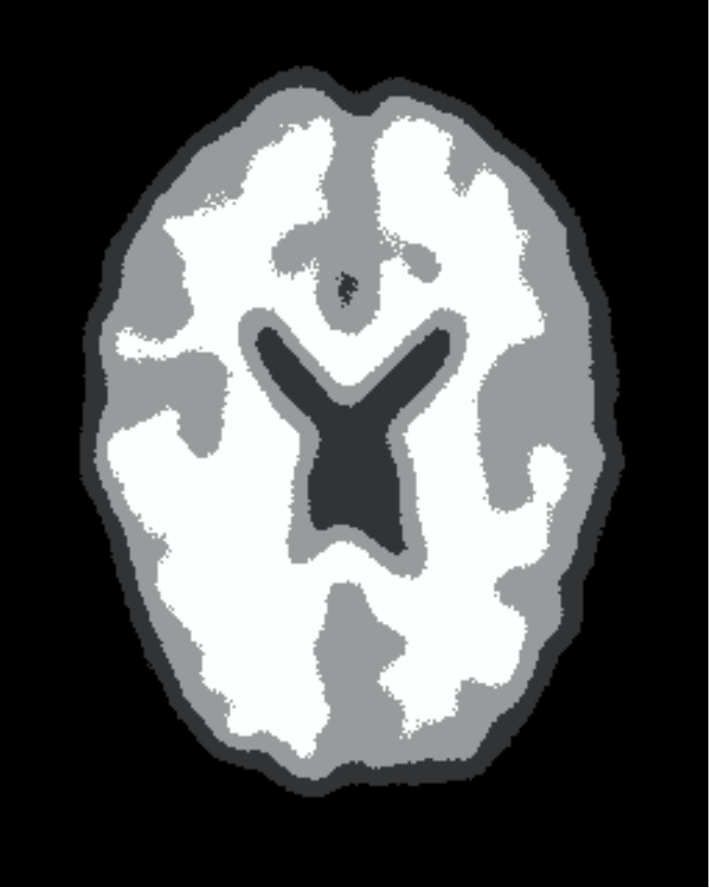}  &
\includegraphics[width=28mm, height=33mm]{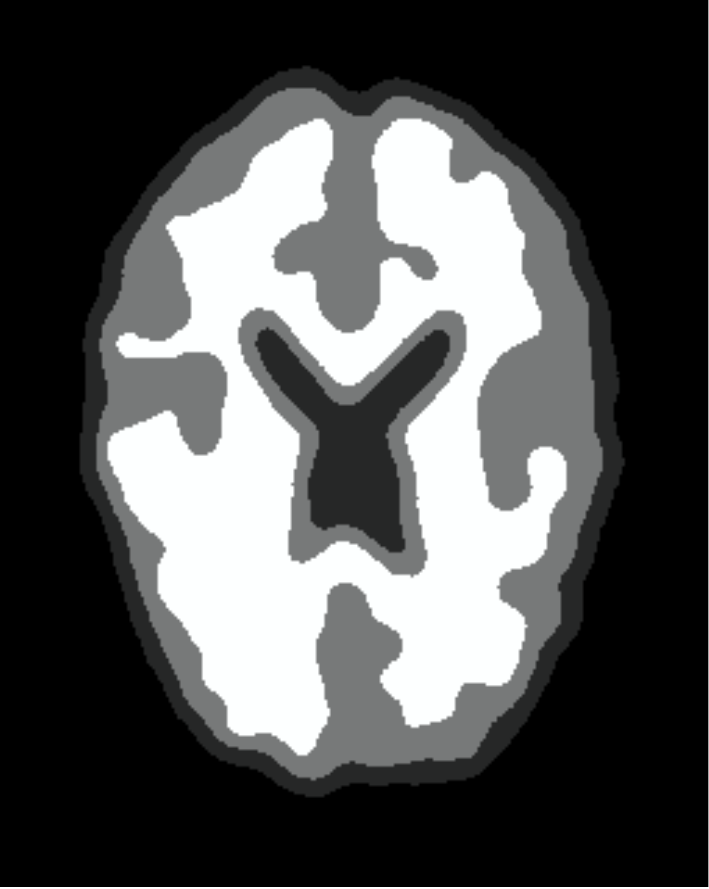}  &
\includegraphics[width=28mm, height=33mm]{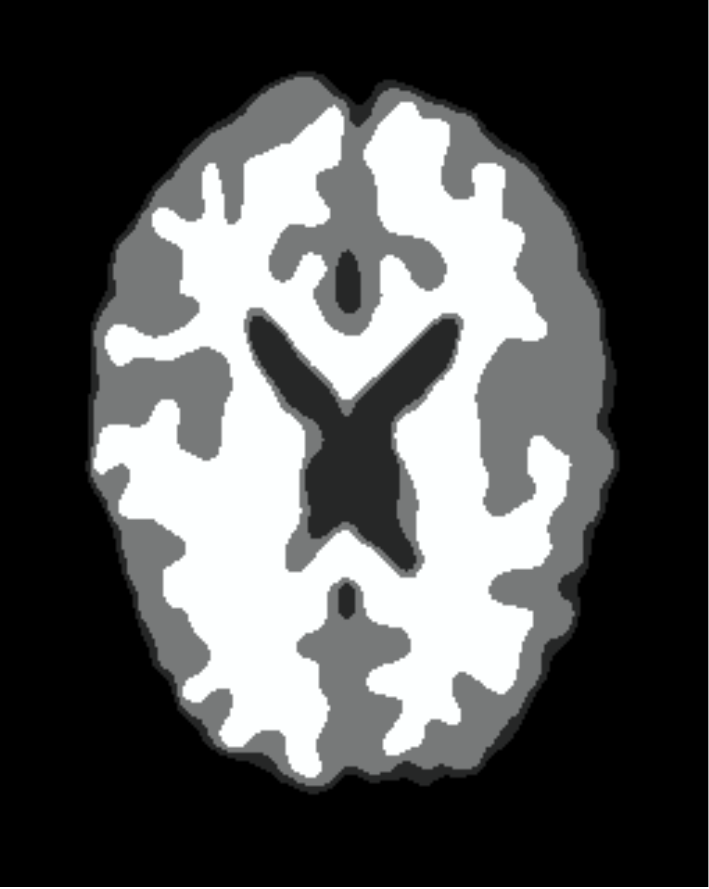}  &
\includegraphics[width=28mm, height=33mm]{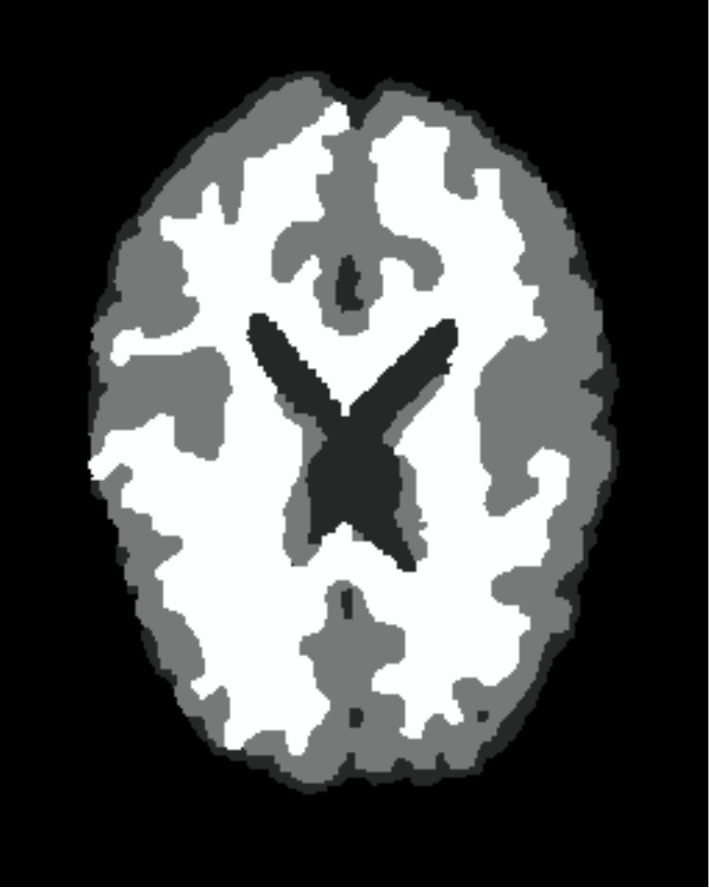} \\
{\small (A3)} &{\small (B3) \cite{YBTBmul} }  & {\small (C3) \cite{HSHMS} } & 
{\small (D3) \cite{CCZ} } & {\small (E3) Our  } \\
\includegraphics[width=28mm, height=33mm]{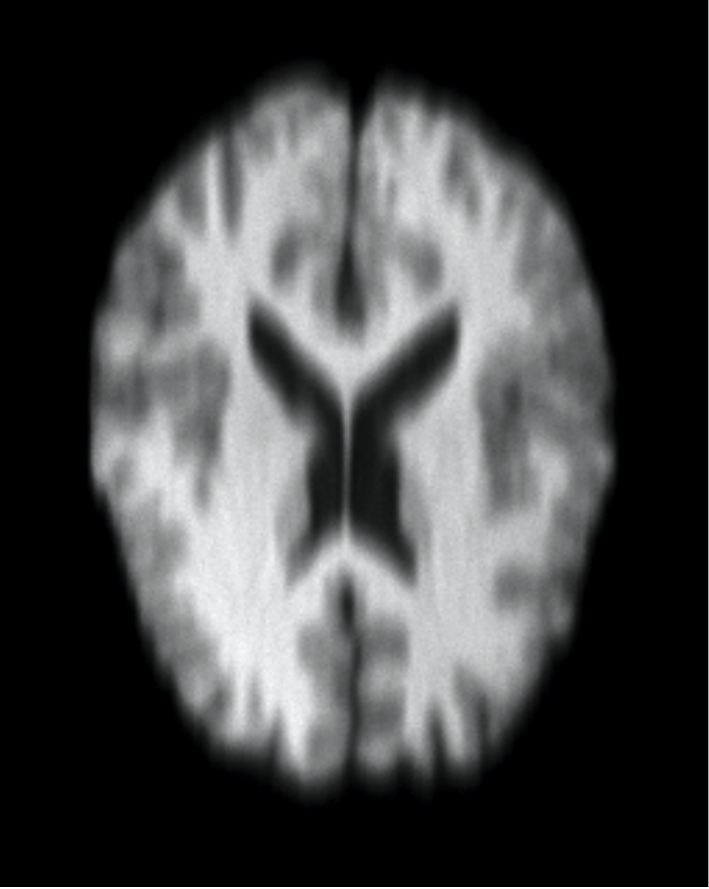}  &
\includegraphics[width=28mm, height=33mm]{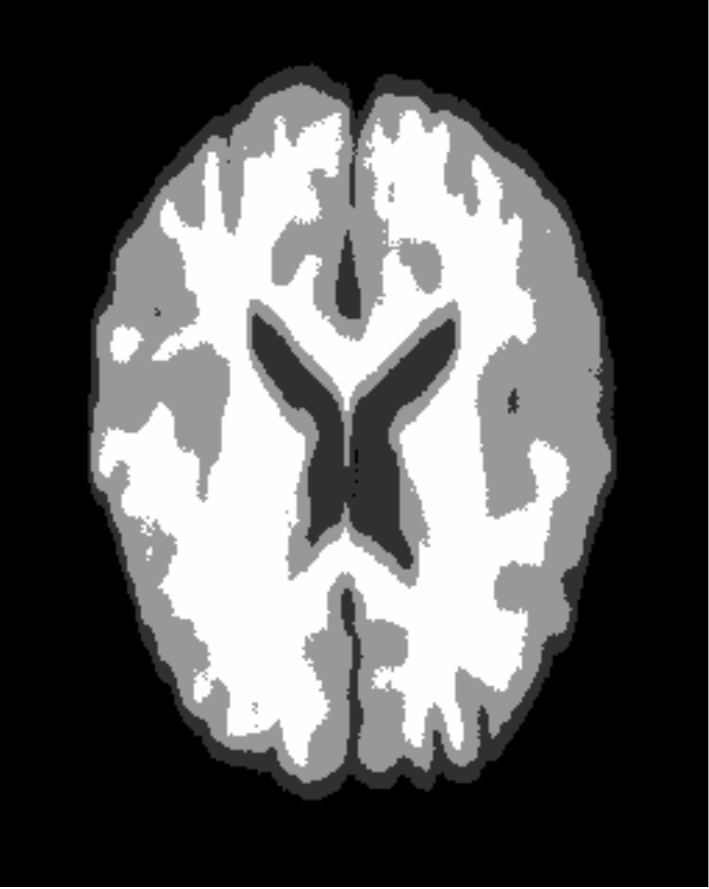}  &
\includegraphics[width=28mm, height=33mm]{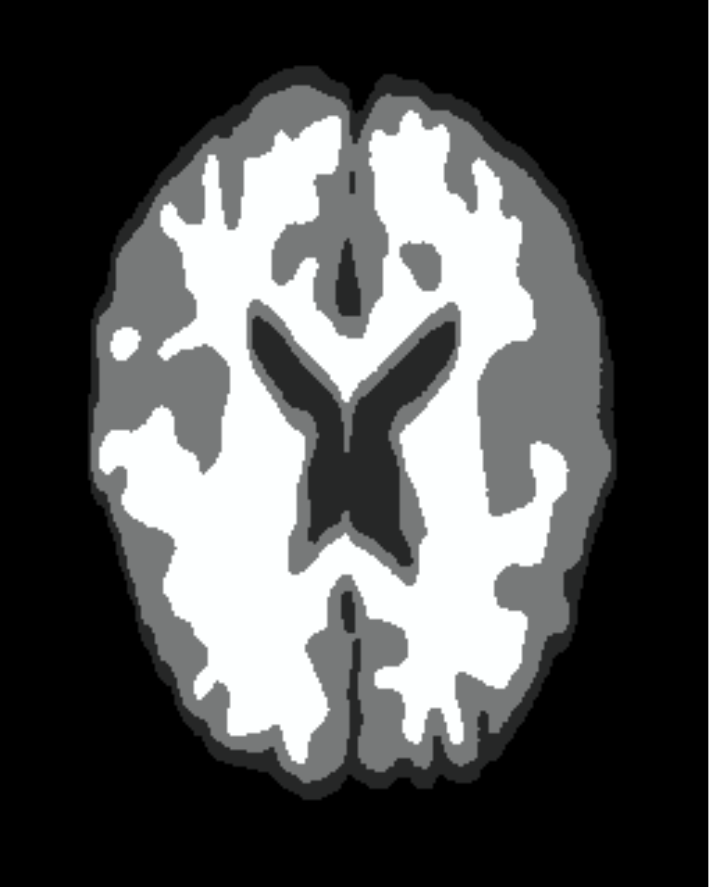}  &
\includegraphics[width=28mm, height=33mm]{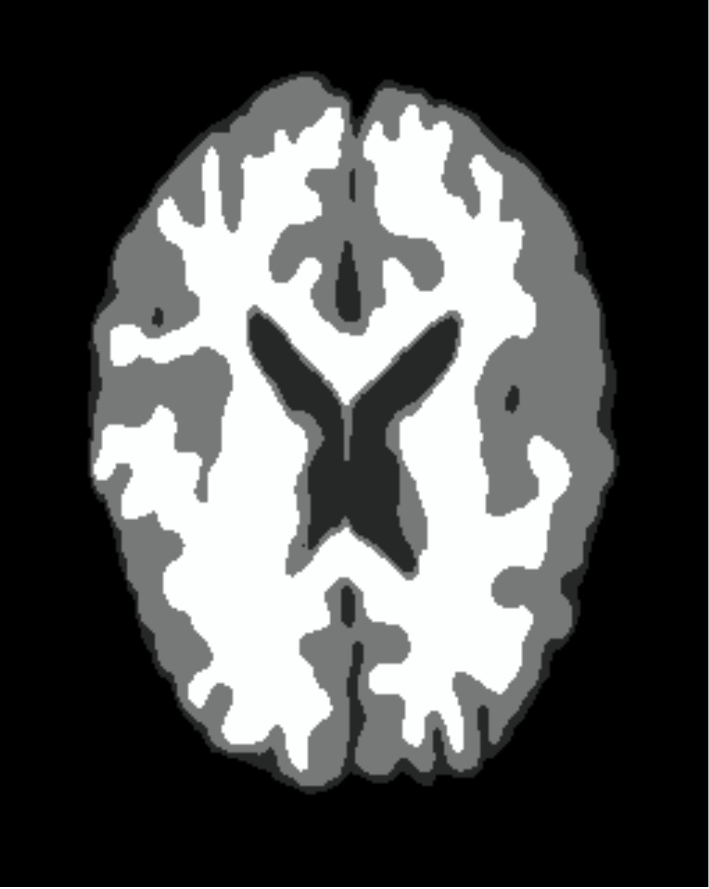}  &
\includegraphics[width=28mm, height=33mm]{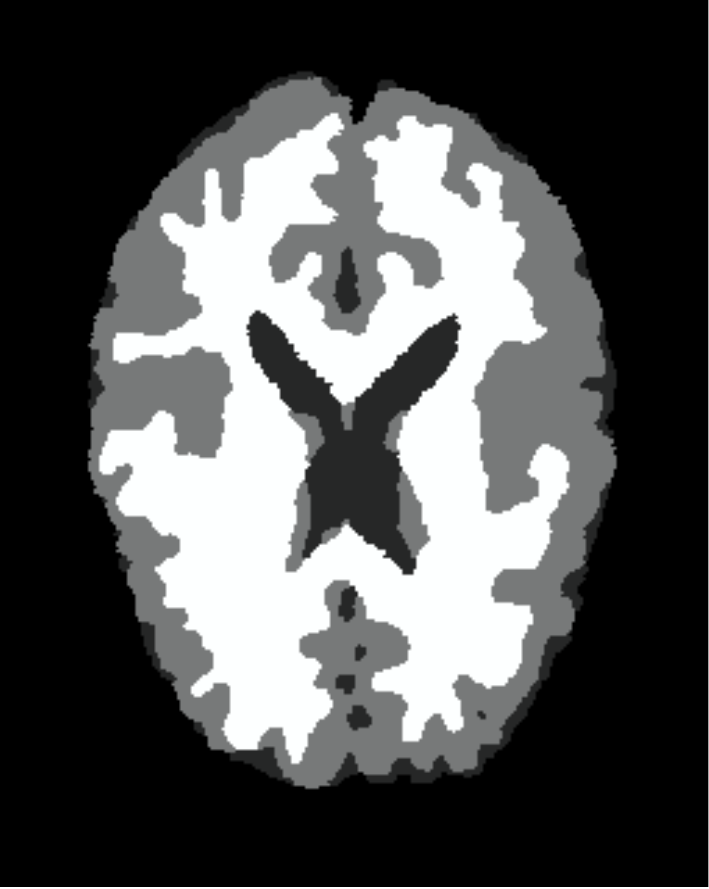} \\
{\small (A4)} &{\small (B4) \cite{YBTBmul}}  & {\small (C4) \cite{HSHMS} } & 
{\small (D4) \cite{CCZ} } & {\small (E4) Our } 
\end{tabular}
\end{center}
\caption{Segmentation of real-world blurry images: camera man and MRI brain 
($256\times 256$ and $319 \times 256$). (A1) and (A3): the given images with Gaussion blur; 
(A2) and (A4): the given images with motion blur; Columns two to five: the results of 
methods \cite{YBTBmul,HSHMS,CCZ} and our method, respectively. 
}\label{mutiphase-real-blur}
\end{figure*}

\subsection{Color image segmentation}

{\it Example 4: two-phase rose image.}
Fig. \ref{rose-color}(A1)--(A4) give the original rose image, and the original image 
corrupted by part information removed randomly, Gaussian blur and motion blur, respectively. 
The columns two and three in Fig. \ref{rose-color} give the results of the extended 
method \cite{HSHMS}  and our method, respectively. After the comparison, 
we see that both of the two methods can give good results 
for the original image, see the first row of Fig. \ref{rose-color}; from the second row
of  Fig. \ref{rose-color}, we see that the boundary of the result of the extended 
method \cite{HSHMS} is coarse compared with our result; from rows three and
four of Fig. \ref{rose-color}, we can see that the boundaries of the results of the extended 
method \cite{HSHMS} are over smoothed compared with our results. 
Hence, we have that our model can get better results in segmenting blurry color images, while
the extended method \cite{HSHMS} can not. Moreover, 
from the results of our method, we can see that the results of our method for the 
corrupted images are as good as the result of our method for the original image,
see the third column of Fig. \ref{rose-color}. This really demonstrates the ability of our method
in segmenting images with information lost or blur.

\begin{figure*}[!t]
\begin{center}
\begin{tabular}{cccc}
\includegraphics[width=33mm, height=38mm]{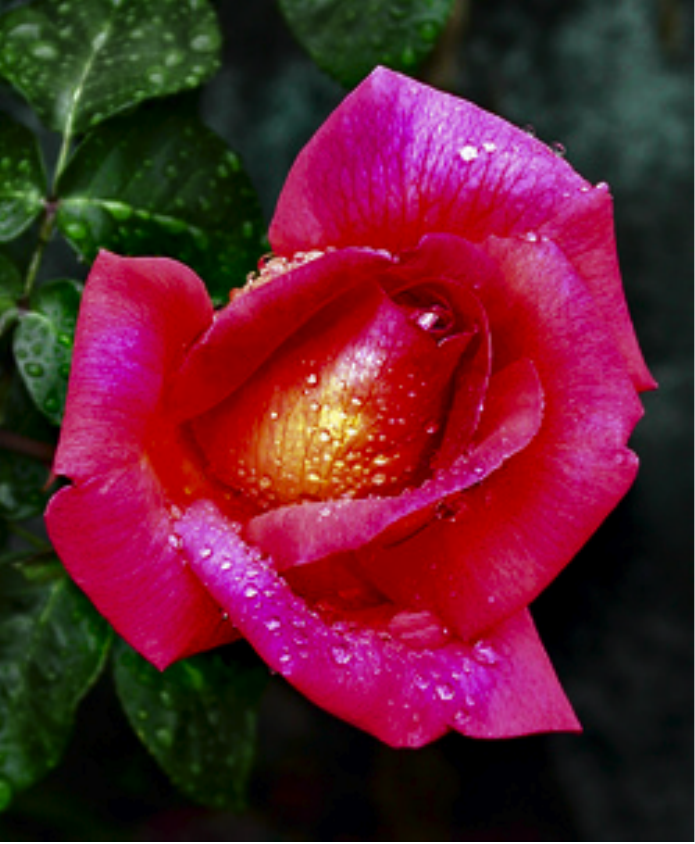}  &
\includegraphics[width=33mm, height=38mm]{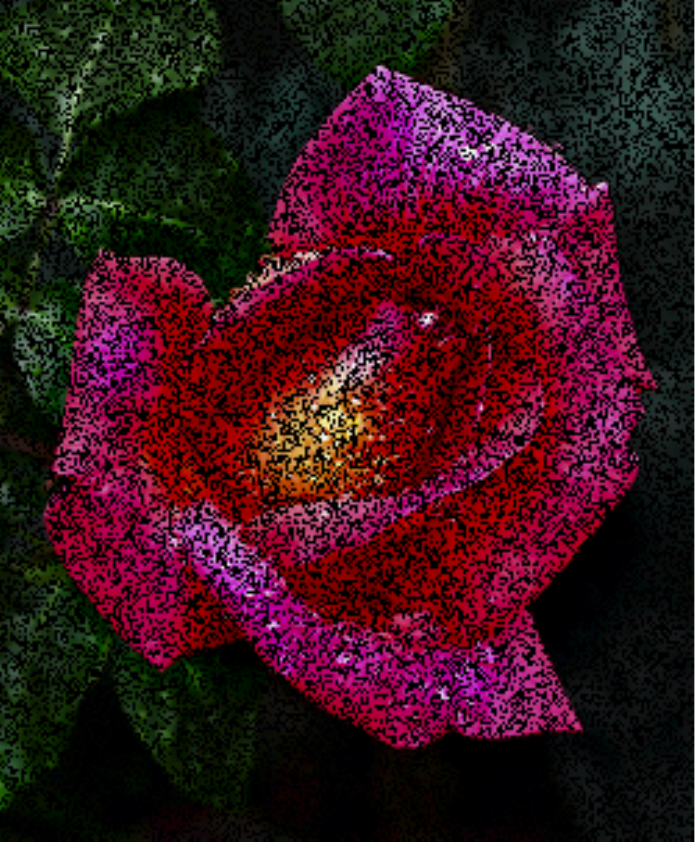}  &
\includegraphics[width=33mm, height=38mm]{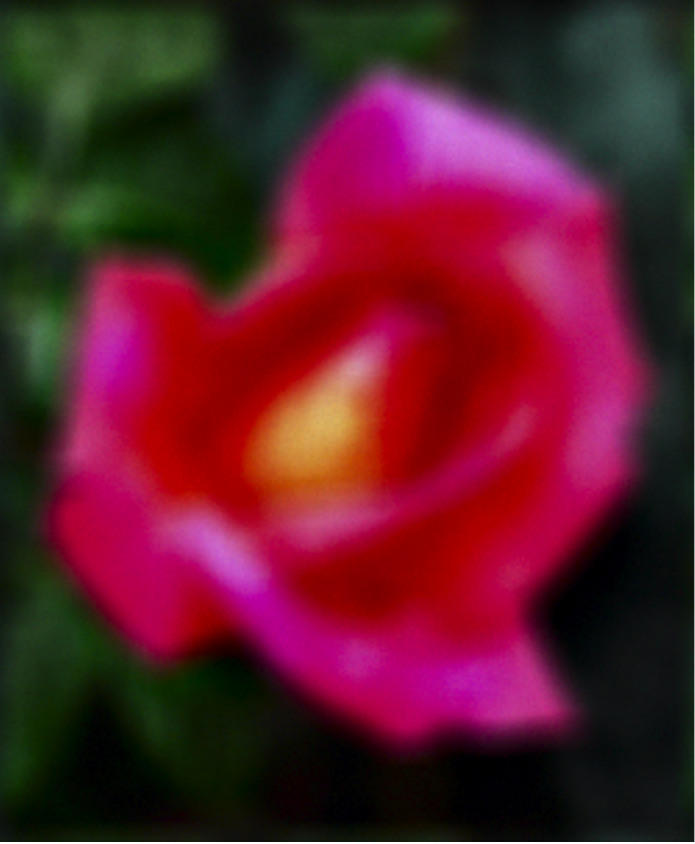}  &
\includegraphics[width=33mm, height=38mm]{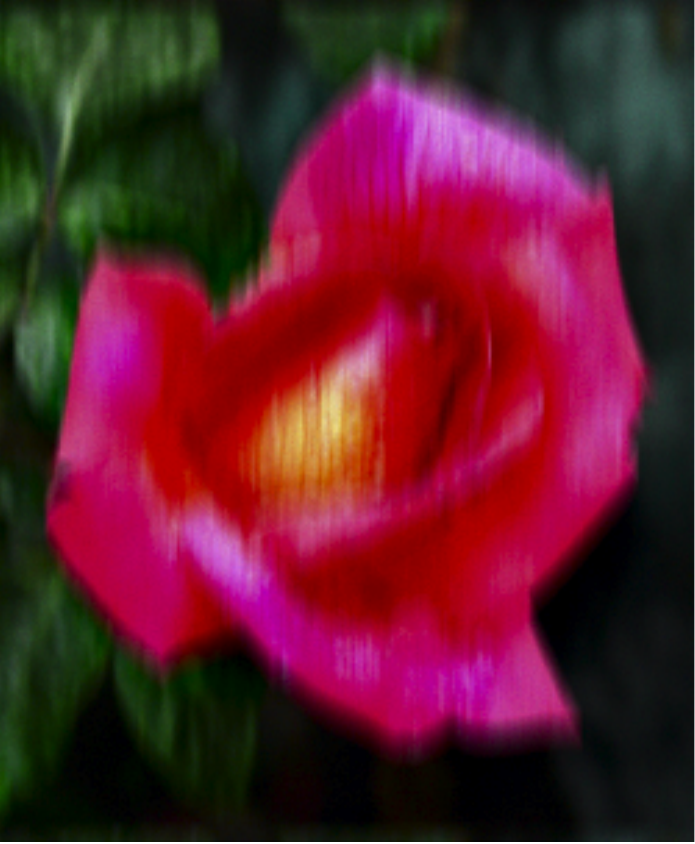}  \\
{\small (A1) } & {\small (A2) } & {\small (A3) }  & {\small (A4) } \\
\includegraphics[width=33mm, height=38mm]{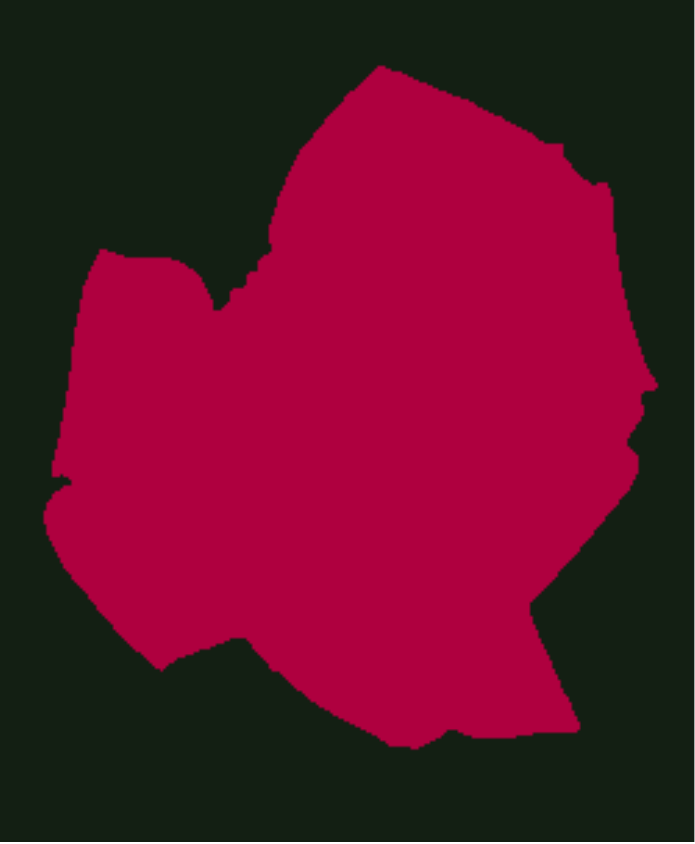} &
\includegraphics[width=33mm, height=38mm]{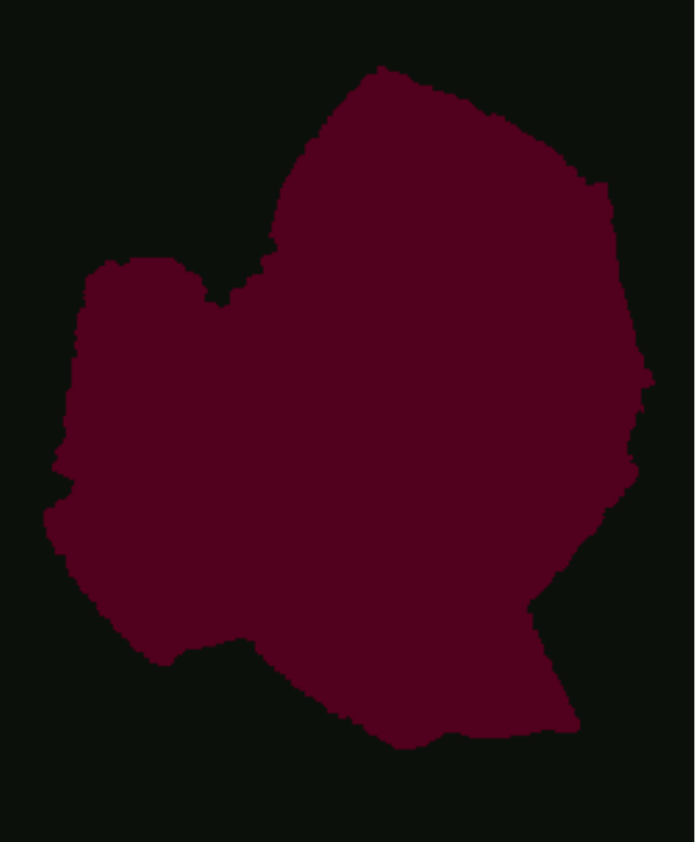} &
\includegraphics[width=33mm, height=38mm]{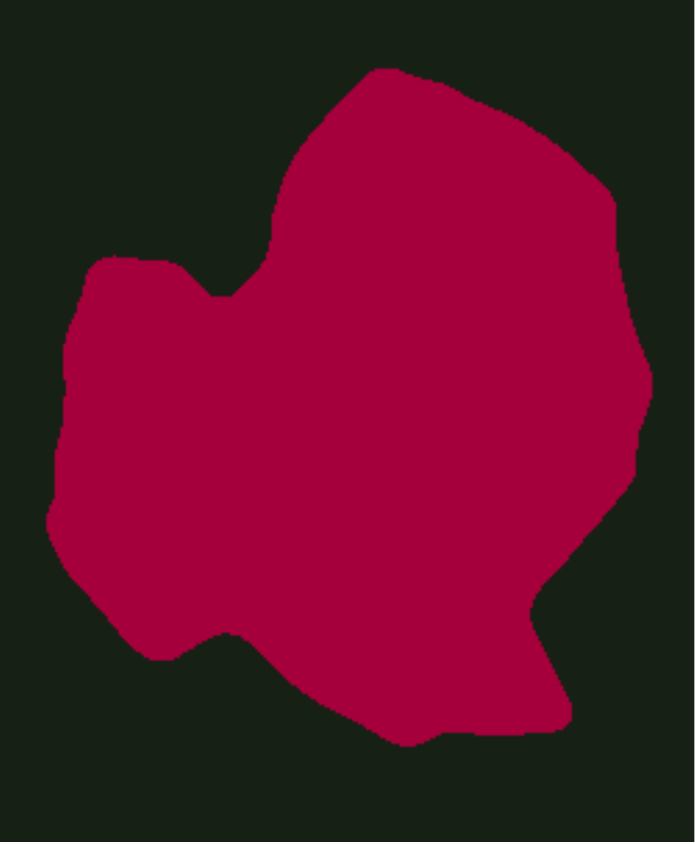} &
\includegraphics[width=33mm, height=38mm]{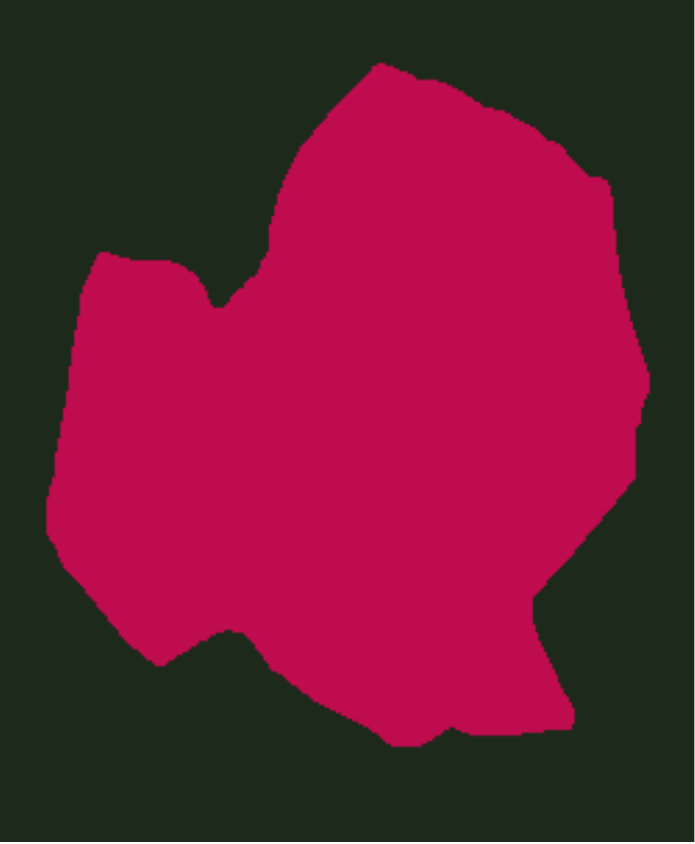} \\
 {\small (B1) Extended \cite{HSHMS}  } &  {\small (B2) Extended  \cite{HSHMS}  }  &
{\small (B3) Extended \cite{HSHMS}  } &  {\small (B4) Extended \cite{HSHMS} } \\
\includegraphics[width=33mm, height=38mm]{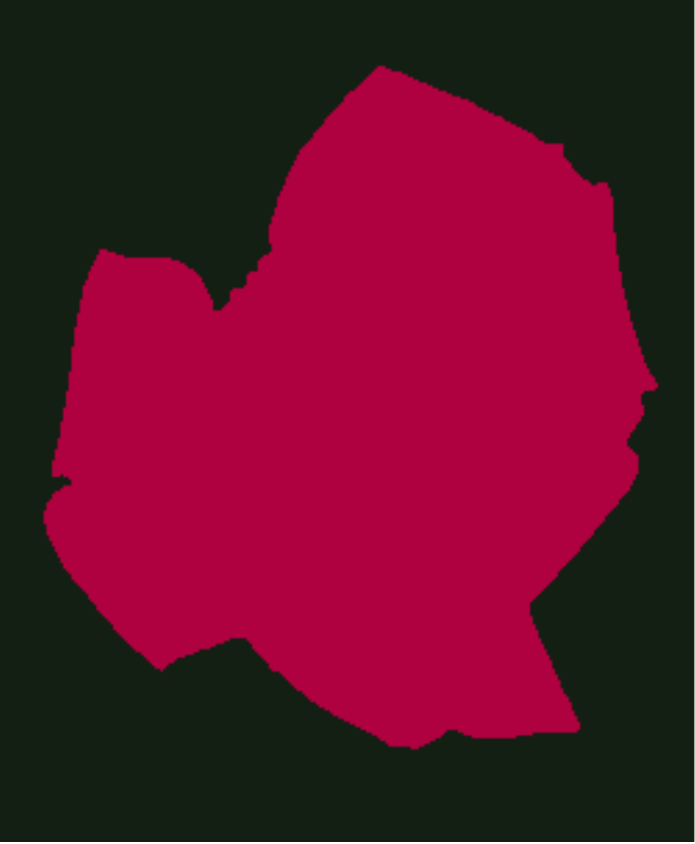} &
\includegraphics[width=33mm, height=38mm]{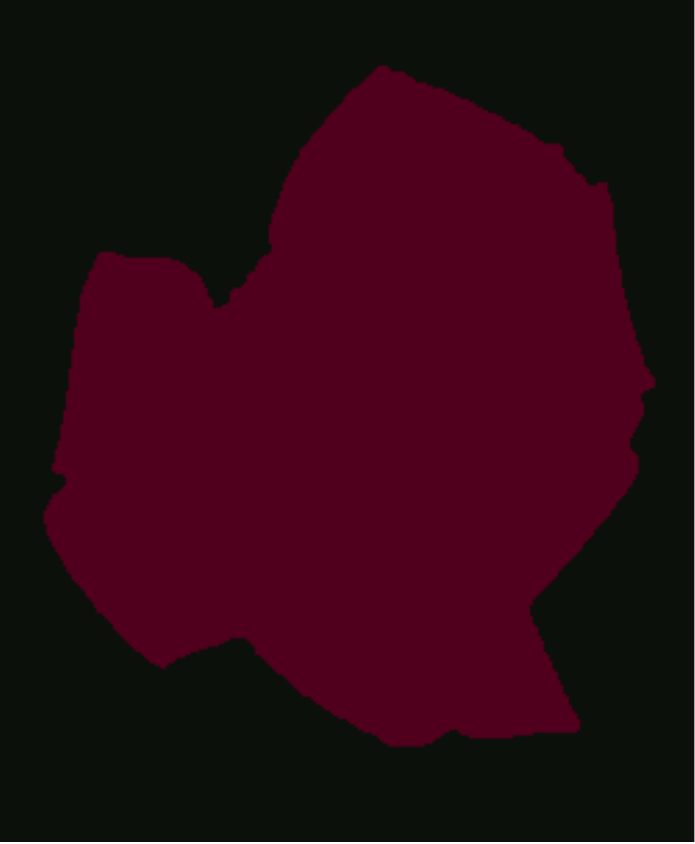} &
\includegraphics[width=33mm, height=38mm]{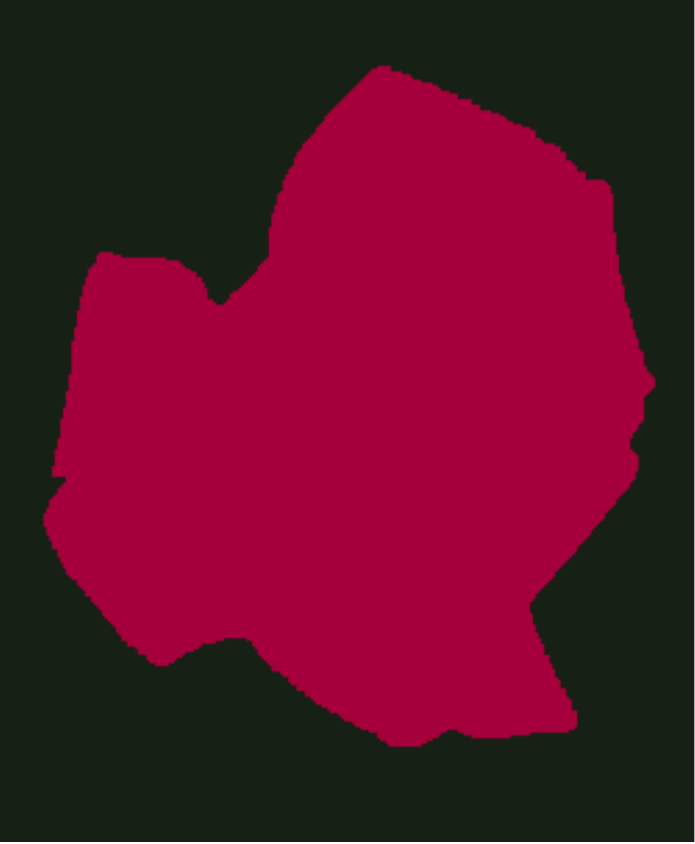} &
\includegraphics[width=33mm, height=38mm]{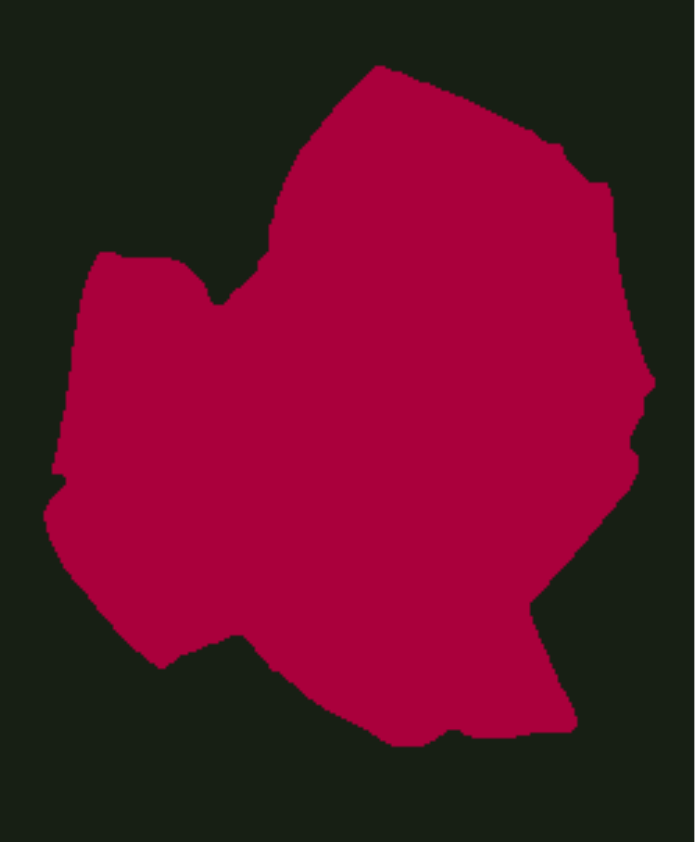} \\
  {\small (C1) Our  }  &  {\small (C2) Our  } & {\small (C3) Our  } & {\small (C4) Our  } \\
\end{tabular}
\end{center}
\caption{Segmentation of rose color image ($303\times 250 \times 3$). Row one: the given
image, and the given images corrupted by $40\%$ information lost, Gaussian blur, and motion 
blur, respectively. Rows two to three: the results of the extended method \cite{HSHMS} 
and our method, respectively. 
}\label{rose-color}
\end{figure*}

{\it Example 5: multiphase images.}
Finally, in order to demonstrate the ability of our method in segmenting color images with information 
lost and blur much more clearly, we test our method in two more multiphase color images, 
i.e. three phases crown image and four phases flowers image, see Fig. \ref{crown-color} 
and Fig. \ref{flowers-color}, respectively. 
Fig. \ref{crown-color}(A1) (Fig. \ref{flowers-color}(A1)) is the original crown (flowers) image, 
and Fig. \ref{crown-color}(A2)--(A4) (Fig. \ref{flowers-color}(A2)--(A4)) are the images corrupted 
by part information removed randomly, Gaussian blur and motion blur, respectively.
Obviously, the extended method \cite{HSHMS}
fails for segmenting the images with information lost, see Fig. \ref{crown-color}(B2) and 
Fig. \ref{flowers-color}(B2).  Moreover, from Fig. \ref{crown-color}(C2) and (D2) and 
Fig. \ref{flowers-color}(C2) and (D2), we can see that the extended method \cite{HSHMS}
gives over smoothed results for blurry color images. On the contrary, from the third column of 
Fig. \ref{crown-color} and Fig. \ref{flowers-color}, the results of our method, we see that all the 
results of our method are very good. Moreover, the results of our method for the corrupted 
images are as good as the results of our method for the original crown and flowers images.

\begin{figure*}[!htb]
\begin{center}
\begin{tabular}{ccc}
\includegraphics[width=50mm, height=35mm]{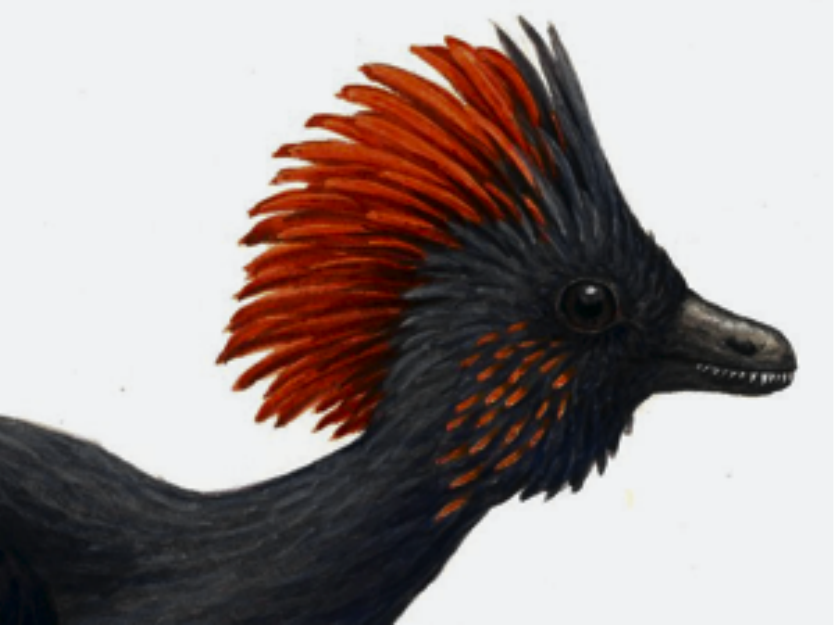}  &
\includegraphics[width=50mm, height=35mm]{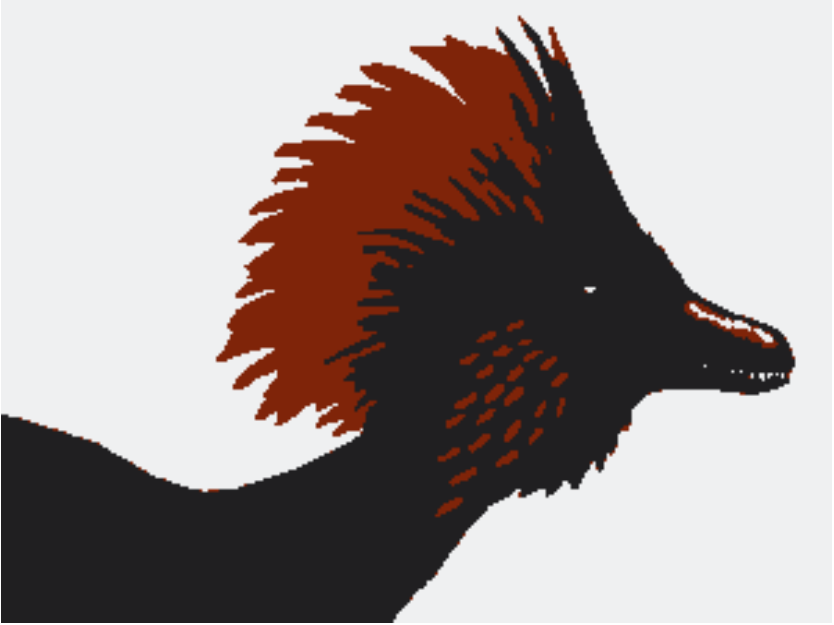} &
\includegraphics[width=50mm, height=35mm]{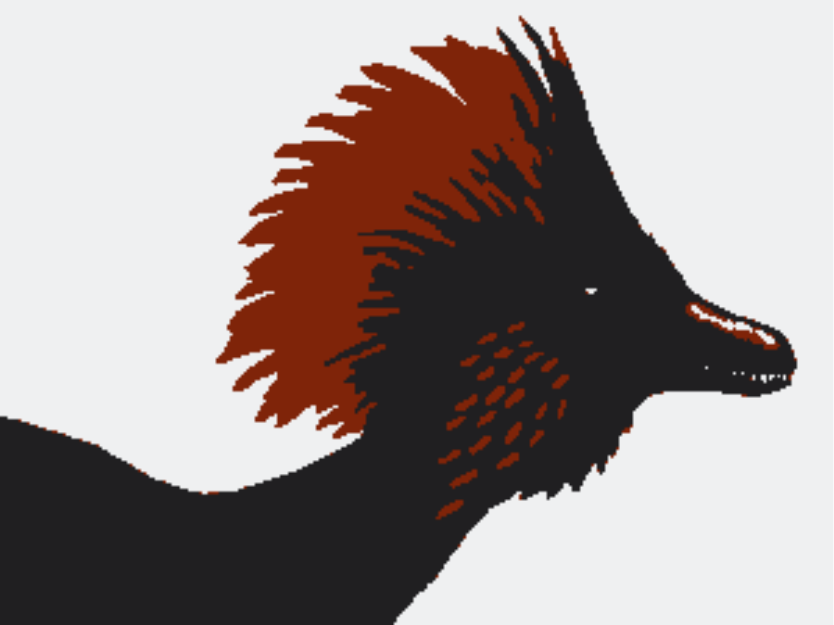} \\
{\small (A1) } & {\small (B1) Extended \cite{HSHMS}  } & {\small (C1) Our  } \\
\includegraphics[width=50mm, height=35mm]{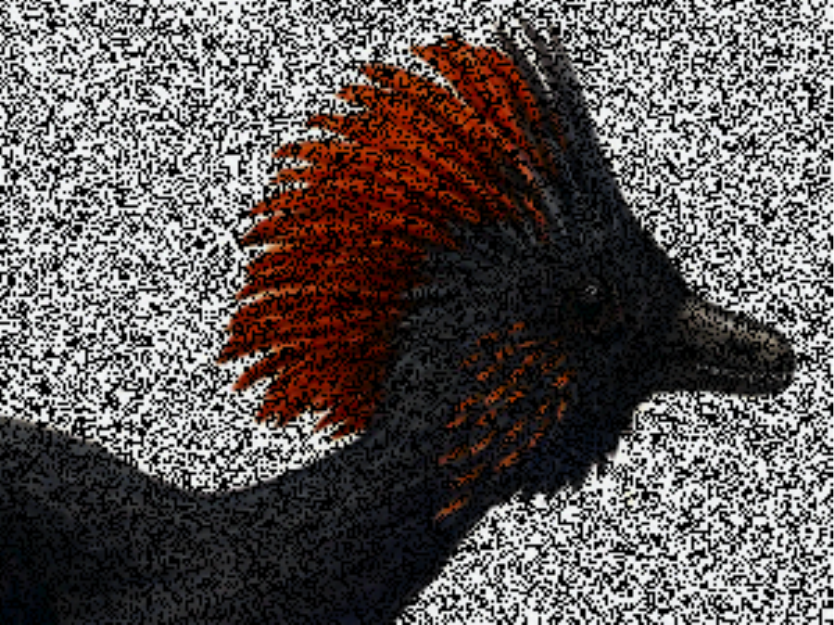}  &
\includegraphics[width=50mm, height=35mm]{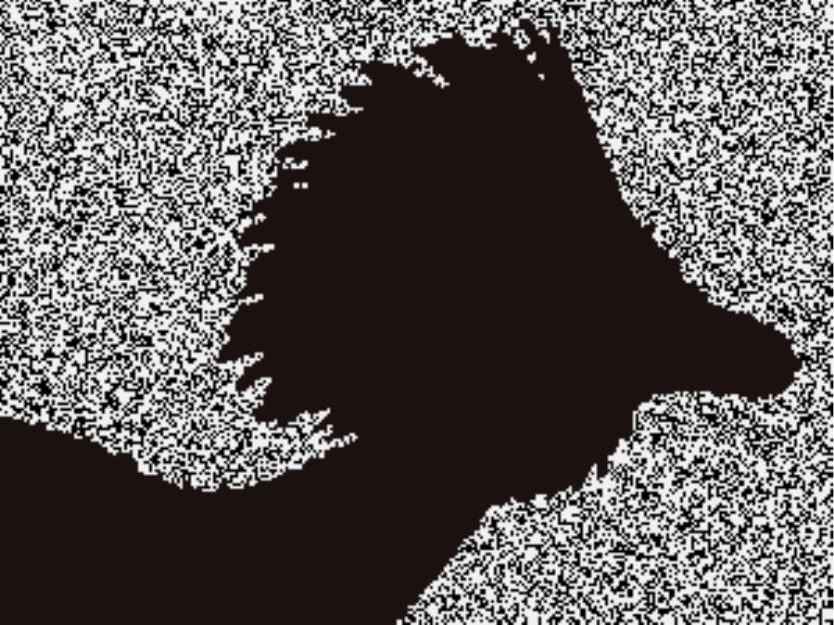} &
\includegraphics[width=50mm, height=35mm]{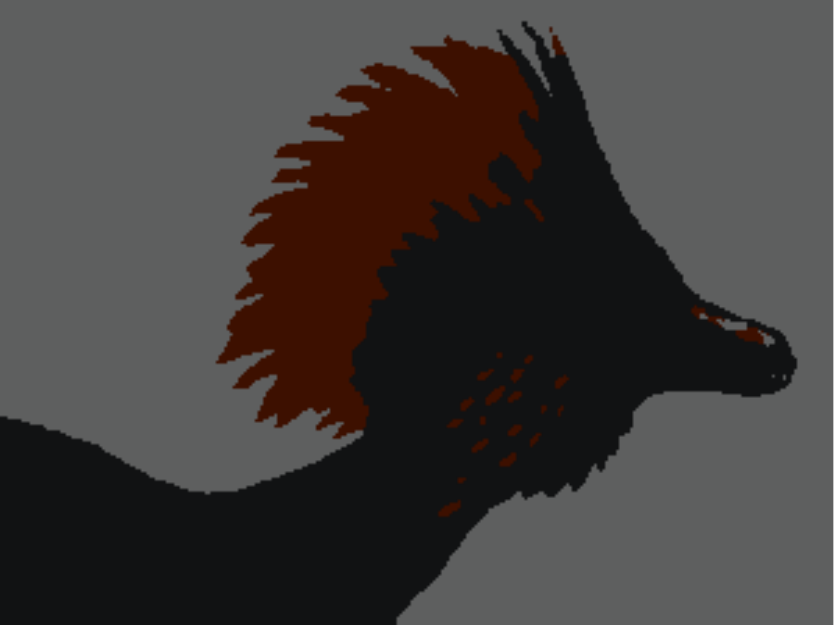} \\
{\small (A2) } & {\small (B2) Extended \cite{HSHMS} } & {\small (C2) Our  } \\
\includegraphics[width=50mm, height=35mm]{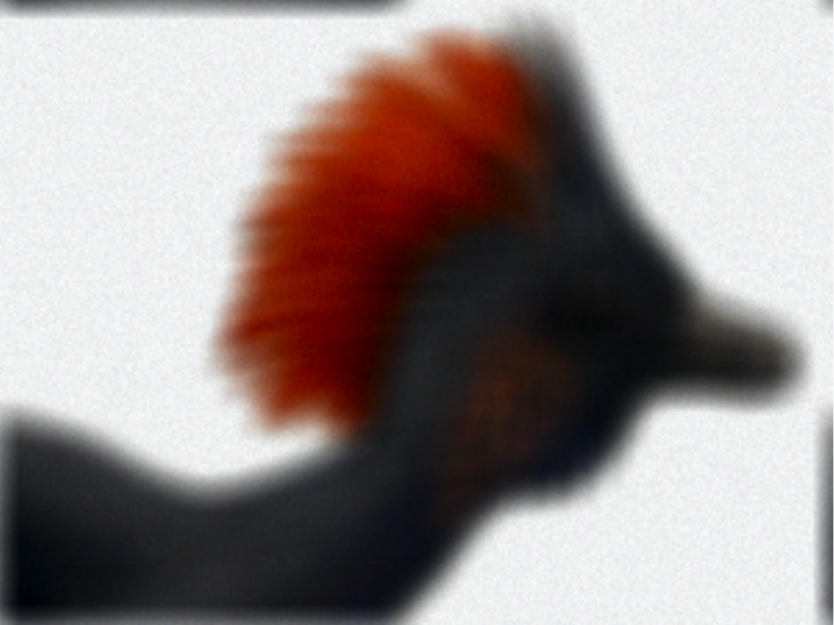}  &
\includegraphics[width=50mm, height=35mm]{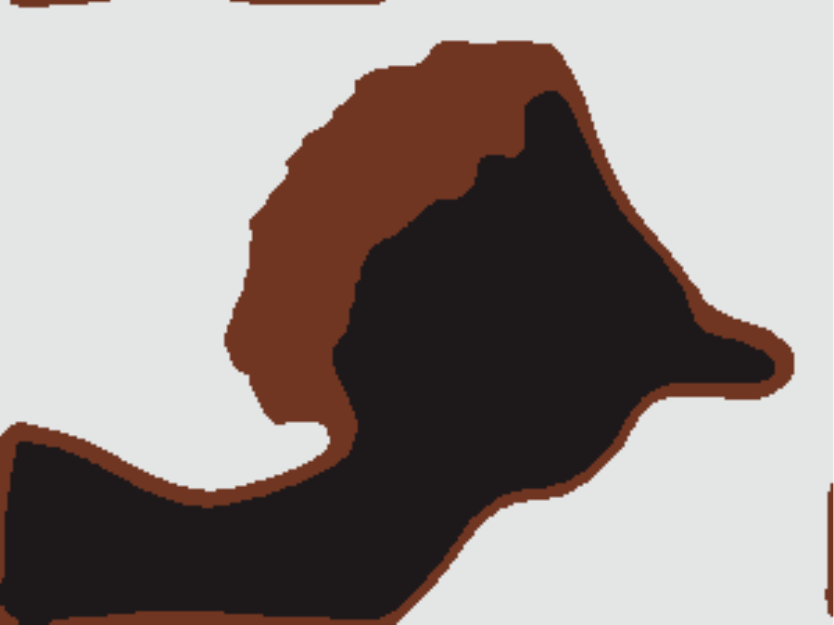} &
\includegraphics[width=50mm, height=35mm]{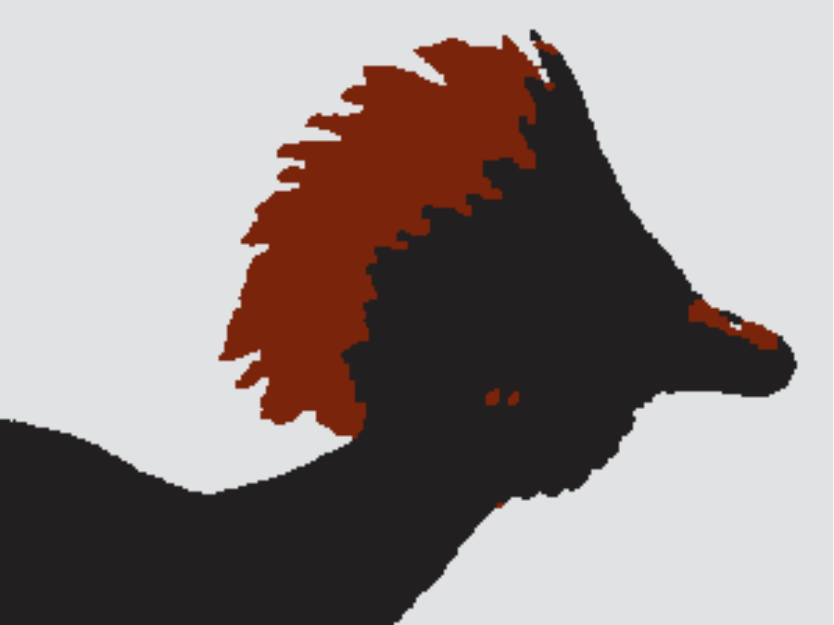} \\
{\small (A3) } & {\small (B3) Extended \cite{HSHMS}  } & {\small (C3) Our  } \\
\includegraphics[width=50mm, height=35mm]{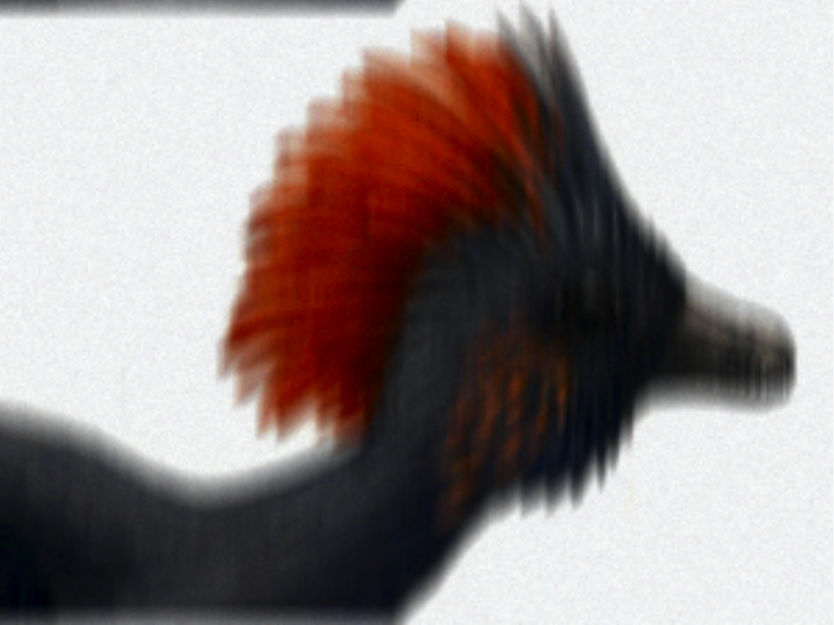}  &
\includegraphics[width=50mm, height=35mm]{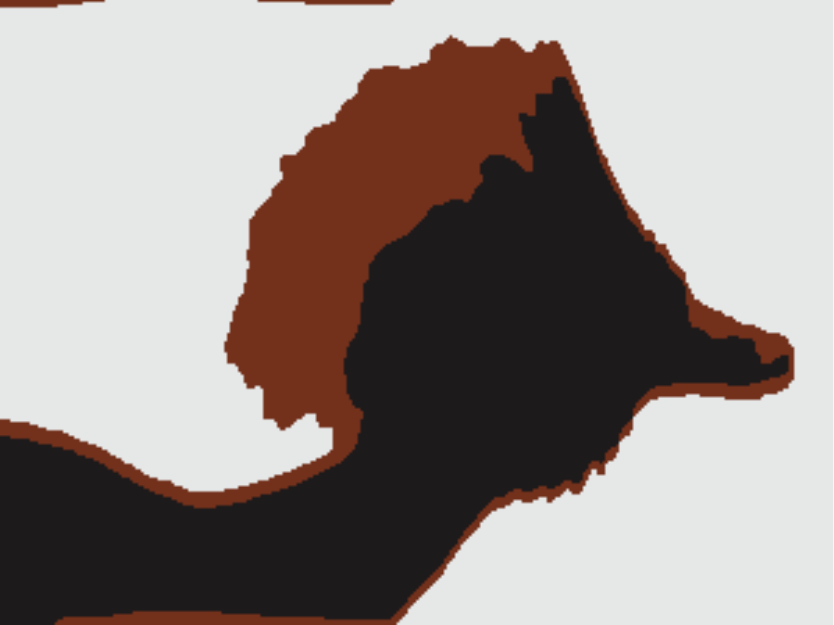} &
\includegraphics[width=50mm, height=35mm]{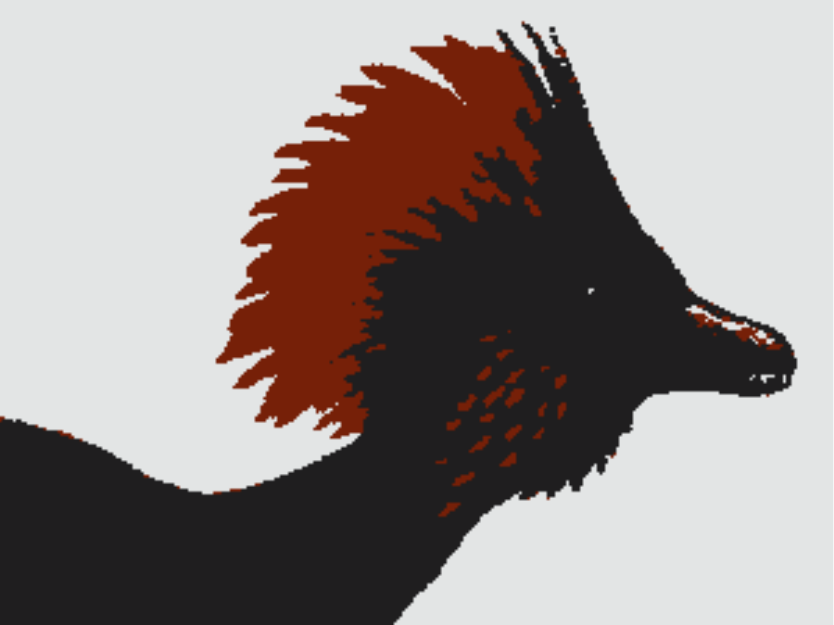} \\
{\small (A4) } & {\small (B4) Extended  \cite{HSHMS}  } & {\small (C4) Our  } \\
\end{tabular}
\end{center}
\caption{Segmentation of crown color image ($225\times 300 \times 3$). Column one: the given
image, and the given images corrupted by $40\%$ information lost, Gaussian blur, and motion 
blur, respectively. Columns two to three: the results of the extended method \cite{HSHMS}
and our method, respectively. 
}\label{crown-color}
\end{figure*}

\begin{figure*}[!htb]
\begin{center}
\begin{tabular}{ccc}
\includegraphics[width=50mm, height=35mm]{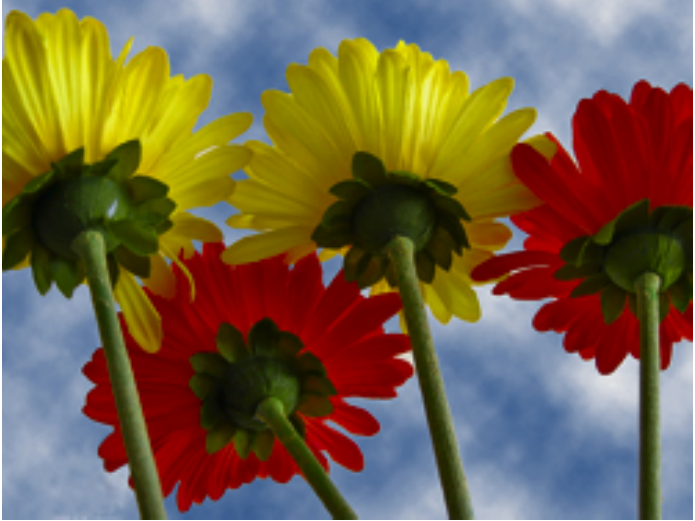}  &
\includegraphics[width=50mm, height=35mm]{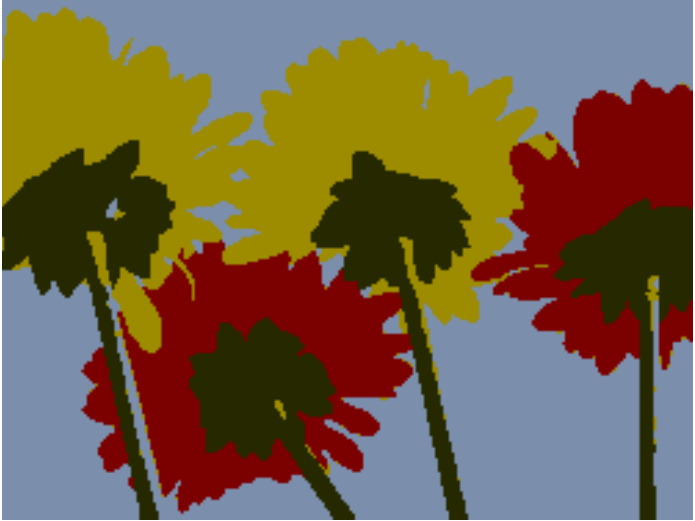} &
\includegraphics[width=50mm, height=35mm]{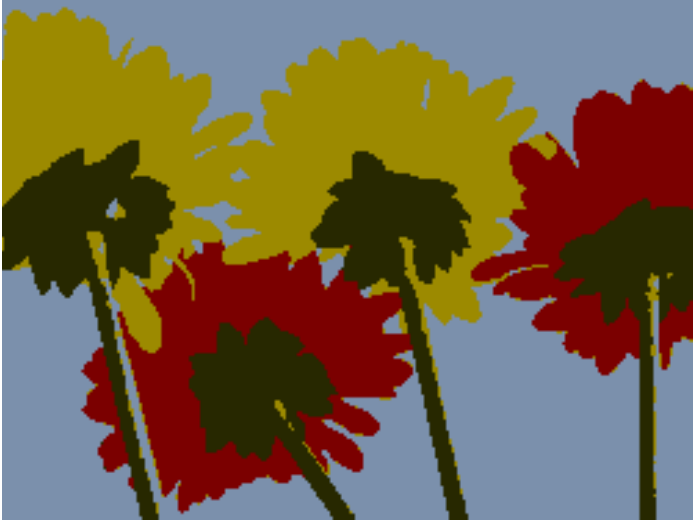} \\
{\small (A1) } & {\small (B1) Extended \cite{HSHMS} } & {\small (C1) Our  } \\
\includegraphics[width=50mm, height=35mm]{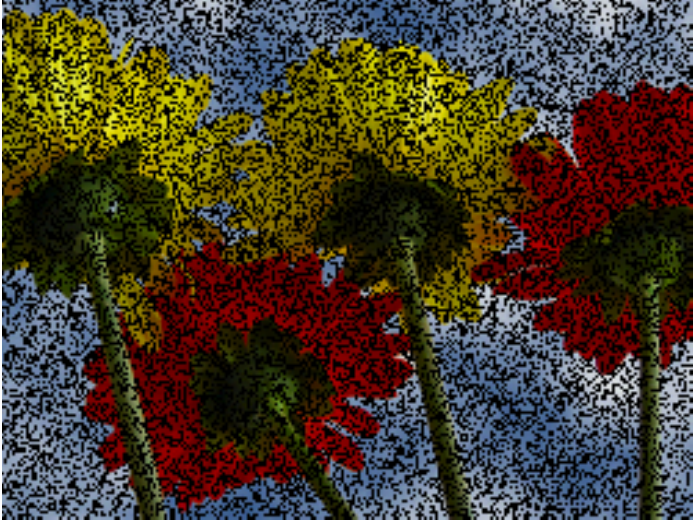}  &
\includegraphics[width=50mm, height=35mm]{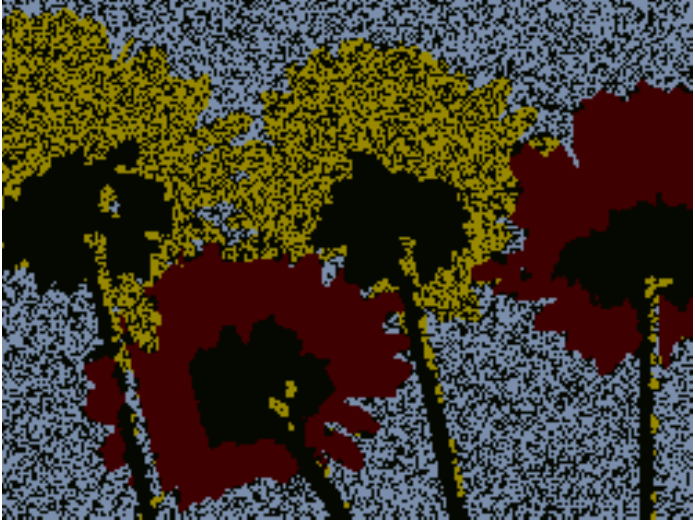} &
\includegraphics[width=50mm, height=35mm]{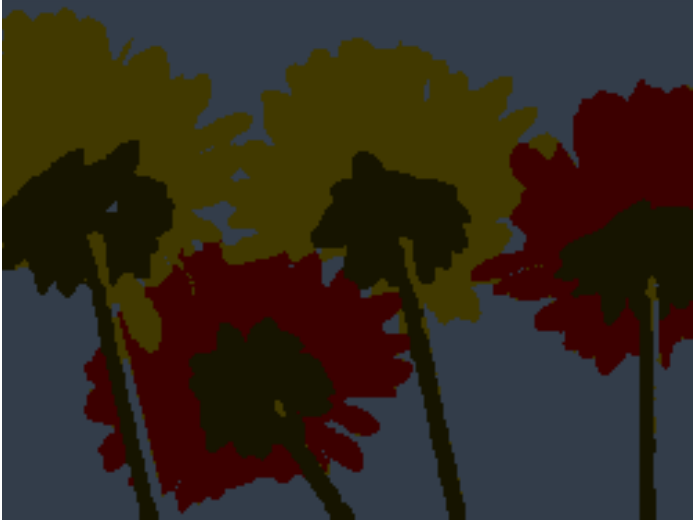} \\
{\small (A2) } & {\small (B2) Extended  \cite{HSHMS} } & {\small (C2) Our  } \\
\includegraphics[width=50mm, height=35mm]{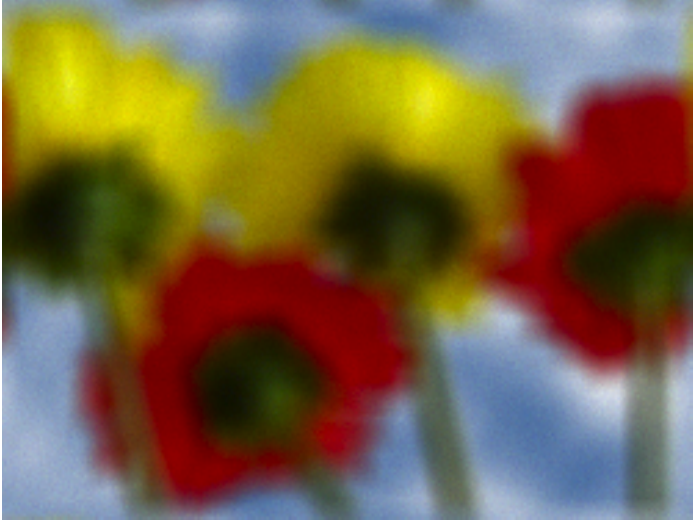}  &
\includegraphics[width=50mm, height=35mm]{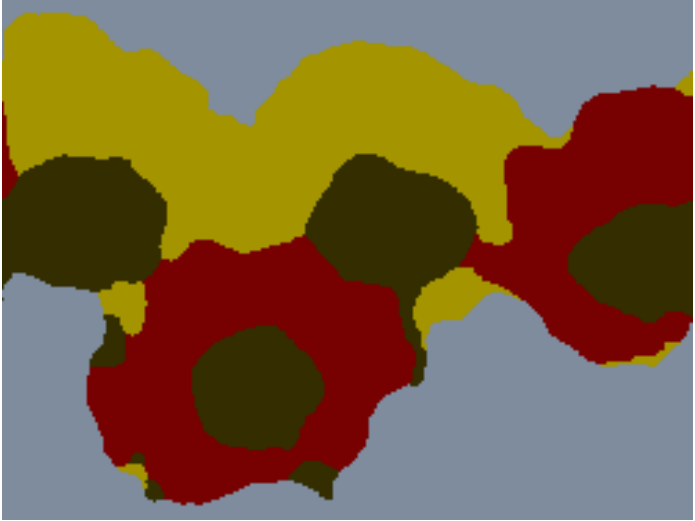} &
\includegraphics[width=50mm, height=35mm]{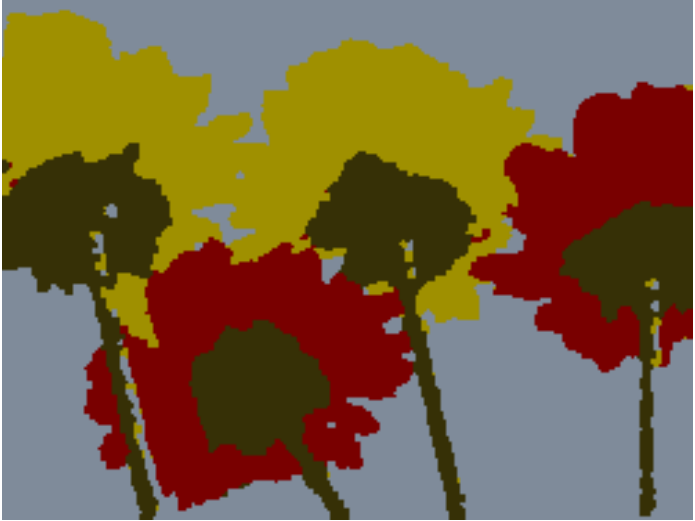} \\
{\small (A3) } & {\small (B3) Extended \cite{HSHMS}} & {\small (C3) Our  } \\
\includegraphics[width=50mm, height=35mm]{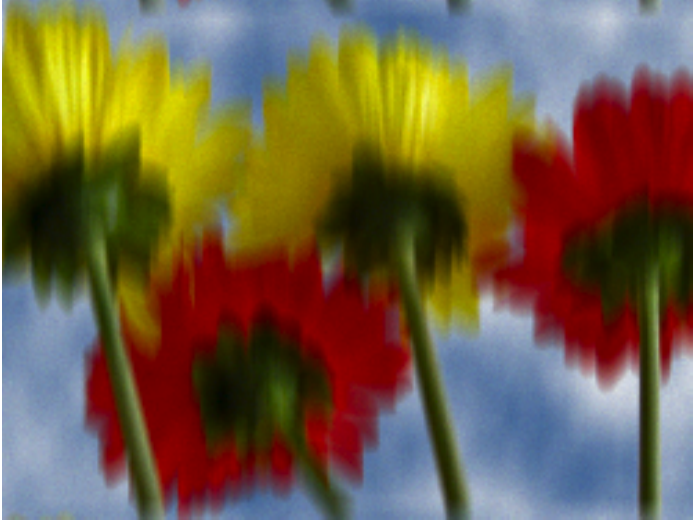}  &
\includegraphics[width=50mm, height=35mm]{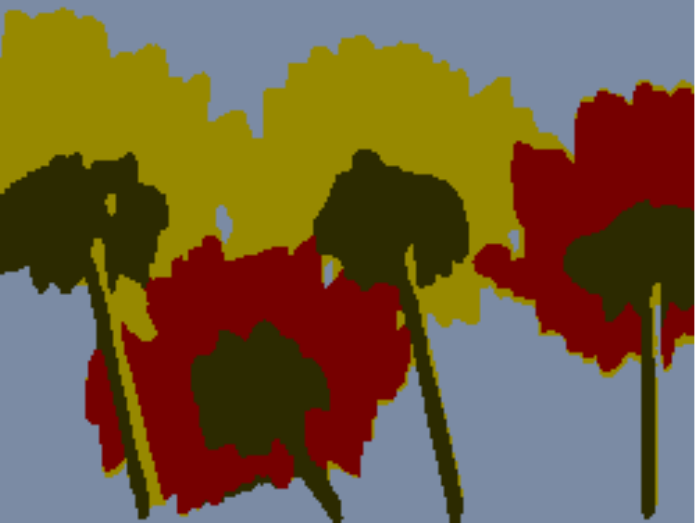} &
\includegraphics[width=50mm, height=35mm]{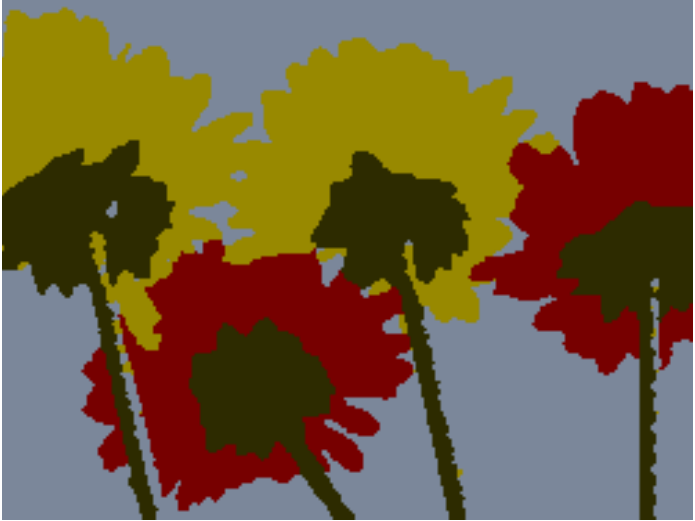} \\
{\small (A4) } & {\small (B4) Extended  \cite{HSHMS} } & {\small (C4) Our  } \\
\end{tabular}
\end{center}
\caption{Segmentation of flowers color image ($188\times 250 \times 3$).  Column one: the given
image, and the given images corrupted by $40\%$ information lost, Gaussian blur, and motion 
blur, respectively. Columns two to three: the results of the extended method \cite{HSHMS}
and our method, respectively. 
}\label{flowers-color}
\end{figure*}

%-------------------------------------------------------------------
\section{Conclusions}\label{sec:conclusions}
%-------------------------------------------------------------------
In this paper, we proposed a new multiphase segmentation model by combining 
the image restoration approaches with the variational image segmentation model.
Utilizing image restoration aspects, the proposed segmentation model is very 
effective and robust to tackle noisy images, blurry images, and images with 
missing pixels. In particular, the piecewise constant Mumford-Shah model
was extended using our strategy so that it can process blurry images.
Moreover, our model can also be extended to process 
vector-valued images for example the color images. It can be solved efficiently using 
the AM algorithm, and we prove its convergence property 
under mild condition. 
Experiments on many kinds of synthetic and real-world images demonstrate that our method 
gives better segmentation results in comparison with others state-of-the-art segmentation 
methods especially in blurry images and images with missing pixels values.
In our future work, we will test our model in images corrupted by other types of noise, 
for example the Poisson noise and the impulsive noise, etc.
 \\
 
%-------------------------------------------------------------------
\noindent{\bf Acknowledgement:}
The main part of this work has been done in the University of Kaiserslautern, Germany. 
Thanks to Prof. Gabriele Steidl (University of Kaiserslautern) for fruitful
discussions and invaluable comments.
%-------------------------------------------------------------------

\section*{Appendix}
\subsection*{Proof of Theorem \ref{thm:existence} }
\begin{proof}
Fix $c_i = c_i^*$ and $u_i = u_i^*$. Note that $L^2(\Omega)$ is a reflective 
Banach space, and $E(g, c_i^*, u_i^*)$ in \eqref{our-model-general} is convex and lower 
semicontinuous. Using Proposition 1.2 in \cite{ET}, for the existence of $g$, we just need to 
prove that $E(g, c_i^*, u_i^*)$ is coercive over $L^2(\Omega)$. The coercive of 
$E(g, c_i^*, u_i^*)$ can be given by
\begin{eqnarray*}
\|g\|_2  & = & \|\sum_{i=1}^K(g-c_i)u_i + \sum_{i=1}^K c_iu_i\|_2  \\ 
&\le & \|\sum_{i= 1}^K |g-c_i| u_i^{\frac{1}{2}}\|_2 + \|\sum_{i= 1}^K c_i u_i\|_2 \\
&\le & \sqrt{K} \sqrt{\sum_{i= 1}^K \int_{\Omega}(g-c_i)^2 u_idx} + \|\sum_{i= 1}^K c_i u_i\|_2 \\
&\le & \frac{\sqrt{K}}{\sqrt{\lambda}} \sqrt{E(g, c_i^*, u_i^*)} + \|\sum_{i= 1}^K c_i u_i\|_2.
\end{eqnarray*}
Moreover, since the middle term of energy  \eqref{our-model-general} is quadratic with respect to
$g$, hence energy \eqref{our-model-general} is strictly convex, therefore it has unique minimizer $g$
for fixed $c_i$ and $u_i$. 
\end{proof}

\subsection*{Proof of Theorem \ref{AM-mono}}
\begin{proof}
From \eqref{AM-ite}, it can be verified easily that
\begin{eqnarray*}
E(x^{(k+1)}, y^{(k+1)}, z^{(k+1)}) &\le & E(x^{(k)}, y^{(k+1)}, z^{(k+1)})   \\
& \le & E(x^{(k)}, y^{(k)}, z^{(k+1)}) \\
& \le & E(x^{(k)}, y^{(k)}, z^{(k)}) \\
& \le & E(x^{(k-1)}, y^{(k)}, z^{(k)}) \\
& \le & E(x^{(k-1)}, y^{(k-1)}, z^{(k)}).
\end{eqnarray*}
Hence, since $E(\cdot, \cdot, \cdot)$ is bounded from below, the sequence 
$\big \{E(x^{(k)}, y^{(k)}, z^{(k)}) \big \}_{k \in \mathbb N}$ converges monotonically.
\end{proof}

\subsection{Proof of Theorem \ref{AM-cong}}
\begin{proof}
Using \eqref{AM-ite} and the idea of \cite[Theorem 5.5]{CS}, for each $i$
\[
E(x^{(k_i)}, y^{(k_i)}, z^{(k_i)}) \leq E(x, y^{(k_i)}, z^{(k_i)}) \quad \forall x \in X.
\]
By the continuity of $E(\cdot, \cdot, \cdot)$, this gives, as $i\rightarrow \infty$,
\[
E(x^*, y^*, z^*) \leq E(x,y^*, z^*)  \quad  \forall x\in X.
\]

On the other hand, for each $i$, note that $k_{i-1} \leq k_i-1$.
From Theorem \ref{AM-mono}, we have for $\forall y\in Y$,
\begin{eqnarray*}
E(x^{(k_i)}, y^{(k_i)}, z^{(k_i)})   &\le & E(x^{(k_i-1)}, y^{(k_i)}, z^{(k_i)}) \\
& \le &  E(x^{(k_{i-1})}, y^{(k_{i-1}+1)}, z^{(k_{i-1}+1)}) \\
& \leq &  E(x^{(k_{i-1})}, y, z^{(k_{i-1}+1)}),
\end{eqnarray*}
and $\forall z\in Z$,
\begin{eqnarray*}
E(x^{(k_i)}, y^{(k_i)}, z^{(k_i)})   &\le & E(x^{(k_i-1)}, y^{(k_i)}, z^{(k_i)}) \\
&\le & E(x^{(k_i-1)}, y^{(k_i-1)}, z^{(k_i)}) \\
& \le &  E(x^{(k_{i-1})}, y^{(k_{i-1})}, z^{(k_{i-1}+1)}) \\
& \leq &  E(x^{(k_{i-1})}, y^{(k_{i-1})}, z).
\end{eqnarray*}

Coupled with the continuity of $E(\cdot, \cdot, \cdot)$, as $i\rightarrow \infty$, we have,
\[
E(x^*, y^*, z^*) \leq E(x^*, y, z^*) \quad {\rm and} \quad E(x^*, y^*, z^*) \leq E(x^*, y^*, z) 
\]
for $\forall y\in Y$ and $\forall z\in Z$ respectively.
\end{proof}

\subsection*{Proof of Theorem \ref{AM-cong-our}}
\begin{proof}
For model \eqref{our-model-multiphase-con}, because all of its three terms 
are continuous and nonnegative, we have $E(\cdot, \cdot, \cdot)$ is continuous and nonnegative. 
If $(u^{(k)}, g^{(k)}, c^{(k)}) \rightarrow (u^*, g^*, c^*)$, as $k\rightarrow \infty$,
using Theorem \ref{AM-cong}, we have $(u^*, g^*, c^*) \in {\cal O}$.
Obviously, the whole components of $u^{(k)}$ are in $[0,1]$, hence $u^{(k)}$ is bounded. 
From  \eqref{sol_c}, we get $c^{(k)}$ is bounded since it is just the convex combination of $g$.
From \eqref{sol_g}, we have
\begin{eqnarray*}
\|g\|_2 & \le&  \|(\mu {\cal A}^T{\cal A} + \lambda {\cal I})^{-1}\|_2 \| \big (\mu {\cal A}^T f 
+ \lambda \sum_{i=1}^K c_i u_i \big)\omega\|_2       \\
& \le & \big (\mu\|{\cal A}^T\|_2 \|f\|_2 
+ \lambda \sum_{i=1}^K c_i \|u_i\|_2 \big) \| \omega\|_2/\lambda,
\end{eqnarray*}
hence $g^{(k)}$  is also bounded. 
Therefore $(u^{(k)}, g^{(k)}, c^{(k)})_{k\in \mathbb N}$ must contain convergent 
subsequence. Using Theorem \ref{AM-cong}, we can get that any of these subsequences
converges to a partial minimizer of model \eqref{our-model-multiphase-con}.
\end{proof}


\bibliographystyle{elsarticle-num}
\bibliography{<your-bib-database>}

\end{document}